\documentclass[journal]{IEEEtran/IEEEtran}
\ifCLASSINFOpdf
  \usepackage[pdftex]{graphicx}
\else
  \usepackage[dvips]{graphicx}
\fi

\usepackage[cmex10]{amsmath}
\usepackage{amsthm}

\usepackage{graphicx, bm, bbm, latexsym, amsfonts, epsfig, verbatim, epsfig, verbatim, enumerate, multirow, caption, graphicx, subcaption, algorithm, algpseudocode, hyperref, url, multicol, blindtext, color, latexsym, amsfonts, epsfig, verbatim, tikz, cases, algcompatible, diagbox, fixmath,makecell,booktabs,comment, lipsum}
\hypersetup{
	colorlinks=true, 
	breaklinks=true,
	urlcolor= black, 
	linkcolor= black, 
	citecolor=red, 
	pdftitle={},
	pdfauthor={},
}

\usepackage{graphicx, bbm, latexsym, amsfonts, epsfig, verbatim, amsmath, color, epsfig, verbatim, enumerate, multirow, caption, graphicx, subcaption, algorithm, algpseudocode}
\usepackage{url, color, latexsym, amsfonts, epsfig, verbatim, tikz, cases, bbm}
\usepackage[top=1.0in, bottom=1.0in, left=1.0in, right=1.0in]{geometry}
\usepackage{algcompatible}
\usepackage{diagbox}
\usepackage{booktabs, caption, makecell}

\usepackage{threeparttable}
\usepackage{array}
\newcolumntype{L}[1]{>{\raggedright\let\newline\\\arraybackslash\hspace{0pt}}m{#1}}
\newcolumntype{C}[1]{>{\centering\let\newline\\\arraybackslash\hspace{0pt}}m{#1}}
\newcolumntype{R}[1]{>{\raggedleft\let\newline\\\arraybackslash\hspace{0pt}}m{#1}}

\DeclareMathOperator*{\argmin}{arg\,min}

\newtheorem{theorem}{Theorem}
\newtheorem{lemma}{Lemma}

\newtheorem{corollary}{Corollary}
\newtheorem{definition}{Definition}
\newtheorem{assumption}{Assumption}

\newtheorem{proposition}{Proposition}
\newtheorem{remark}{Remark}

\makeatletter
\newcommand{\subalign}[1]{%
  \vcenter{%
    \Let@ \restore@math@cr \default@tag
    \baselineskip\fontdimen10 \scriptfont\tw@
    \advance\baselineskip\fontdimen12 \scriptfont\tw@
    \lineskip\thr@@\fontdimen8 \scriptfont\thr@@
    \lineskiplimit\lineskip
    \ialign{\hfil$\m@th\scriptstyle##$&$\m@th\scriptstyle{}##$\hfil\crcr
      #1\crcr
    }%
  }%
}
\makeatother

\allowdisplaybreaks

\begin{document}
\title{Differentially Private Federated Learning via Inexact ADMM with Multiple Local Updates
}

\author{Minseok~Ryu,~\IEEEmembership{Member,~IEEE,} 
  and~Kibaek~Kim,~\IEEEmembership{Member,~IEEE,}
\thanks{M. Ryu and K. Kim are with the Mathematics and Computer Science Division, Argonne National Laboratory, Lemont, IL, USA (Contact: \{mryu, kimk\}@anl.gov).
This material is based upon work  supported by the U.S. Department of Energy, Office of Science, Advanced Scientific Computing Research, under Contract DE-AC02-06CH11357.
We gratefully acknowledge the computing resources provided on Swing, a high-performance computing cluster operated by the Laboratory Computing Resource Center at Argonne National Laboratory.
}
}

\maketitle

\begin{abstract}  
Differential privacy (DP) techniques can be applied to the federated learning model to statistically guarantee data privacy against inference attacks to communication among the learning agents. While ensuring strong data privacy, however, the DP techniques hinder achieving a greater learning performance. In this paper we develop a DP inexact alternating direction method of multipliers algorithm with multiple local updates for federated learning, where a sequence of convex subproblems is solved with the objective perturbation by random noises generated from a Laplace distribution. 
We show that our algorithm provides $\bar{\epsilon}$-DP for every iteration, where $\bar{\epsilon}$ is a privacy budget controlled by the user. We also present convergence analyses of the proposed algorithm. Using MNIST and FEMNIST datasets for the image classification, we demonstrate that our algorithm reduces the testing error by at most $31\%$ compared with the existing DP algorithm, while achieving the same level of data privacy. The numerical experiment also shows that our algorithm converges faster than the existing algorithm.
\end{abstract}


%
\IEEEpeerreviewmaketitle

\section{Introduction} \label{sec:introduction}
In this work we propose a privacy-preserving algorithm for training a federated learning (FL) model \cite{konevcny2015federated}, namely, a machine learning (ML) model that aims to learn global model parameters \textit{without} collecting locally stored data from agents to a central server.
The proposed algorithm is based on an inexact alternating direction method of multipliers (IADMM) that solves a sequence of subproblems whose \textit{objective functions} are perturbed by injecting some random noises for ensuring \textit{differential privacy} (DP) on the distributed data.
We show that the proposed algorithm provides more accurate solutions compared with the state-of-the-art DP algorithm \cite{huang2019dp} while both algorithms provide the same level of data privacy.
As a result, the proposed algorithm can mitigate a trade-off between data privacy and solution accuracy (i.e., learning performance in the context of ML), which is one of the main challenges in developing DP algorithms, as described in \cite{dwork2014algorithmic}.

Developing highly accurate privacy-preserving algorithms can enhance the practical uses of FL in applications with sensitive data (e.g., electronic health records \cite{shickel2017deep} and mobile device data \cite{mcmahan2017communication}) because a greater learning performance can be achieved while preserving privacy on the sensitive data exposed to be leaked during a training process.
Furthermore, it would allow a stronger differential privacy budget to FL.
Because of the importance of FL, incorporating privacy-preserving techniques into optimization algorithms for solving the FL models has been studied extensively \cite{zhang2016dynamic, huang2019dp, wei2020federated, naseri2020toward}.

\textbf{Related Work.}
The empirical risk minimization (ERM) model used for learning parameters in supervised ML is often vulnerable to adversarial attacks \cite{madry2017towards}, a situation that motivates the application of privacy-preserving techniques (e.g., DP \cite{dwork2006calibrating} and homomorphic encryption \cite{kaissis2020secure}) to protect data.
Among these techniques, DP has been widely used in the ML community and is especially useful for protecting data against inference attacks \cite{shokri2017membership}.

Formally, DP is a privacy-preserving technique that randomizes the output of a query such that any single data point cannot be inferred by an adversary that can reverse-engineer the randomized output.
Depending on where to inject noises to randomize the output, DP can be categorized by input \cite{fukuchi2017differentially, kang2020input}, output \cite{dwork2006calibrating, chaudhuri2011differentially}, and objective \cite{chaudhuri2011differentially, kifer2012private} perturbation methods.
Compared with the input perturbation, which directly perturbs input data by adding random noises,  \textit{output perturbation} and \textit{objective perturbation} methods provide a randomized output of an optimization problem by injecting random noises into its true output and objective function, respectively.
In \cite{chaudhuri2011differentially}, the authors propose a differentially private ERM that utilizes the output and objective perturbation methods to ensure DP on data.
Also, Abadi et al.~\cite{abadi2016deep} apply the output perturbation to stochastic gradient descent (SGD) in order to ensure DP on data for every iteration of the algorithm.
The privacy-preserving technique in our work is the \textit{objective perturbation} method.
For details of differentially private ML, we refer readers to \cite{sarwate2013signal,kifer2012private,iyengar2019towards}.

Within the context of FL, various optimization algorithms have been developed for solving the distributed ERM model in a communication-efficient manner.
For example, \texttt{FedAvg} in \cite{mcmahan2017communication} reduces the number of communication rounds by allowing each agent to conduct multiple local updates via SGD while a central server performs model averaging for a global update.
Another example is \texttt{FedProx}  \cite{li2018federated}, constructed by adding a proximal function to the objective function of the local model considered in \texttt{FedAvg}, resulting in better learning performance.
Recently, the authors in \cite{zhou2021communication} develop a communication-efficient ADMM by enhancing local computation whereas the vanilla ADMM conducts a single local computation per communication round.
Even though these algorithms can reduce the number of communication rounds to mitigate the chance of data leakage, they do not guarantee data privacy during a training process, preventing their practical uses.
Readers interested in details of FL should see \cite{kairouz2019advances,li2019survey,li2020federated}. 

In order to preserve privacy on data used for the FL model, various DP algorithms have been proposed in the literature, where the output and objective perturbations are incorporated for ensuring DP (see \cite{agarwal2018cpsgd, wei2020federated, naseri2020toward,zhang2016dynamic, huang2019dp}).
For example, the intermediate model parameters and/or gradients computed for every iteration of the \texttt{FedAvg}-type and \texttt{FedProx}-type algorithms are perturbed for guaranteeing DP as in \cite{naseri2020toward} and \cite{wei2020federated}, respectively, which can be seen as the output perturbation.
Also, in \cite{zhang2016dynamic}, the primal and dual variables computed for every iteration of the vanilla ADMM algorithm are perturbed, which can be seen as the output and objective perturbations, respectively. 
Zhang and Zhu \cite{zhang2016dynamic} compare the two perturbation methods, as Chaudhuri et al.~\cite{chaudhuri2011differentially} did under the general ML setting, and show that the objective perturbation can provide more accurate solutions compared with the output perturbation.
The use of the objective perturbation is somewhat limited, however, because it requires the objective function to be twice differentiable and strongly convex whereas the twice differentiability restriction can be relaxed to the differentiability for the output perturbation.
In \cite{huang2019dp}, the authors incorporate the output perturbation into IADMM that utilizes a first-order approximation with a proximal function.
Introducing the first-order approximation in ADMM enforces smoothness of the objective function, hence satisfying the  differentiability assumption for ensuring DP.
Also, the authors show that the algorithm has a $\mathcal{O}(1/\sqrt{T})$ rate of convergence in expectation, where $T$ is the number of iterations.
Moreover, their numerical experiments demonstrate that the algorithm outperforms DP-ADMM in \cite{zhang2016dynamic} and DP-SGD in \cite{abadi2016deep}.

\textbf{Contributions.}
In this paper, as compared with the DP-IADMM algorithm in \cite{huang2019dp}, we incorporate the \textit{objective perturbation} into IADMM that utilizes the first-order approximation.
Moreover, we introduce a multiple local updates technique into DP-IADMM to reduce communication rounds. We note that the technique has been applied to IADMM in \cite{zhou2021communication}, but not in the context of DP-IADMM.
Our main contributions are summarized as follows:
\begin{itemize}
  \itemsep0em
  \item Development of new IADMM algorithm with $\bar{\epsilon}$-DP on data
  \item Application of multiple local updates for better communication efficiency
  \item Proof that the rate of convergence in expectation for our DP algorithm is 
  \begin{itemize}
    \item $O(\frac{1}{\bar{\epsilon} \sqrt{T}})$ under a smooth convex function, 
    \item $O(\frac{1}{\bar{\epsilon}^2 \sqrt{T}})$ under a nonsmooth convex function,
    \item $O(\frac{1}{\bar{\epsilon}^2 T})$ under a strongly-convex function
  \end{itemize}   
  \item Numerical demonstration that our DP algorithm provides more accurate solutions compared with the existing DP algorithm \cite{huang2019dp}
\end{itemize}

\textbf{Organization and Notation.}
The remainder of the paper is organized as follows.
In Section \ref{sec:model} we describe an FL model using a distributed ERM and present the existing IADMM algorithm for solving the FL model.
In Section \ref{sec:MultipleLocalUpdate} we propose a new DP-IADMM algorithm for solving the FL model that ensures DP on data and converges to an optimal solution with the sublinear convergence rate.
In Section \ref{sec:experiments} we describe numerical experiments to demonstrate the outperformance of the proposed algorithm.

We denote by $\mathbb{N}$ a set of natural numbers. For $A \in \mathbb{N}$, we define $[A]:= \{1,\ldots,A\}$ and denote by $\mathbb{I}_{A}$ an $A \times A$ identity matrix.
We use $\langle \cdot, \cdot \rangle$ and $\|\cdot \|$ to denote the scalar product and the Euclidean norm, respectively.
For a given function $f$, we use $f'$ and $\nabla f$ to denote a subgradient and a gradient of $f$, respectively.

\section{Preliminaries} \label{sec:model}
\textbf{Distributed ERM.}
Consider a set $[P]$ of agents connected to a central server.
Each agent $p \in [P]$ has a training dataset $\mathcal{D}_p := \{x_{pi}, y_{pi} \}_{i=1}^{I_p}$, where $I_p$ is the number of data samples, $x_{pi} \in \mathbb{R}^J$ is a $J$-dimensional data feature, and $y_{pi} \in \mathbb{R}^K$ is a $K$-dimensional data label.
We consider a \textit{distributed} ERM problem given by
\begin{align}
\min_{w  \in \mathcal{W}} \ &  \sum_{p=1}^P \Big\{ \frac{1}{I} \sum_{i=1}^{I_p} l(w; x_{pi},y_{pi}) + \frac{\beta}{P} r(w) \Big\}, \label{ERM_0}
\end{align}
where
$w \in \mathbb{R}^{J \times K}$ is a global model parameter vector,
$\mathcal{W}$ is a compact convex set,
$l(\cdot)$ is a convex loss function, $r(\cdot)$ is a convex regularizer function, $\beta > 0$ is a regularizer parameter, and $I := \sum_{p=1}^P I_p$.

By introducing a local model parameter $z_p \in \mathbb{R}^{J \times K}$ defined for every agent $p \in [P]$, we can rewrite \eqref{ERM_0}  as
\begin{subequations}
\label{ERM_1}
\begin{align}
  \min_{w, \{z_p\}_{p=1}^P \in \mathcal{W}} \ & \sum_{p=1}^P f_p(z_p)  \\
   \mbox{s.t.} \ &  w  = z_{p}, \ \forall p \in [P], \label{ERM_1-1}
\end{align}
where  
\begin{align}
& f_p(z_p) :=  \frac{1}{I} \sum_{i=1}^{I_p} l(z_p; x_{pi},y_{pi}) + \frac{\beta}{P} r(z_p). \label{def_fn_f}
\end{align}
\end{subequations}   
Since \eqref{ERM_1} is a convex optimization problem, it can be expressed by the \textit{equivalent} Lagrangian dual problem:
\begin{align}
  \max_{ \{\lambda_p\}_{p=1}^P  } \min_{w, \{z_p\}_{p=1}^P  \in \mathcal{W}} \sum_{p=1}^P \left\{ f_p(z_p) + \langle \lambda_p, w-z_p \rangle \right\}, \label{ERM_LDual}
\end{align}
where $\lambda_p \in \mathbb{R}^{J \times K}$ is a dual vector associated with constraints \eqref{ERM_1-1}.

\textbf{ADMM.} 
ADMM is an iterative optimization algorithm that can find an optimal solution of \eqref{ERM_LDual} in an augmented Lagrangian form.
More specifically, for every $t \in [T]$, where $T$ is the number of iterations, it updates 
\begin{align}
(w^{t}, z^{t}, \lambda^{t}) \rightarrow (w^{t+1}, z^{t+1}, \lambda^{t+1})  \nonumber 
\end{align}
by solving the following subproblems sequentially:
\begin{subequations}
\label{ADMM}
\begin{align}
& w^{t+1} \leftarrow \argmin_{w} \ \sum_{p=1}^P \Big( \langle \lambda^t_p, w \rangle + \frac{\rho^t}{2} \|w-z^t_p\|^2 \Big), \label{ADMM-1} \\
& z^{t+1}_p \leftarrow \argmin_{z_p \in \mathcal{W}} \ f_p(z_p) - \langle \lambda^t_p, z_p \rangle \nonumber \\
& \hspace{25mm} + \frac{\rho^t}{2} \|w^{t+1}-z_p\|^2, \ \forall p \in [P], \label{ADMM-2} \\
& \lambda^{t+1}_p \leftarrow  \lambda^{t}_p + \rho^t (w^{t+1}-z^{t+1}_p), \ \forall p \in [P], \label{ADMM-3}
\end{align}
\end{subequations}
where $\rho^t > 0$ is a hyperparameter that may be fine-tuned for better performance. 

\textbf{Inexact ADMM.}
The subproblem \eqref{ADMM-2} does not need to be solved exactly in each iteration to guarantee the overall convergence. 
In~\cite{huang2019dp}, \eqref{ADMM-2} is replaced with the following inexact subproblem:
\begin{align}
z_p^{t+1} \leftarrow  \argmin_{z_p  \in \mathcal{W} } \ & \langle f'_p(z^t_p), z_p \rangle + \frac{1}{2\eta^{t}}\|z_p - z^t_p\|^2 + \nonumber \\
& \frac{\rho^t}{2} \|w^{t+1}-z_p + \frac{1}{\rho^t} \lambda^t_p  \|^2 ,  \label{ADMM-2-Prox}
\end{align}
which is obtained by (i) replacing the convex function $f_p(z_p)$ in \eqref{ADMM-2} with its lower approximation $\widehat{f}_p(z_p) := f_p(z^t_p) + \langle f'_p(z^t_p), \ z_p - z_p^t \rangle$, where $f'_p(z^t_p)$ is a subgradient of $f_p$ at $z^t_p$, and (ii) adding a \textit{proximal} term $\frac{1}{2\eta^{t}}\|z_p - z^t_p\|^2$ with a proximity parameter $\eta^{t} > 0$ that controls the proximity of a new solution $z^{t+1}_p$ from $z^t_p$ computed from the previous iteration.
Note that the proximal term is used for finding a new solution within a certain distance from the solution computed from the previous iteration and has been widely used for numerous optimization algorithms (e.g., the bundle method \cite{teo2010bundle}).

In this paper we refer to $\{\text{\eqref{ADMM-1}}\rightarrow\eqref{ADMM-2-Prox}\rightarrow\text{\eqref{ADMM-3}}\}_{t=1}^T$ as IADMM.
Within the context of federated learning, IADMM is composed of the following four components:
\begin{enumerate}
  \item The central server solves \eqref{ADMM-1} to update the global model parameter $w^{t+1}$.
  \item The central server broadcasts $w^{t+1}$ to all agents.
  \item Each agent $p$ solves \eqref{ADMM-2-Prox} and \eqref{ADMM-3} to update local model parameter $z_p^{t+1}$ and dual information $\lambda_p^{t+1}$.
  \item Each agent $p$ sends the local update $(z_p^{t+1}, \lambda_p^{t+1})$ to the server.
\end{enumerate}

\section{Differentially Private IADMM with Multiple Local Updates} \label{sec:MultipleLocalUpdate}

We generalize the IADMM algorithm by introducing multiple local updates and differential privacy techniques.
The proposed algorithm aims to
(i) improve learning performance by introducing multiple local updates and
(ii) protect data privacy against adversaries that can infer the locally stored data by reverse-engineering the local model parameters communicated during the training process.
We present the privacy and convergence analyses of the proposed algorithm in Section \ref{sec:privacy} and \ref{sec:convergence}, respectively.

\textbf{Multiple Local Updates.}
We introduce the multiple local updates in IADMM, namely, solving \eqref{ADMM-2-Prox}  multiple times, to improve communication efficiency.
In other words, for every $e \in [E]$, where $E$ is the number of local updates, we solve
\begin{align}
z_p^{t,e+1} \leftarrow  \argmin_{z_p  \in \mathcal{W} } \ &  \langle f'_p(z^{t,e}_p), z_p \rangle + \frac{1}{2\eta^{t}}\|z_p - z^{t,e}_p\|^2 + \nonumber \\
& \frac{\rho^t}{2} \|w^{t+1}-z_p + \frac{1}{\rho^{t}} \lambda^t_p  \|^2.  \label{ADMM-2-MLU}
\end{align}
This is different from the existing work \cite{zhou2021communication} that considers both multiple local primal and dual updates, namely, solving \eqref{ADMM-2-Prox} and \eqref{ADMM-3} multiple times per iteration, resulting in communicating not only local model parameters but also dual information.
In contrast, our approach does not require communicating dual information and hence reduces the communication burden.
This point will be made clearer when describing Algorithm \ref{algo:DP-IADMM-Prox}. 

\textbf{DP via Objective Perturbation.}
We propose to perturb the objective function of the constrained subproblem \eqref{ADMM-2-MLU} by adding some random noise for ensuring differential privacy.
DP is a data privacy preservation technique that aims to protect data by randomizing outputs of a function that takes data as inputs.
A formal definition follows.
\begin{definition}{\textbf{(Definition 3 in \cite{chaudhuri2011differentially}})} \label{def:differential_privacy_1}
A randomized function $\mathcal{A}$ provides $\bar{\epsilon}$-DP if for any two datasets $\mathcal{D}$ and $\mathcal{D}'$ that differ in a single entry and for any set $\mathcal{S}$,
\begin{align}
\Big| \ln \Big( \frac{\mathbb{P}( \mathcal{A}(\mathcal{D}) \in \mathcal{S} ) }{\mathbb{P}( \mathcal{A}(\mathcal{D}') \in \mathcal{S} )  } \Big) \Big| \leq  \bar{\epsilon},  \label{def_differential_privacy_1}
\end{align}
where $\mathcal{A}(\mathcal{D})$ (resp. $\mathcal{A}(\mathcal{D}')$) is the randomized output of $\mathcal{A}$ on input $\mathcal{D}$ (resp. $\mathcal{D}'$).
\end{definition}
The definition implies that as $\bar{\epsilon}$ decreases, it becomes harder to distinguish the two datasets $\mathcal{D}$ and $\mathcal{D}'$ by analyzing the randomized outputs, thus providing stronger data privacy.

We aim to construct the randomized function $\mathcal{A}$ satisfying \eqref{def_differential_privacy_1} by introducing some calibrated random noise into the objective function of the subproblem \eqref{ADMM-2-MLU} to protect data in an $\bar{\epsilon}$-DP manner.
To this end, we add an affine function $\frac{1}{2\rho^t} \| \tilde{\xi}^{t,e}_p \|^2 - \langle w^{t+1} - z_p + \frac{1}{\rho^t} \lambda^t_p, \tilde{\xi}^{t,e}_p \rangle$ to \eqref{ADMM-2-MLU}, resulting in
\begin{align}
  z^{t,e+1}_p \leftarrow & \argmin_{z_p \in \mathcal{W}} \ \langle f'_p(z^{t,e}_p), z_p \rangle + \frac{1}{2\eta^{t}} \| z_p - z_p^{t,e}\|^2 \nonumber \\
  & + \frac{\rho^t}{2} \|w^{t+1}-z_p + \frac{1}{\rho^t}(\lambda^t_p - \tilde{\xi}^{t,e}_p) \|^2,  \label{DPADMM-2-Prox}
\end{align}  
where $\tilde{\xi}^{t,e}_p \in \mathbb{R}^{J \times K}$ is a noise vector sampled from a Laplace distribution with zero mean and a scale parameter $\bar{\Delta}^{t,e}_p / \bar{\epsilon}$ whose probability density function (pdf) is given by
\begin{subequations}  
\label{Laplace}  
\begin{align}
& \text{Lap} ( \tilde{\xi}^{t,e}_p ; 0, \bar{\Delta}_p^{t,e}/\bar{\epsilon} ) :=  \frac{\bar{\epsilon}}{2 \bar{\Delta}_p^{t,e}} \exp \big(- \frac{\bar{\epsilon} \| \tilde{\xi}^{t,e}_p \|_1 }{\bar{\Delta}_p^{t,e}} \big), \label{Laplace-pdf} 
\end{align}
where $\bar{\epsilon} > 0$,
\begin{align}
& \bar{\Delta}_p^{t,e} := \max_{\mathcal{D}'_p \in \widehat{\mathcal{D}}_p} \| f'_p(z^{t,e}_p;\mathcal{D}_p) - f'_p(z^{t,e}_p;\mathcal{D}'_p)\|_1, \label{Delta} \\
& \widehat{\mathcal{D}}_p := \text{a collection of datasets differing a single } \nonumber \\
& \hspace{10mm} \text{entry from a given dataset } \mathcal{D}_p. \label{DataCollect}
\end{align}
\end{subequations}
Note that 
\eqref{DPADMM-2-Prox} with $\tilde{\xi}^{t,e}_p = 0$ is equal to \eqref{ADMM-2-MLU}. 
We use $f'_p(z^{t,e}_p;\mathcal{D}_p)$ and $f'_p(z^{t,e}_p)$ interchangeably, where $\mathcal{D}_p$ is a given dataset.


\textbf{DP-IADMM.}
In Algorithm \ref{algo:DP-IADMM-Prox}, we present the proposed DP-IADMM with multiple local updates.
We describe the steps of the algorithm
as follows.
The computation at the central server is described in lines 1--9, while the local computation for each agent $p$ is described in lines 11--24.
In lines 2--3, the initial points are sent from the server to all agents.
In lines 5--6, the global parameter $w^{t+1}$ is computed and sent to the local agents.
In lines 15--22, the local agent $p$ receives $w^{t+1}$ from the server,  conducts local updates for $E$ times, and sends the resulting local model parameter $z_p^{t+1}$ to the server.
Note that $z_p^{t+1}$ is a randomized output:  it is perturbed by injecting random noise to the objective function of \eqref{DPADMM-2-Prox}.
The dual updates are performed at the server and the local agents individually as in line 8 and in line 23, respectively.
Note that those dual updates are identical since the initial points at the server and the local agents are the same.

\begin{algorithm}[!ht]
\caption{DP-IADMM with multiple local updates.}
\label{algo:DP-IADMM-Prox}
\begin{algorithmic}[1]
\STATE \textbf{(Server):}
\STATE Initialize $\lambda^1_1, \ldots, \lambda^1_P, z^1_1, \ldots, z^1_P$.
\STATE \texttt{Send} $\lambda^1_p, z^1_p$ to all agent $p \in [P]$ (\textbf{to line 12}).
\FOR{$t \in [T]$}
\STATE $w^{t+1} \leftarrow \frac{1}{P} \sum_{p=1}^P ( z_p^t - \frac{1}{\rho^t} \lambda^t_p) $.
\STATE \texttt{Send} $w^{t+1}$ to all agents (\textbf{to line 15}).
\STATE \texttt{Receive} $z^{t+1}_p$ from all agents (\textbf{from line 22}).
\STATE $\lambda^{t+1}_p \leftarrow \lambda^t_p + \rho^t(w^{t+1}-z^{t+1}_p)$ for all $p \in [P]$.
\ENDFOR
\STATE  
\STATE \textbf{(Agent $p \in [P]$):}  
\STATE Receive $\lambda^1_p, z^1_p$ from the server (\textbf{from line 3}).
\STATE Initialize $z^{0,E+1}_p = z^1_p$.
\FOR{$t \in [T]$}
\STATE \texttt{Receive} $w^{t+1}$ from the server (\textbf{from line 6}).    
\STATE Set $z_{p}^{t,1} \leftarrow z_p^{t-1,E+1}$
\FOR{$e \in [E]$}
\STATE Sample $\tilde{\xi}^{t,e}_p$ from \eqref{Laplace}.
\STATE Compute $z^{t,e+1}_p$ by solving \eqref{DPADMM-2-Prox}.  
\ENDFOR        
\STATE $z^{t+1}_p \leftarrow \frac{1}{E} \sum_{e=1}^E z_p^{t,e+1}$.    
\STATE \texttt{Send} $z^{t+1}_p$ to the server (\textbf{to line 7}).
\STATE $\lambda^{t+1}_p \leftarrow \lambda^t_p + \rho^t(w^{t+1}-z^{t+1}_p)$.
\ENDFOR
\end{algorithmic}
\end{algorithm}
 
The benefits of Algorithm \ref{algo:DP-IADMM-Prox} include that (i) the quality of the solution can be improved via the multiple local updates that could result in reducing the total number of iterations,
(ii) the amount of communication is reduced by excluding the communication of the dual information, and
(iii) $\bar{\epsilon}$-DP on data is guaranteed for any communication rounds, which will be proved in the next section.

\subsection{Privacy Analysis} \label{sec:privacy}
In this section we show that $\bar{\epsilon}$-DP in Definition \ref{def:differential_privacy_1} is guaranteed for any iteration of Algorithm \ref{algo:DP-IADMM-Prox}.
To this end, using the following lemma, we show that the constrained subproblem \eqref{DPADMM-2-Prox} provides $\bar{\epsilon}$-DP.

\begin{lemma}{\textbf{(Theorem 1 in \cite{kifer2012private})}} \label{lemma:theorem1}
Let $\mathcal{A}$ be a randomized algorithm induced by the random noise $\tilde{\xi}$ that provides output $\phi(\mathcal{D},\tilde{\xi})$.
Let $\{\mathcal{A}_{\ell}\}_{\ell=1}^{\infty}$ be a sequence of randomized algorithms, each of which is induced by $\tilde{\xi}$ and provides output $\phi^{\ell}(\mathcal{D},\tilde{\xi})$.
If $\mathcal{A}_{\ell}$ is $\bar{\epsilon}$-DP for all $\ell$ and satisfies a pointwise convergence condition, namely, $\lim_{{\ell} \rightarrow \infty} \phi^{\ell}(\mathcal{D},\tilde{\xi}) = \phi(\mathcal{D},\tilde{\xi})$, then $\mathcal{A}$ is also $\bar{\epsilon}$-DP.
\end{lemma}

For the rest of this section we fix $t \in [T]$, $e \in [E]$, and $p \in [P]$. 
For ease of exposition, we denote the objective function of \eqref{DPADMM-2-Prox}, which is strongly convex, by  
\begin{align}
  G^{t,e}_p(z_p) := & \langle f'_p(z^{t,e}_p), z_p \rangle   + \textstyle \frac{1}{2\eta^{t}} \| z_p - z_p^{t,e}\|^2 \nonumber \\
  & + \textstyle \frac{\rho^t}{2} \|w^{t+1}-z_p + \frac{1}{\rho^t}(\lambda^t_p - \tilde{\xi}^{t,e}_p) \|^2 \label{fn_G}
\end{align}
and the feasible region of \eqref{DPADMM-2-Prox} by
\begin{align*}
     \mathcal{W} = \{ z_p \in \mathbb{R}^{J \times K} : h_m (z_p) \leq 0, \ \forall m \in [M] \}, 
\end{align*}
where $h_m$ is convex and twice continuously differentiable and $M$ is the total number of inequalities.

By utilizing an indicator function $\mathcal{I}_{\mathcal{W}}(z_p)$ that outputs zero if $z_p \in \mathcal{W}$ and $\infty$ otherwise, \eqref{DPADMM-2-Prox} can be expressed by the following problem:
\begin{align*}
\min_{z_p \in \mathbb{R}^{J \times K}} \ G^{t,e}_p (z_p) + \mathcal{I}_{\mathcal{W}}(z_p).
\end{align*}
We note that the indicator function can be approximated by the following function:
\begin{align}
& g_{\ell} (z_p) :=  \sum_{m=1}^M \ln ( 1+ e^{\ell h_m(z_p)}), \label{logfn-prox}  
\end{align}
where $\ell > 0$. 
Increasing $\ell$ enforces the feasibility, namely, $h_m(z_p) \leq 0$, resulting in $g_{\ell}(z_p) \rightarrow 0$.
It is similar to the logarithmic barrier function (LBF), namely $-(1/\ell) \sum_{m=1}^M \ln (-h_m(z_p))$, in that the approximation becomes closer to the indicator function as $\ell \rightarrow \infty$.
However, the penalty function $g_{\ell}$ is different from LBF in that the domain of $z_p$ is not restricted.
By replacing the indicator function with the penalty function in \eqref{logfn-prox}, we construct the following \textit{unconstrained} problem:
\begin{align}
z^{t,e+1}_{p\ell} \leftarrow  \argmin_{z_p \in \mathbb{R}^{J \times K}} \ & G^{t,e}_p (z_p)  + g_{\ell}(z_p), \label{ADMM-2-Prox-log}
\end{align}
where the objective function is strongly convex because $g_{\ell}$ is convex over all domains and $G^{t,e}_p$ is strongly convex. Therefore, $z^{t,e+1}_{p\ell}$ is the unique optimal solution.
We first show that \eqref{ADMM-2-Prox-log} satisfies the pointwise convergence condition and provides $\bar{\epsilon}$-DP as in Propositions \ref{prop:pointwise_convergence} and \ref{prop:dpinapproximation}, respectively.
\begin{proposition}\label{prop:pointwise_convergence}
  For fixed $t$, $e$, and $p$,  $\lim_{\ell \rightarrow \infty} z^{t,e+1}_{p \ell}  = z_p^{t,e+1}$, where $z_p^{t,e+1}$ and $z^{t,e+1}_{p \ell}$ are from \eqref{DPADMM-2-Prox} and \eqref{ADMM-2-Prox-log}, respectively.
\end{proposition}
\begin{proof}
See Appendix \ref{apx-prop:pointwise_convergence}.
\end{proof}

\begin{proposition} \label{prop:dpinapproximation}
  For fixed $t$, $e$, $p$, $\ell$, and the dataset $\mathcal{D}_p$, we denote by $z^{t,e+1}_{p\ell}(\mathcal{D}_p)$ the optimal solution of \eqref{ADMM-2-Prox-log}. It provides $\bar{\epsilon}$-DP that satisfies
  \begin{align}
    \label{DP_1}
  \textstyle  \Big| \ln \Big(\frac{\mathbb{P} ( z^{t,e+1}_{p \ell} (\mathcal{D}_p )\in \mathcal{S} )}{\mathbb{P} ( z^{t,e+1}_{p \ell} (\mathcal{D}'_p )\in \mathcal{S} )} \Big)\Big|   \leq  \bar{\epsilon} ,
  \end{align}
  for all $\mathcal{S} \subset \mathbb{R}^{J \times K}$ and $\mathcal{D}'_p \in \widehat{\mathcal{D}}_p$, where $\widehat{\mathcal{D}}_p$ is from \eqref{DataCollect}.
\end{proposition}
\begin{proof}
See Appendix \ref{apx-prop:dpinapproximation}.
\end{proof}

Based on Propositions \ref{prop:pointwise_convergence} and \ref{prop:dpinapproximation}, Lemma \ref{lemma:theorem1} can be used for proving the following theorem.
\begin{theorem} \label{thm:privacy}
  For fixed $t$, $e$, $p$, and the dataset $\mathcal{D}_p$, we denote by $z^{t,e+1}_p(\mathcal{D}_p)$ the optimal solution of \eqref{DPADMM-2-Prox}. 
  It provides $\bar{\epsilon}$-DP that satisfies
  \begin{align*}
  \textstyle  \Big| \ln \Big(\frac{\mathbb{P} ( z^{t,e+1}_{p} (\mathcal{D}_p )\in \mathcal{S} )}{\mathbb{P} ( z^{t,e+1}_{p} (\mathcal{D}'_p )\in \mathcal{S} )} \Big)\Big|   \leq  \bar{\epsilon}, 
  \end{align*}
  for all $\mathcal{S} \subset \mathbb{R}^{J \times K}$ and $\mathcal{D}'_p \in \widehat{\mathcal{D}}_p$, where $\widehat{\mathcal{D}}_p$ is from \eqref{DataCollect}.
\end{theorem}
\begin{remark}
Theorem \ref{thm:privacy} shows that $\bar{\epsilon}$-DP is guaranteed for every iteration of Algorithm \ref{algo:DP-IADMM-Prox}. This result can be extended by introducing the existing composition theorem in \cite{dwork2014algorithmic} to ensure $\bar{\epsilon}$-DP for the entire process of the algorithm.
\end{remark}

\subsection{Convergence Analysis} \label{sec:convergence}
In this section we show that a sequence of iterates generated by Algorithm \ref{algo:DP-IADMM-Prox} converges to an optimal solution of \eqref{ERM_1} in \textit{expectation} under the following assumptions.
\begin{assumption}\label{assump:convergence}  
  \noindent
  \begin{enumerate}
    \item[(i)] $\rho^t$ in \eqref{DPADMM-2-Prox} satisfies $\rho^1 \leq \rho^2 \leq \ldots \leq \rho^T \leq \rho^{\text{max}}$.
    \item[(ii)] $\exists \gamma > 0 : \gamma \geq 2\| \lambda^*\|$, where $\lambda^*$ is a dual optimal. 
    \item[(iii)] $f_p$ in \eqref{def_fn_f} is $H$-Lipschitz over a set $\mathcal{W}$ with respect to the Euclidean norm.    
  \end{enumerate}  
\end{assumption}
Assumption \ref{assump:convergence} is typically used for the convergence analysis of ADMM and IADMM (see Chapter 15 of \cite{beck2017first}). 
We adopt the assumptions because IADMM is a special case of Algorithm \ref{algo:DP-IADMM-Prox} by setting $E=1$ and $\tilde{\xi}^{t,e}_p=0$.

Based on Assumption \ref{assump:convergence} (iii) used for bounding subgradients, we define the following bounds (see Appendix \ref{apx:existence_of_UBs} for details):
\begin{align} 
& \textstyle U_1 := \max_{u \in \mathcal{W}, p \in [P]} \| f'_p(u; \mathcal{D}_p) \|, \label{def_upper} \\
& \textstyle  U_2 := \max_{u, v \in \mathcal{W} } \|u-v\|, \nonumber \\
& \textstyle  U_3 := \max_{u \in \mathcal{W}, p \in [P], \mathcal{D}'_p \in \widehat{\mathcal{D}}_p} \| f'_p(u;\mathcal{D}_p) - f'_p(u;\mathcal{D}'_p)\|_1. \nonumber
\end{align}
In what follows, we show that the rate of convergence in expectation produced by Algorithm \ref{algo:DP-IADMM-Prox} is 
\begin{enumerate}
  \item[C1.] $O(\frac{1}{\bar{\epsilon} \sqrt{T}})$ when $f_p$ is smooth in Theorem \ref{thm:smooth};
  \item[C2.] $O(\frac{1}{\bar{\epsilon}^2 \sqrt{T}})$ when $f_p$ is nonsmooth in Theorem \ref{thm:nonsmooth}; and
  \item[C3.] $O(\frac{1}{\bar{\epsilon}^2 T})$ when $f_p$ is strongly convex in Theorem \ref{thm:strong}.
\end{enumerate}
The result in C3 requires additional assumptions:
\begin{assumption}\label{assump:convergence_stronglyconvex}
\noindent
\begin{enumerate}
\item[(i)] $\exists \ \gamma > 0: \| \lambda^t \| \leq \gamma, \ \forall t$.
\item[(ii)] $\rho^t \leq \frac{t}{t-1} \rho^{t-1}, \ \forall t$.
\end{enumerate}
\end{assumption}
Assumption \ref{assump:convergence_stronglyconvex} (i) can be strict in practice. As indicated in \cite{azadi2014towards}, however, it can be considered as a price that we have to pay for faster convergence (see Assumption 3 in \cite{azadi2014towards}).
Assumption \ref{assump:convergence_stronglyconvex} (ii) is not strict since it is satisfied with a constant penalty method (i.e., $\rho^t = \rho, \ \forall t$) that is commonly considered in the literature (e.g.,~\cite{beck2017first}).

\begin{theorem} \label{thm:smooth}
  Suppose that Assumption \ref{assump:convergence} holds, $f_p$ (defined in \eqref{def_fn_f}) is $L$-smooth convex, and 
  \begin{align}
    \eta^t  = 1/(L + \sqrt{t}/\bar{\epsilon}), \ \forall t, \label{smooth_proximity}
  \end{align}
  where $\bar{\epsilon} > 0$ is from \eqref{Laplace}.
  Then we have
  \begin{subequations}    
  \begin{align}
    & \mathbb{E} \Big[ F(z^{(T)}) - F(z^*) + \gamma \| Aw^{(T)} - z^{(T)} \| \Big] \leq R^{\text{S}} (\sqrt{T}, \bar{\epsilon})  \nonumber \\
    & + \frac{U_2^2  (\rho^{\text{max}} +L/E) +  (\gamma + \|\lambda^1\|)^2 / \rho^1 }{2T}, \label{smooth_inequality}
  \end{align}
  where 
  \begin{align}    
    R^{\text{S}} (\sqrt{T}, \bar{\epsilon}) :=  \frac{2 PJKU_3^2 + U_2^2/(2E) }{ \bar{\epsilon} \sqrt{T} }, \label{smooth_rhs_term}
  \end{align}
  \end{subequations}  
  $U_2$, $U_3$ are from \eqref{def_upper}, $z^*$ is an optimal solution of \eqref{ERM_1},
  \begin{subequations}
  \label{Common_Notation}   
  \begin{align}    
    & F(z) := \textstyle \sum_{p=1}^P f_p(z_p), \\    
    & A^{\top} := \begin{bmatrix}
        \mathbb{I}_J \ \cdots \  \mathbb{I}_J        
      \end{bmatrix}_{J \times PJ}, \\
    & w^{(T)} := \textstyle \frac{1}{T} \sum_{t=1}^T w^{t+1},                   
  \end{align}  
  \end{subequations}
  and
  \begin{align}
  z^{(T)} := \textstyle \frac{1}{TE} \sum_{t=1}^T \sum_{e=1}^E z^{t,e+1}. \label{smooth_z}
  \end{align}
\end{theorem}
\begin{proof}
See Appendix \ref{apx-thm:smooth}.
\end{proof}
According to Theorem 3.60 in \cite{beck2017first}, the inequality \eqref{smooth_inequality} derived under Assumption \ref{assump:convergence} implies that the rate of convergence in expectation is 
$\mathcal{O}(1/ (\bar{\epsilon} \sqrt{T}))$ for $\bar{\epsilon} \in (0, \infty)$, 
while in a nonprivate setting it is $\mathcal{O}(1/T)$ because $R^{\text{S}} (\sqrt{T}, \bar{\epsilon})$ in \eqref{smooth_rhs_term} is zero when $\bar{\epsilon}=\infty$.
 
\begin{theorem} \label{thm:nonsmooth}
Suppose that Assumption \ref{assump:convergence} holds, $f_p$ (defined in \eqref{def_fn_f}) is a nonsmooth convex function, and
\begin{align}
\eta^{t} = 1/\sqrt{t}, \ \forall t.  \label{nonsmooth_proximity}
\end{align}
Then we have
\begin{subequations}
\begin{align}
& \mathbb{E} \Big[ F(z^{(T)}) - F(z^*) + \gamma \| Aw^{(T)} - z^{(T)} \| \Big] \leq R^{\text{NS}}(\sqrt{T}, \bar{\epsilon}) \nonumber \\
& + \frac{U_2^2 \rho^{\text{max}} +  (\gamma + \|\lambda^1\|)^2 / \rho^1 + 2\gamma U_2}{2T},  \label{nonsmooth_inequality}
\end{align}  
where 
\begin{align}
  R^{\text{NS}} (\sqrt{T},\bar{\epsilon}) := \frac{ 2PJKU_3^2 /\bar{\epsilon}^2 + PU_1^2 + U_2^2 / (2E) }{ \sqrt{T} }, \label{nonsmooth_rhs_term}
\end{align}
\end{subequations}
$U_1$, $U_2$, $U_3$ are from \eqref{def_upper}, $z^*$ is an optimal solution of \eqref{ERM_1}, $F(z), A, w^{(T)}$ are from \eqref{Common_Notation}, and
\begin{align}
z^{(T)} := \textstyle \frac{1}{TE} \sum_{t=1}^T \sum_{e=1}^E z^{t,e}. \label{nonsmooth_z}
\end{align}
\end{theorem}
\begin{proof}
  See Appendix \ref{apx-thm:nonsmooth}.
\end{proof}
The inequality \eqref{nonsmooth_inequality} derived under Assumption \ref{assump:convergence} implies that the rate of convergence in expectation is 
$\mathcal{O}(1/ (\bar{\epsilon}^2 \sqrt{T}))$ for $\bar{\epsilon} \in (0, \infty)$, 
while in a nonprivate setting it is $\mathcal{O}(1/\sqrt{T})$ because $R^{\text{NS}} (\sqrt{T}, \bar{\epsilon})$ in \eqref{nonsmooth_rhs_term} is $(PU_1^2 + U_2^2/(2E))/\sqrt{T}$ when $\bar{\epsilon}=\infty$. 

\begin{theorem} \label{thm:strong}
Suppose that Assumptions \ref{assump:convergence} and \ref{assump:convergence_stronglyconvex} hold, $f_p$ (defined in \eqref{def_fn_f}) is $\alpha$-strongly convex, and
\begin{align}
\eta^{t} = 2/(\alpha(t+2)), \ \forall t. 
\end{align}
Then we have
\begin{align}
  & \mathbb{E} \Big[ F(z^{(T)}) - F(z^*) + \gamma \| Aw^{(T)} - z^{(T)} \| \Big] \leq  \nonumber \\
  &\frac{1}{T+1} \Big\{2 U_2 \gamma + U_2^2 \rho^{\text{max}} + 4 \gamma^2 / \rho^1 + \alpha U_2^2/(2E)  \nonumber \\
  & + 2P(U_1^2 + 2JKU_3^2/\bar{\epsilon}^2)/\alpha \Big\}, \label{strong_inequality}
\end{align}
where $U_1$, $U_2$, $U_3$ are from \eqref{def_upper}, $z^*$ is an optimal solution, $F(z)$, $A$ are from \eqref{Common_Notation}, and 
\begin{align*}
  & w^{(T)} := \textstyle \frac{2}{T(T+1)} \sum_{t=1}^T t w^{t+1}, \\ 
  & z^{(T)} := \textstyle \frac{2}{T(T+1)} \sum_{t=1}^T t  (\frac{1}{E} \sum_{e=1}^E z^{t,e}).    
\end{align*}  
\end{theorem}
\begin{proof}
  See Appendix \ref{apx-thm:strong}. 
\end{proof}
The inequality \eqref{strong_inequality} derived under Assumptions \ref{assump:convergence} and \ref{assump:convergence_stronglyconvex} implies that the rate of convergence in expectation is 
$\mathcal{O}(1/ (\bar{\epsilon}^2 T))$ for $\bar{\epsilon} \in (0, \infty)$, 
while in a nonprivate setting it is $\mathcal{O}(1/T)$ when $\bar{\epsilon}=\infty$.

\begin{corollary}{\textbf{(Effect of the multiple local update)}} \label{cor:MultipleLocalUpdate}
Increasing the number $E$ of local updates decreases the values on the right-hand side of \eqref{smooth_inequality}, \eqref{nonsmooth_inequality}, and \eqref{strong_inequality}.
This implies that the gap between $F(z^{(T)})$ and $F(z^*)$ can become smaller by increasing $E$ for fixed $T$.
This may result in greater learning performance by introducing the multiple local updates, which will be numerically demonstrated in Section \ref{sec:experiments}.
\end{corollary}

\section{Numerical Experiments} \label{sec:experiments}
In this section we compare Algorithm \ref{algo:DP-IADMM-Prox} with the state of the art in \cite{huang2019dp} as a baseline algorithm. The algorithm in \cite{huang2019dp} has demonstrated more accurate solutions than have the other existing DP algorithms, such as DP-SGD \cite{abadi2016deep}, DP-ADMM with the output perturbation method (Algorithm 2 in \cite{huang2019dp}), and DP-ADMM with the objective perturbation method \cite{zhang2016dynamic} (see Figure 6 in \cite{huang2019dp}).
Note that as a DP technique, the output perturbation method is used in the baseline algorithm in \cite{huang2019dp}, whereas the objective perturbation method is used in our algorithm.

We implemented the algorithms in Python and ran the experiments on Swing, a 6-node GPU computing cluster at Argonne National Laboratory. Each node of Swing has 8 NVIDIA A100 40 GB GPUs, as well as 128 CPU cores.
The implementation is available at \url{https://github.com/APPFL/DPFL-IADMM-Classification.git}.

\textbf{Algorithms.}
We denote 
\begin{itemize}
  \item the baseline algorithm in \cite{huang2019dp} by \texttt{OutP},
  \item Algorithm \ref{algo:DP-IADMM-Prox} with $E=1$ by \texttt{ObjP}, and
  \item Algorithm \ref{algo:DP-IADMM-Prox} with $E=10$ by \texttt{ObjPM}.
\end{itemize}
Note that \texttt{OutP} and \texttt{ObjP} are equivalent in a nonprivate setting.

\textbf{FL Model.}
We consider a multiclass logistic regression model (see Appendix \ref{apx:model} for details).

\textbf{Datasets.}
We consider two publicly available datasets for image classification: MNIST \cite{lecun1998mnist} and FEMNIST \cite{caldas2018leaf}.
For the MNIST dataset, we split the 60,000 training data points over $P=10$ agents, each of which is assigned to have the same number of independent and identically distributed (IID) dataset.
For the FEMNIST dataset, we follow the preprocess procedure\footnote{https://github.com/TalwalkarLab/leaf/tree/master/data/femnist} to sample 5\% of the entire 805,263 data points in a non-IID manner, resulting in 36,708 training samples distributed over $P=195$ agents.


\textbf{Parameters.}
Under the multiclass logistic regression model, we can compute $\bar{\Delta}_p^{t,e}$ in \eqref{Delta} as
\begin{align*}
\bar{\Delta}_p^{t,e} = \max_{i^* \in [I_p]} \sum_{j=1}^J \sum_{k=1}^K \Big| \frac{1}{I} \big\{ x_{pi^*j} \big(h_k(z^{t,e}_p;x_{pi^*}) - y_{pi^* k} \big)  \big\} \Big|.
\end{align*}
Note that $\bar{\Delta}_p^{t,e}/\bar{\epsilon}$ is proportional to the standard deviation of the Laplace distribution in \eqref{Laplace-pdf}, thus controlling the noise level.
In the experiments, we consider various $\bar{\epsilon} \in \{ 0.05, 0.1, 1, 5 \}$, where stronger data privacy is achieved with smaller $\bar{\epsilon}$.

We emphasize that the baseline algorithm \texttt{OutP} guarantees $(\bar{\epsilon},\bar{\delta})$-DP, which provides stronger privacy as $\bar{\delta} > 0$ decreases for fixed $\bar{\epsilon}$, but still weaker than $\bar{\epsilon}$-DP.
In the experiments, we set $\bar{\delta}=10^{-6}$ for \texttt{OutP}.
In addition, we set the regularization parameter $\beta$ in \eqref{def_fn_f} by $\beta \leftarrow 10^{-6}$, as in \cite{huang2019dp}.

The parameter $\rho^t$ in Assumption \ref{assump:convergence} affects the learning performance because it controls the proximity of the local model parameters from the global model parameters.
For all algorithms, we set $\rho^t \leftarrow \hat{\rho}^t$ given by
\begin{align}
  & \hat{\rho}^t := \min \{ 10^9, \ c_1 (1.2)^{\lfloor t/T_c \rfloor} + c_2/\bar{\epsilon}\}, \ \forall t \in [T], \label{dynamic_rho}
\end{align}
where (i) $c_1=2$, $c_2=5$, and $T_c=10000$ for MNIST and (ii) $c_1=0.005$, $c_2=0.05$, and $T_c=2000$ for FEMNIST, which are chosen based on the justifications described in Appendix \ref{apx-hyperparameter-rho}. 
Note that the chosen parameter $\hat{\rho}^t$ is nondecreasing and bounded above, thus satisfying Assumption \ref{assump:convergence} (i).

\subsection{Comparison of Testing Errors} \label{sec:testing_errors}
Using the MNIST and FEMNIST datasets, we compare testing errors produced by \texttt{OutP}, \texttt{ObjP}, and \texttt{ObjPM} under various $\bar{\epsilon}$. 
We note that the testing errors produced by a nonprivate IADMM (i.e., Algorithm \ref{algo:DP-IADMM-Prox} with $\bar{\epsilon}=\infty$) with the multiclass logistic regression model on MNIST and FEMNIST are $9.1\%$ and $37.27\%$, respectively.

For each dataset and a given $\bar{\epsilon}$, we collect the testing errors for $10$ runs, each of which has different realizations of the random noises, but all of which guarantee the $\bar{\epsilon}$-DP on data.
In Figure \ref{fig:Testing_errors} we report the testing errors on average (solid line) with the $20$- and $80$-percentile confidence bounds (shaded) for every iteration $t \in [20000]$. 
The subfigures on the top and bottom rows are the testing error results for MNIST and FEMNIST, respectively.

\begin{figure*}[!ht]  
  \centering
  \begin{subfigure}[b]{0.24\textwidth}
      \centering      
      \includegraphics[width=\textwidth]{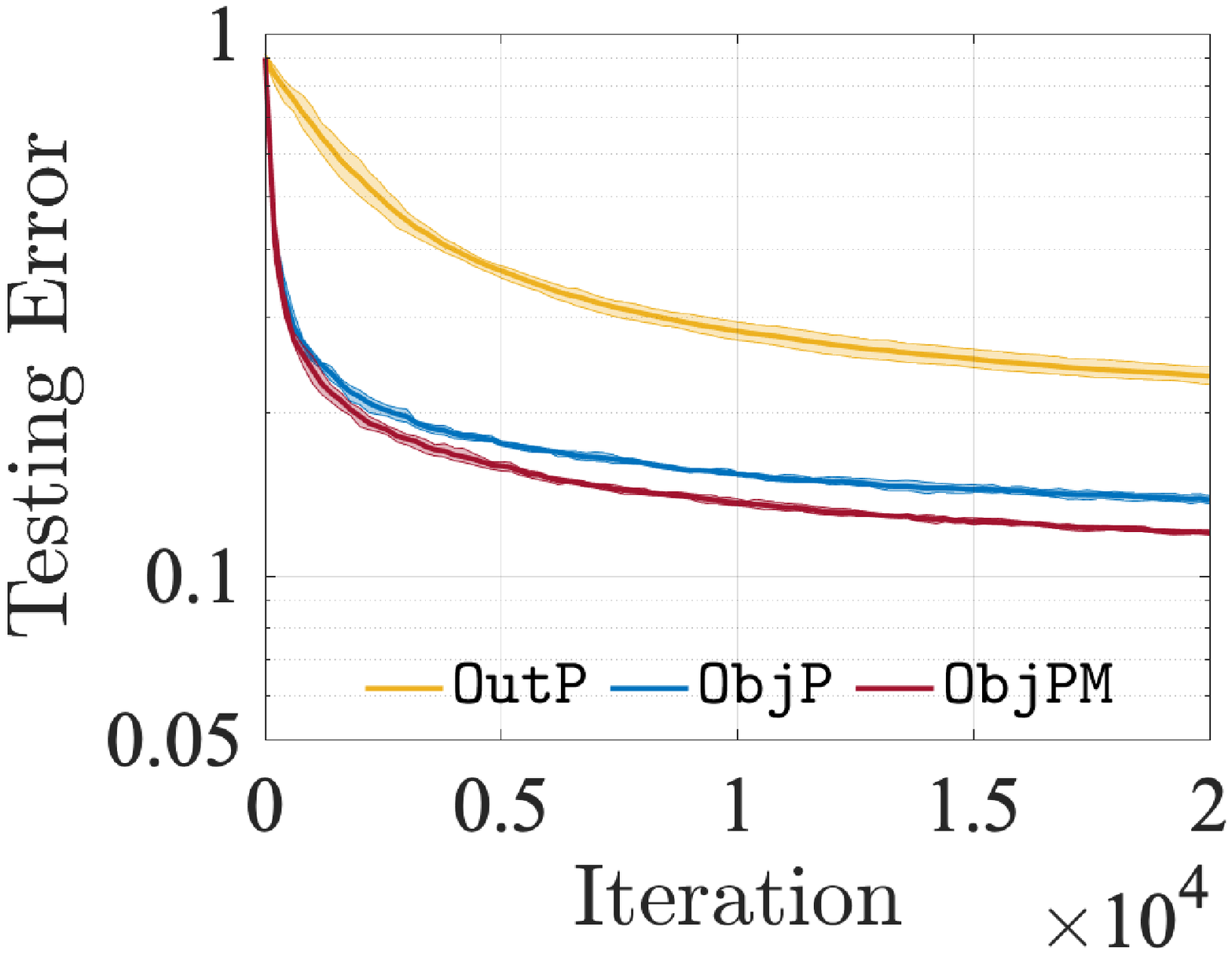}
      \includegraphics[width=\textwidth]{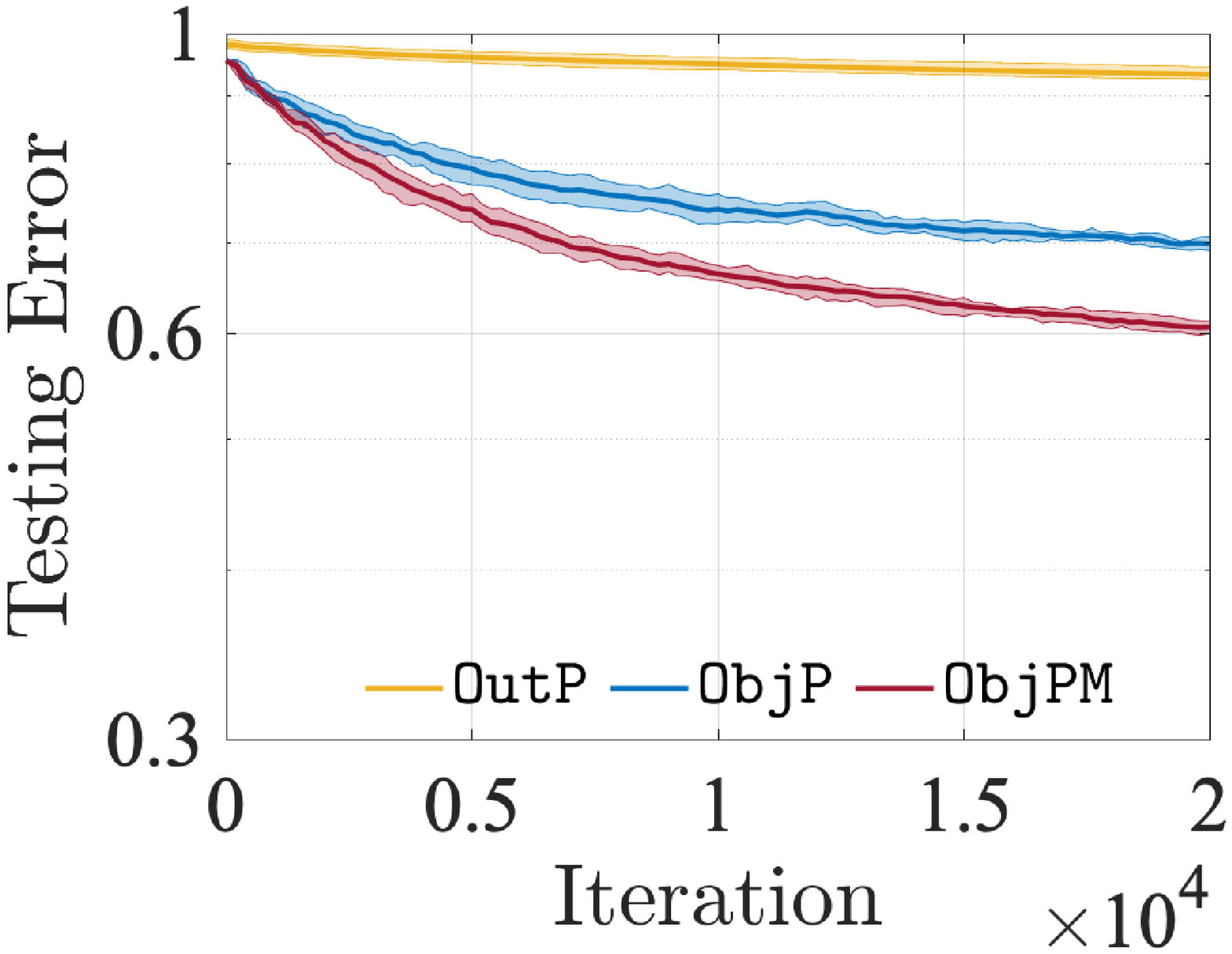}
      \caption{$\bar{\epsilon}=0.05$}
  \end{subfigure}
  \begin{subfigure}[b]{0.24\textwidth}
      \centering      
      \includegraphics[width=\textwidth]{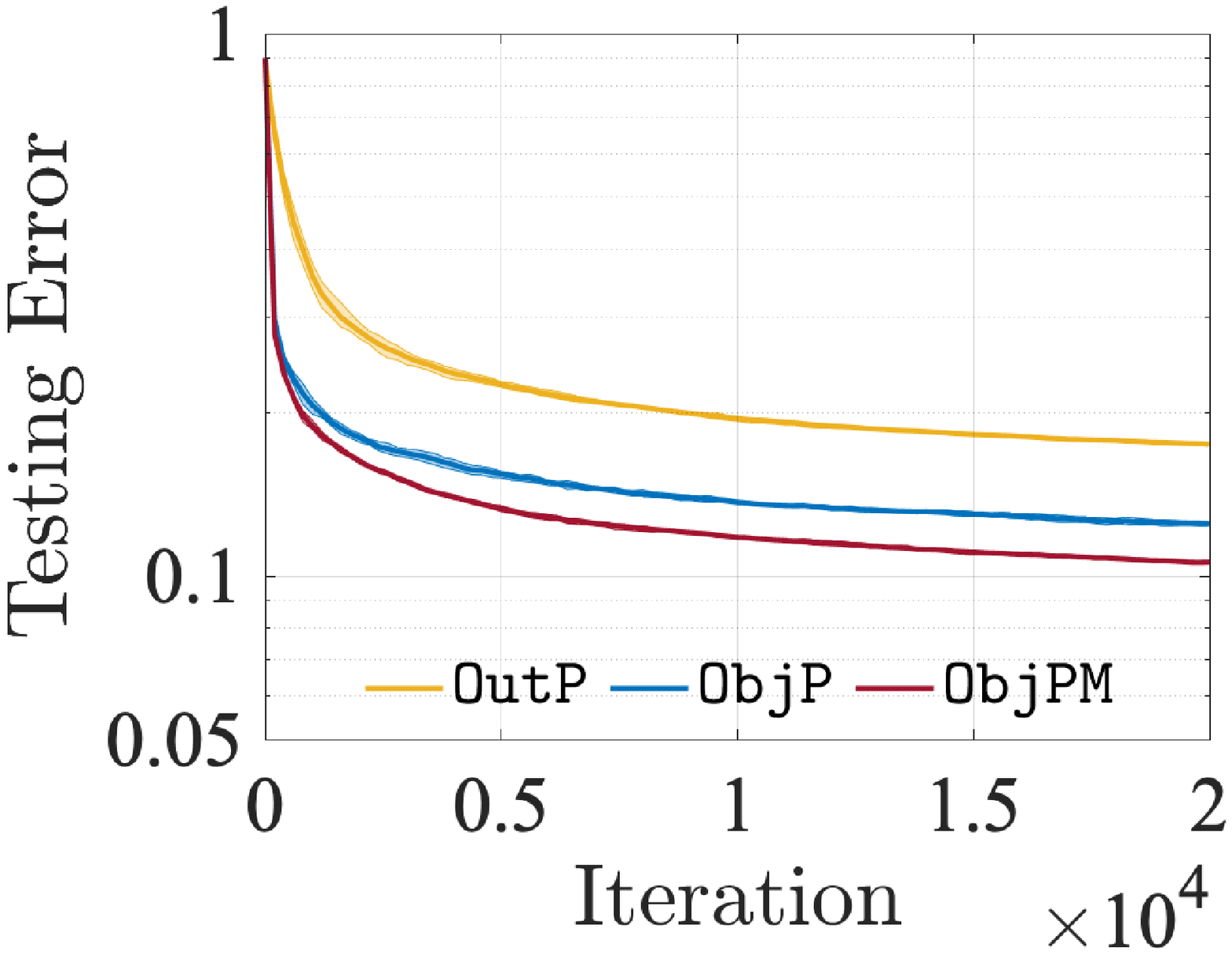}
      \includegraphics[width=\textwidth]{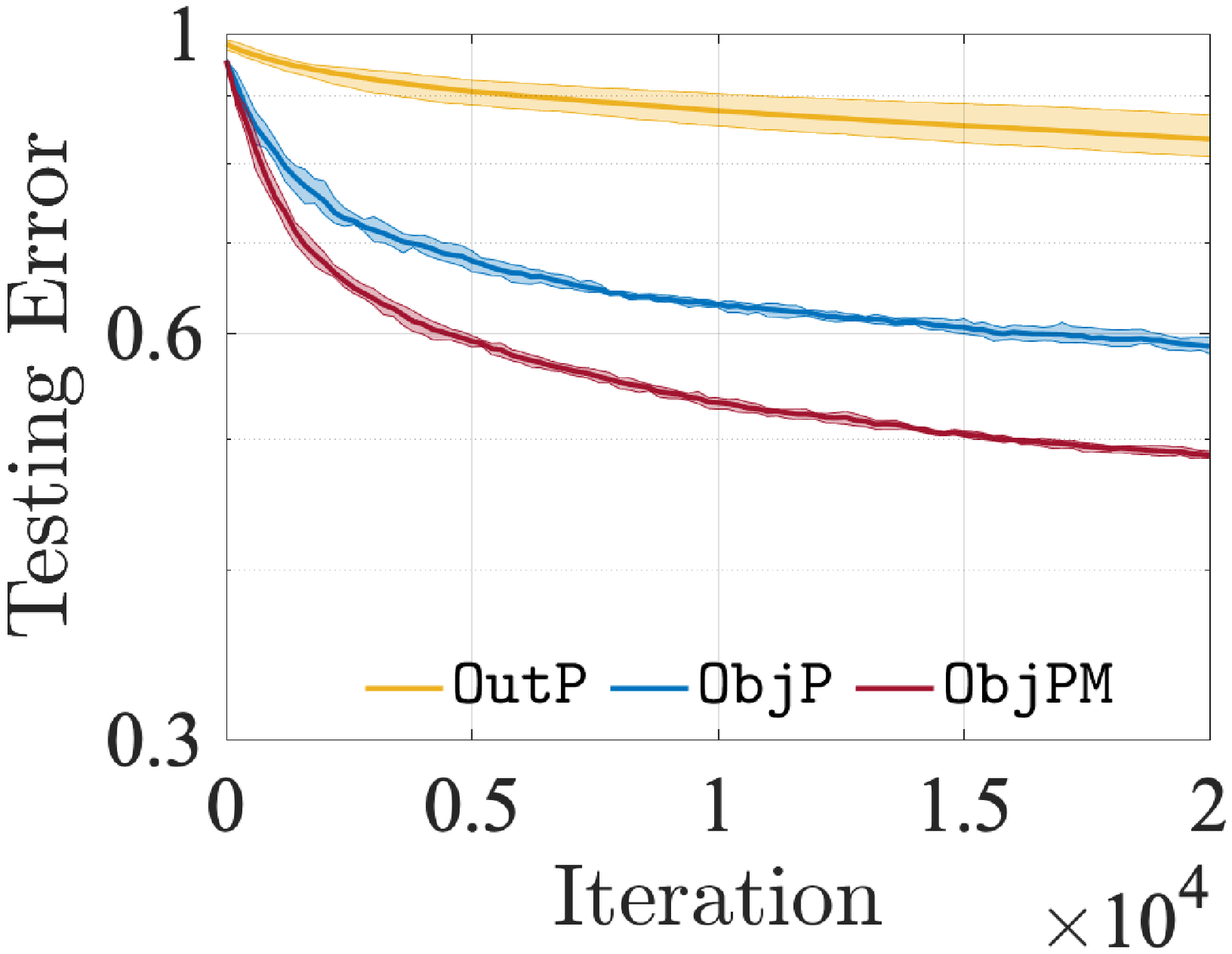}
      \caption{$\bar{\epsilon}=0.1$}
  \end{subfigure}
  \begin{subfigure}[b]{0.24\textwidth}
      \centering      
      \includegraphics[width=\textwidth]{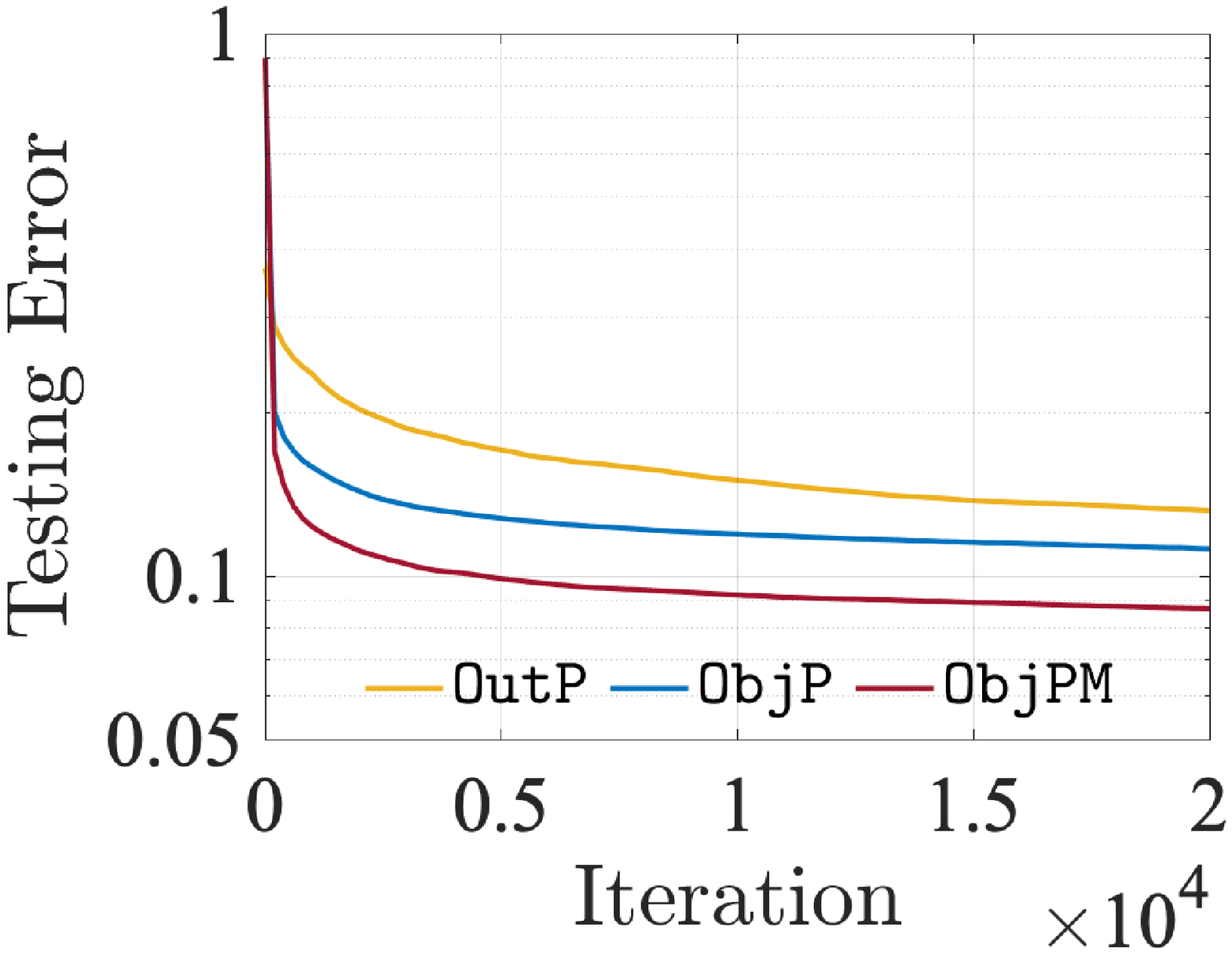}
      \includegraphics[width=\textwidth]{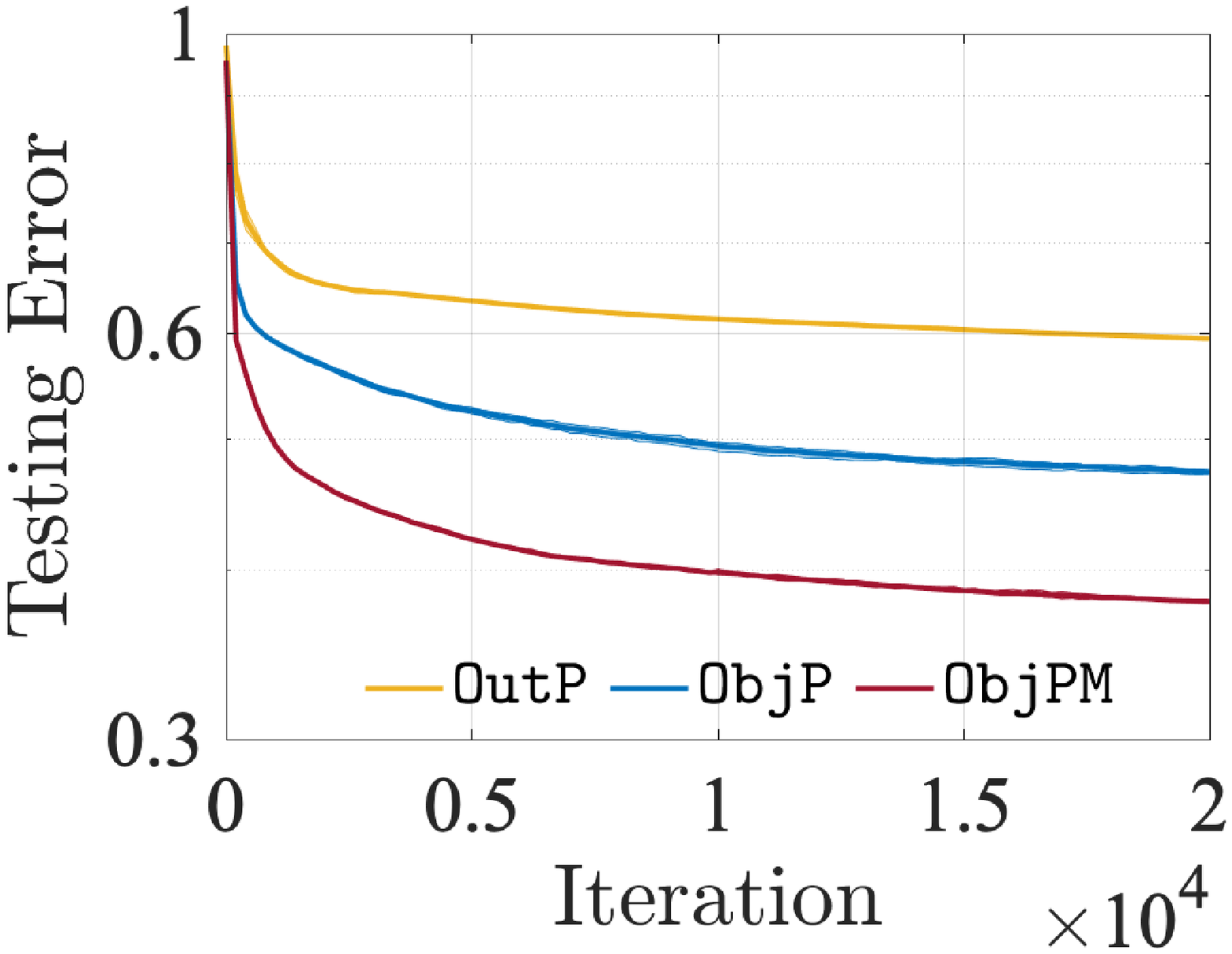}
      \caption{$\bar{\epsilon}=1$}
  \end{subfigure}  
  \begin{subfigure}[b]{0.24\textwidth}
    \centering      
    \includegraphics[width=\textwidth]{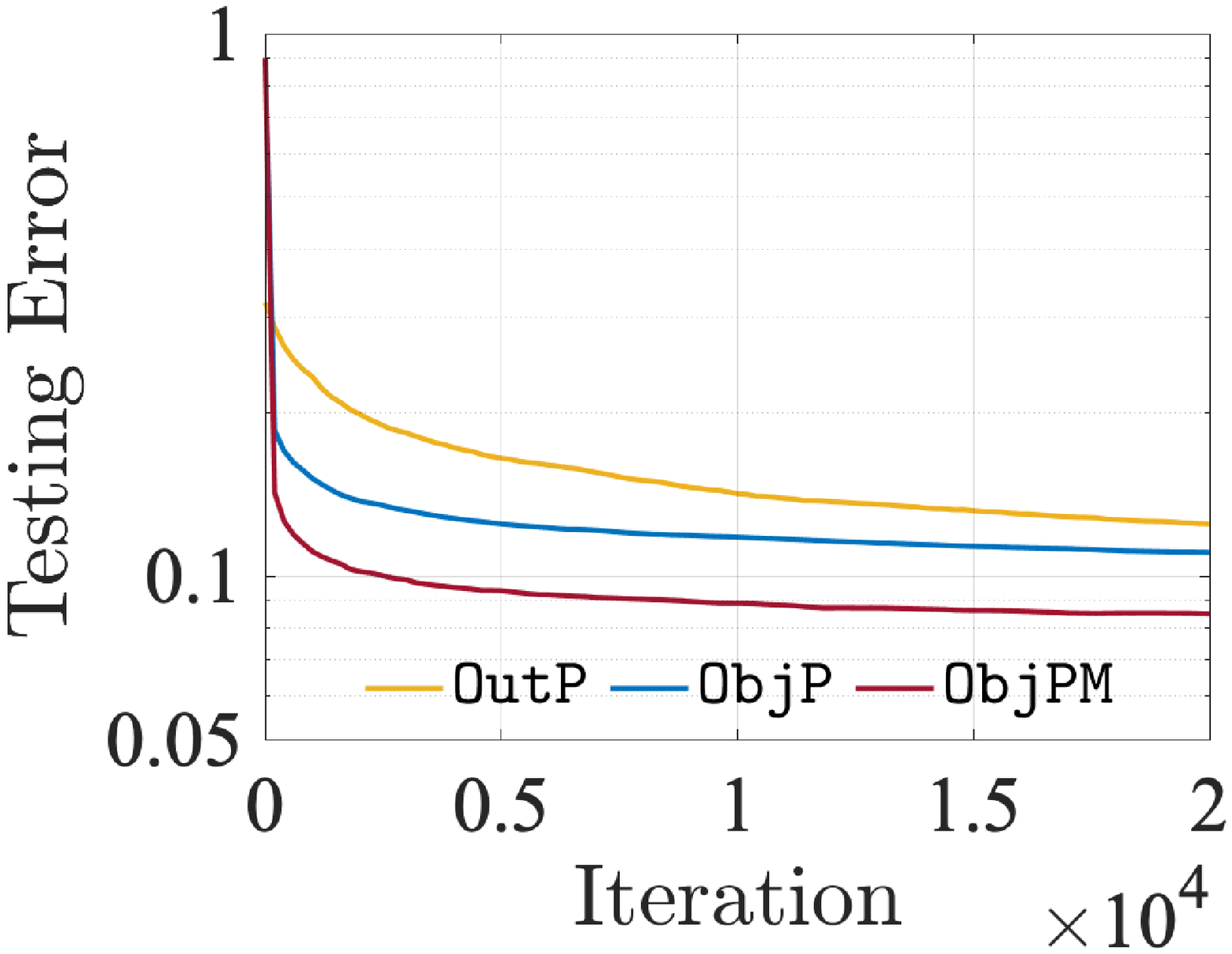}
    \includegraphics[width=\textwidth]{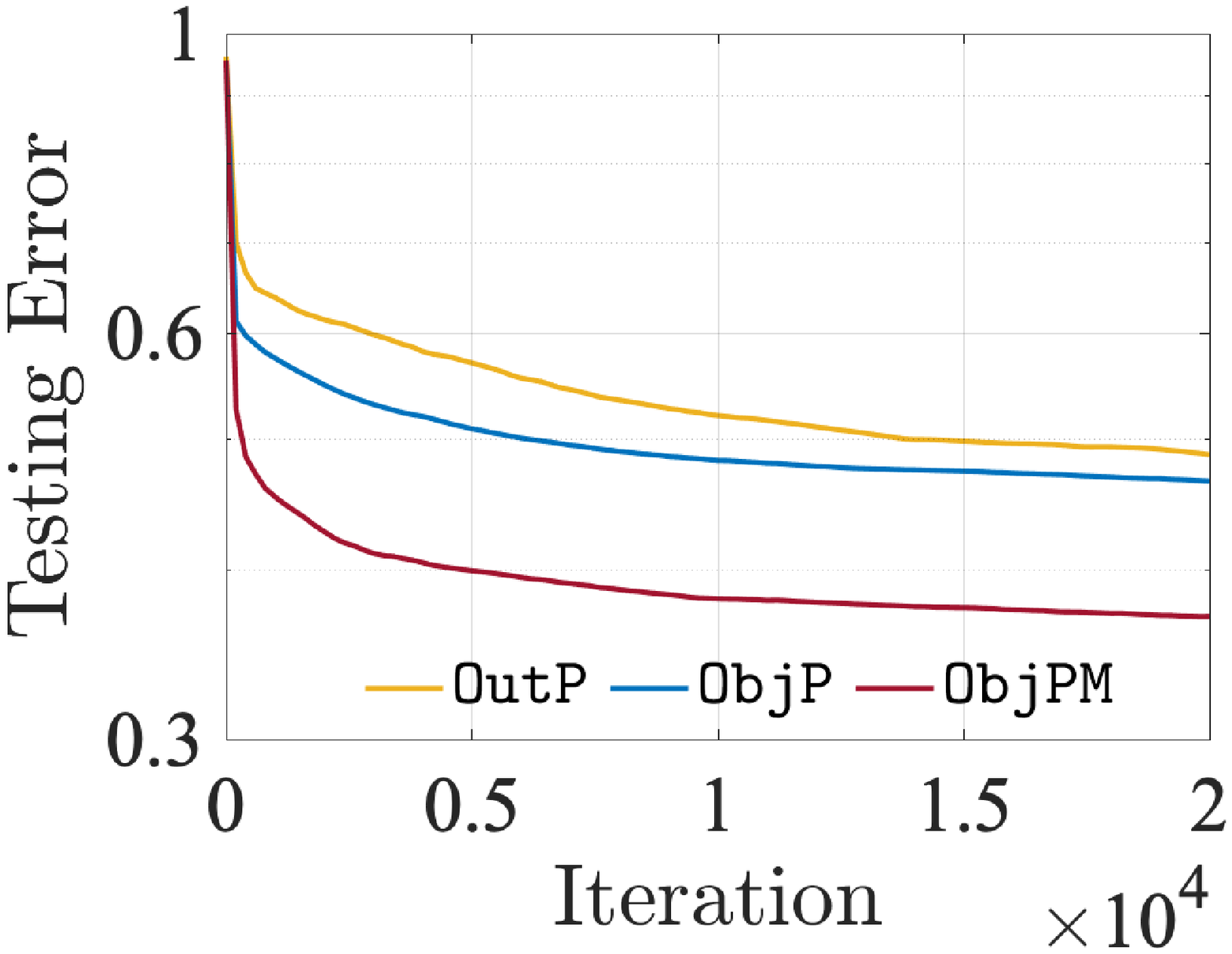}
    \caption{$\bar{\epsilon}=5$}
\end{subfigure}  
  \caption{Testing errors for every iteration under various $\bar{\epsilon}$ (top: MNIST; bottom: FEMNIST).}
  \label{fig:Testing_errors}
\end{figure*}

In what follows, we present some observations from the figures and their implications.

\begin{itemize}
  \item The testing errors of all algorithms increase as $\bar{\epsilon}$ decreases (i.e., stronger data privacy). This indicates the trade-off between data privacy and learning performance, well known in the literature on DP algorithms \cite{dwork2014algorithmic}.
  
  \item The testing errors of \texttt{ObjP} are lower than those of \texttt{OutP}. This result is consistent with the findings in \cite{chaudhuri2011differentially,zhang2016dynamic} that the better performance of the objective perturbation than the output perturbation is guaranteed with higher probability.

  \item The testing errors of \texttt{ObjPM} are lowest,  demonstrating the effectiveness of the multiple local updates presented in Corollary \ref{cor:MultipleLocalUpdate}. When $\bar{\epsilon}=1$, \texttt{ObjPM} produces  testing errors close to those of the nonprivate IADMM while the other algorithms do not. This result implies that \texttt{ObjPM} can mitigate the trade-off between data privacy and learning performance. 
  
  \item When $\bar{\epsilon}=0.05$, among the $10$ runs from the MNIST dataset, the best testing error of \texttt{ObjPM} is $11.74\%$ while that of \texttt{OutP} is $21.79\%$, a $10.05\%$ improvement.

  \item When $\bar{\epsilon}=0.05$, among the $10$ runs from the FEMNIST dataset, the best testing error of \texttt{ObjPM} is $59.42\%$ while that of \texttt{OutP} is $91.05\%$, a $31.63\%$ improvement.
\end{itemize}

In Figure \ref{fig:Summary}, for every algorithm and $\bar{\epsilon}$, we report the best testing error among the $10$ instances, which showcases the outperformance of \texttt{ObjPM}.

\begin{figure}[!ht]
  \centering
  \begin{subfigure}[b]{0.23\textwidth}
    \centering
    \includegraphics[width=\textwidth]{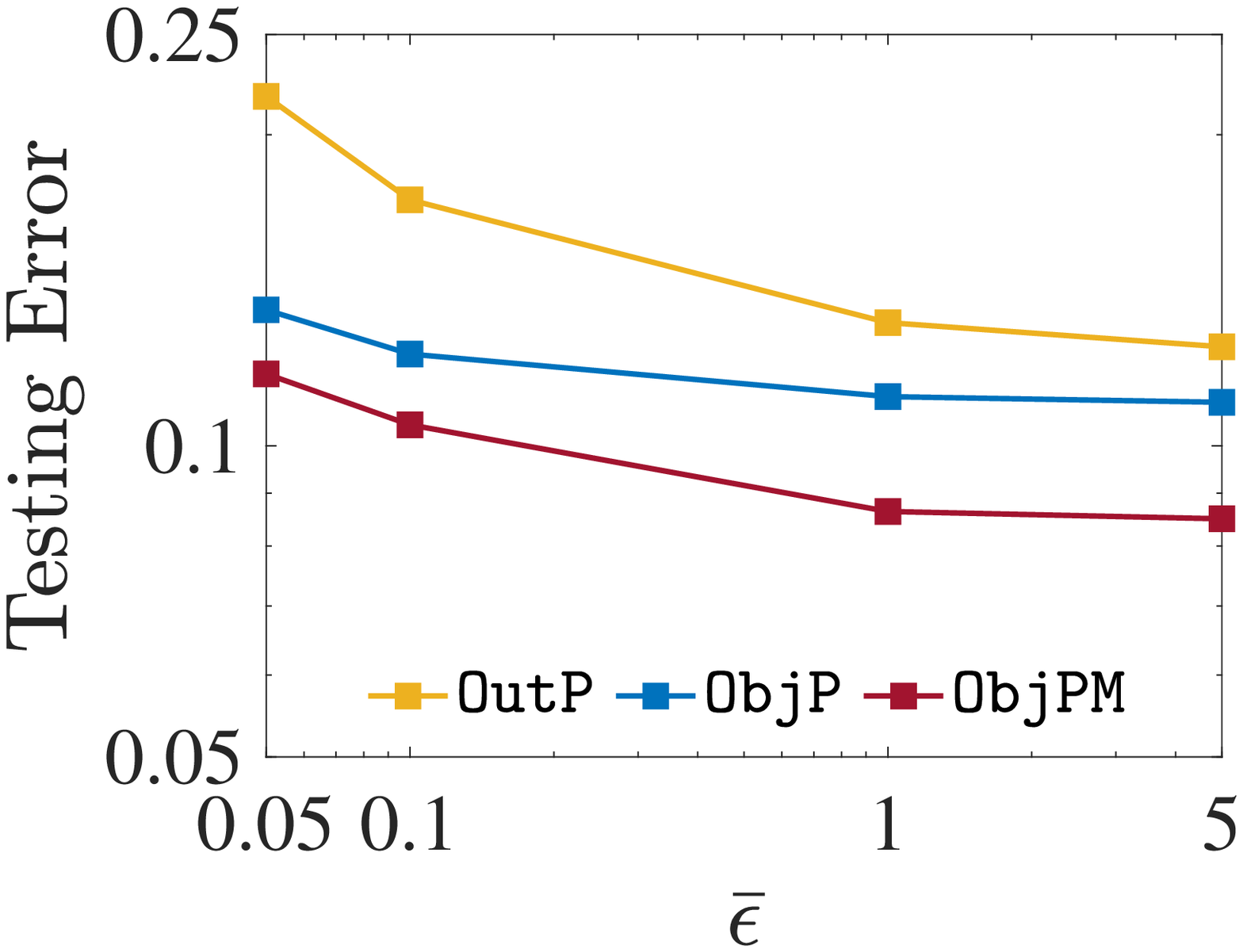}
    \caption{MNIST}
    \label{fig:MNIST_Summary}
  \end{subfigure}
  \begin{subfigure}[b]{0.23\textwidth}
    \centering
    \includegraphics[width=\textwidth]{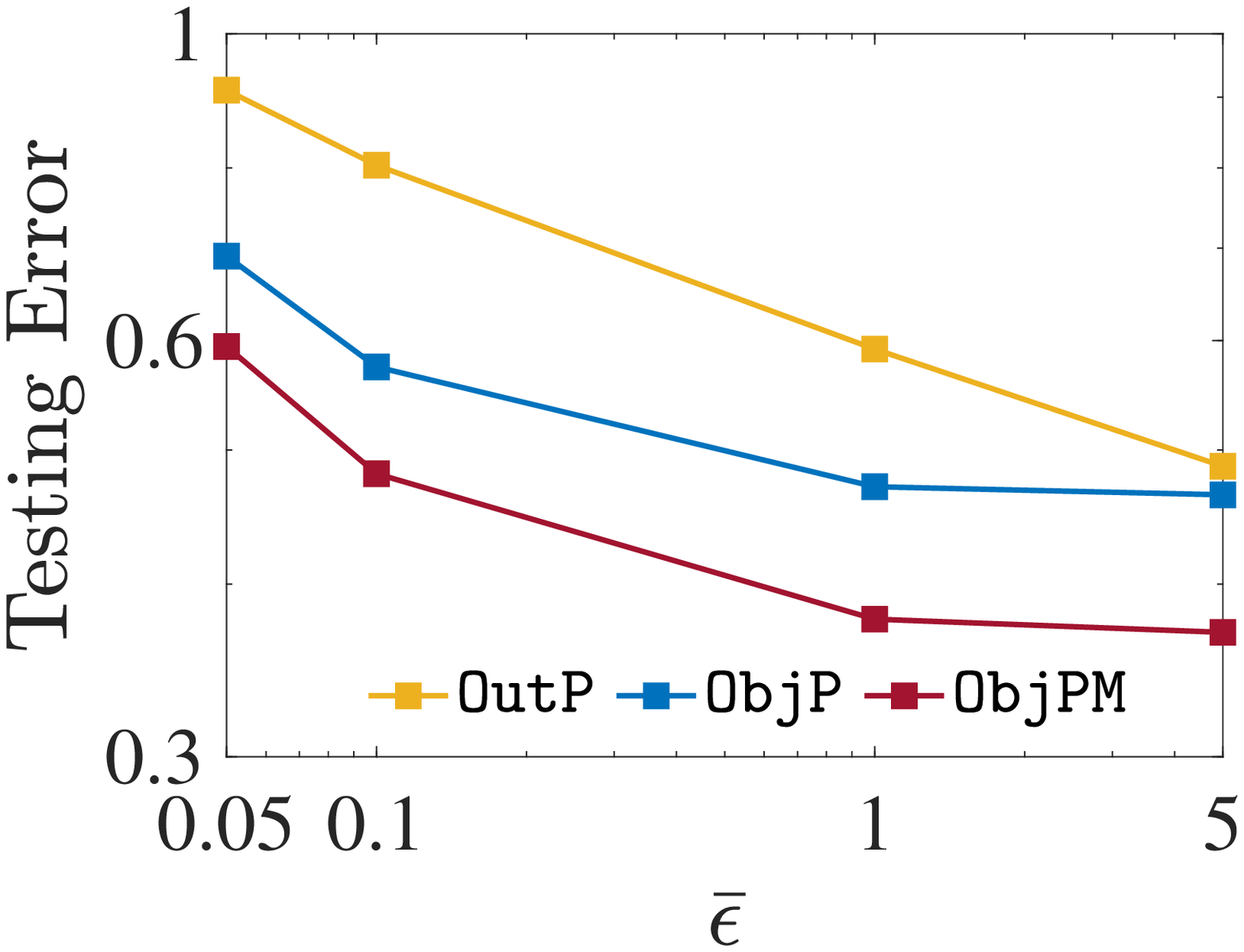}
    \caption{FEMNIST}
    \label{fig:FEMNIST_Summary}
  \end{subfigure}
  \caption{Best testing errors of the three algorithms under various $\bar{\epsilon}$.}
  \label{fig:Summary}
\end{figure} 

\subsection{Comparison of Random Noises}
The random noises to \texttt{OutP} are generated by the Gaussian mechanism with \textit{decreasing} variance as in~\cite{huang2019dp} and injected into the output of the subproblem, whereas the noises to our algorithms are generated by the Laplacian mechanism and injected into the objective function of the subproblem. 
To compare the two different mechanisms in terms of the magnitude of noises generated, we compute the following average noise magnitude:
\begin{align}
  \textstyle\frac{1}{PJK} \sum_{p=1}^P \sum_{j=1}^J \sum_{k=1}^K |\hat{\xi}^t_{pjk}|, \ \forall t \in [T],  \nonumber
\end{align}
where $\hat{\xi}^t_{pjk}$ is a realization of random noise $\tilde{\xi}^t_{pjk}$.

\begin{figure}[!ht]  
  \centering
  \begin{subfigure}[b]{0.23\textwidth}
      \centering      
      \includegraphics[width=\textwidth]{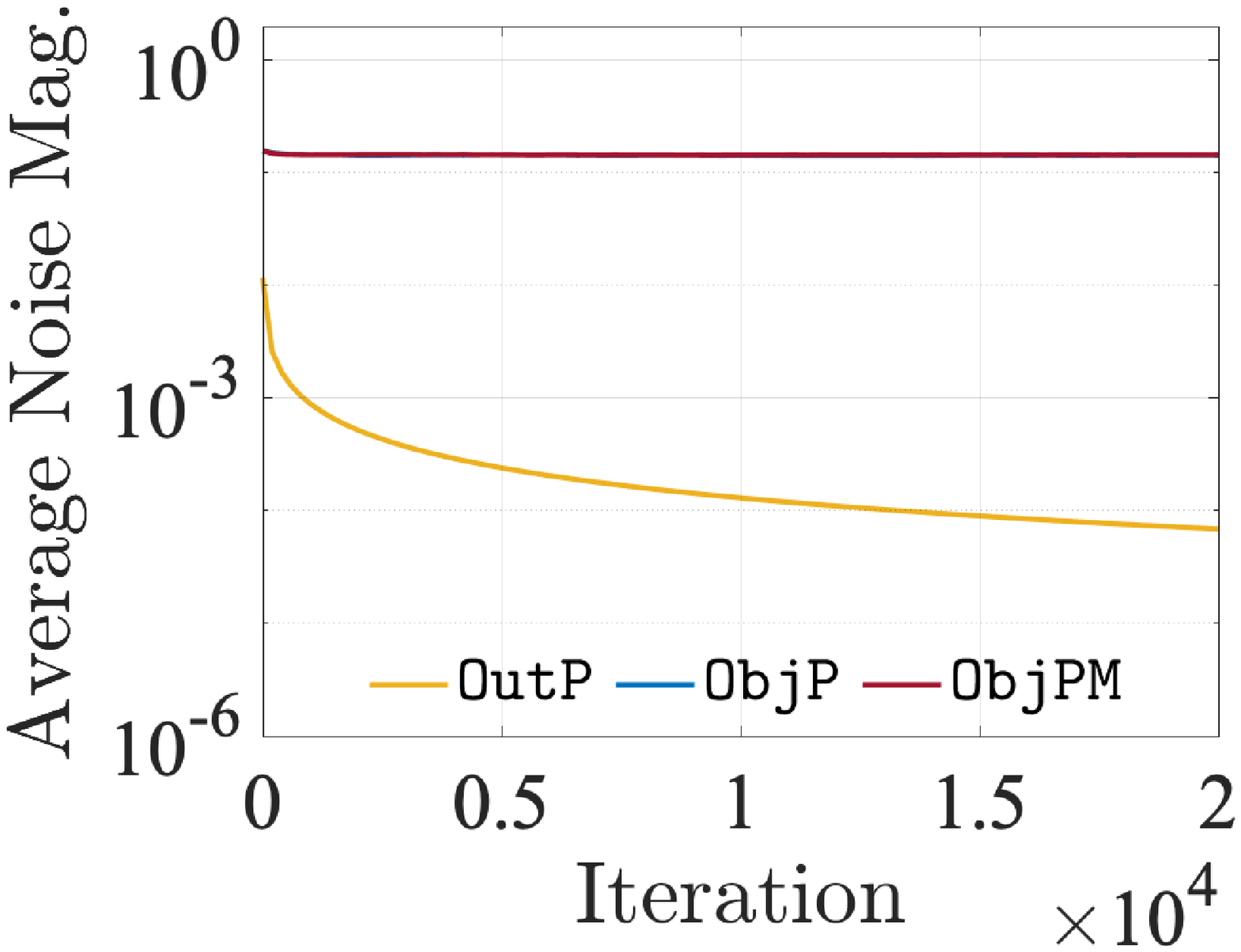}
      \includegraphics[width=\textwidth]{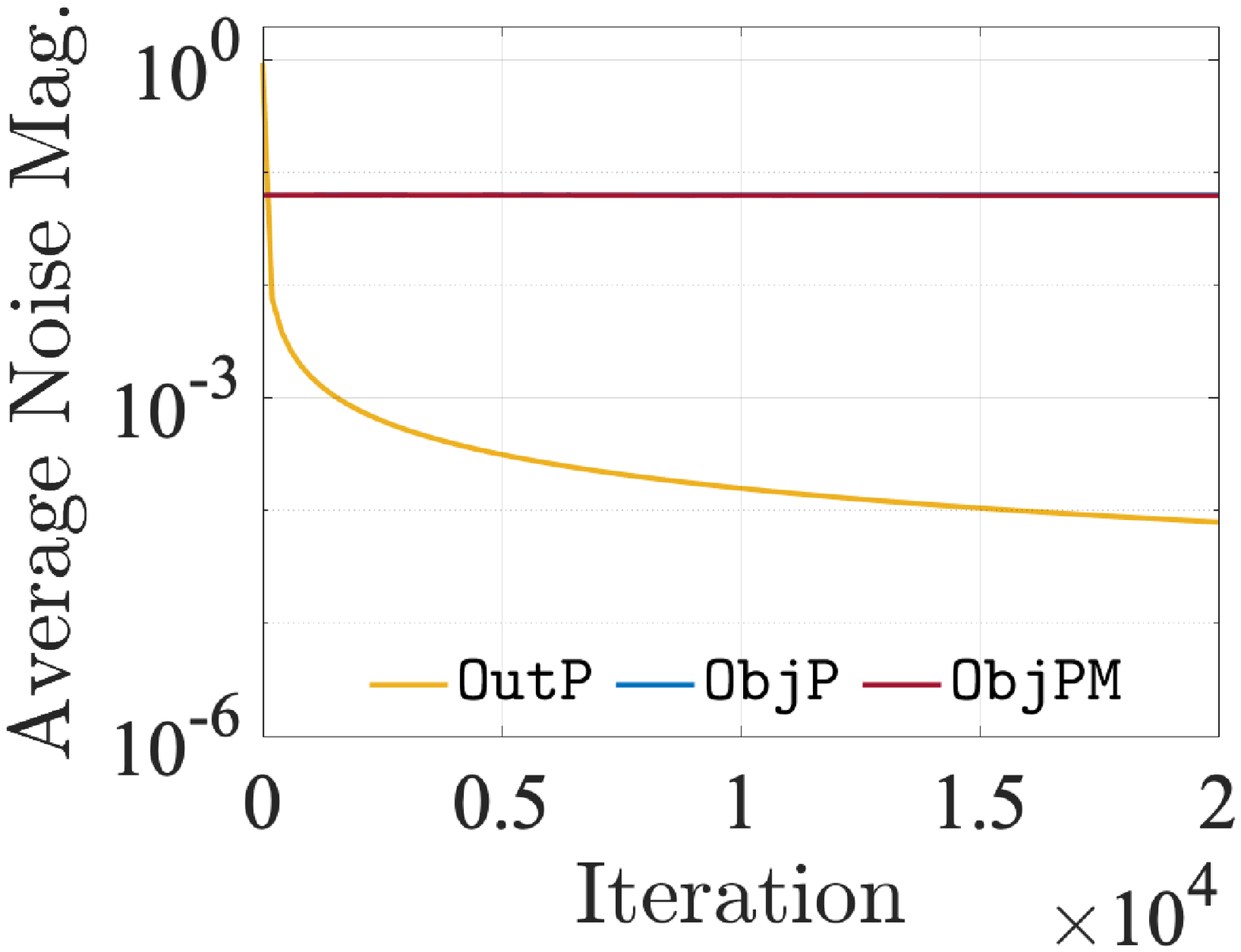}      
      \caption{$\bar{\epsilon}=0.05$}
  \end{subfigure}
  \begin{subfigure}[b]{0.23\textwidth}
      \centering      
      \includegraphics[width=\textwidth]{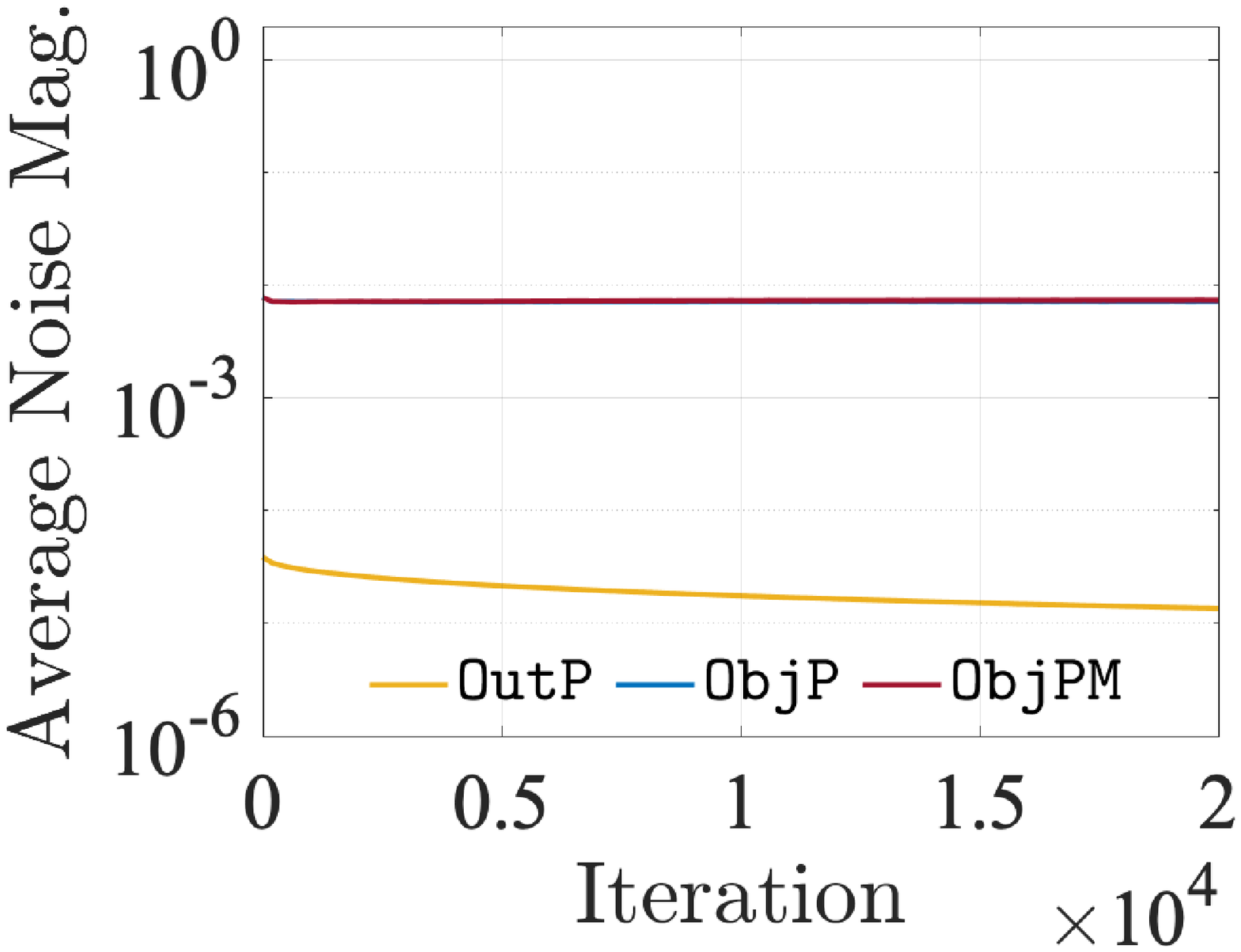}
      \includegraphics[width=\textwidth]{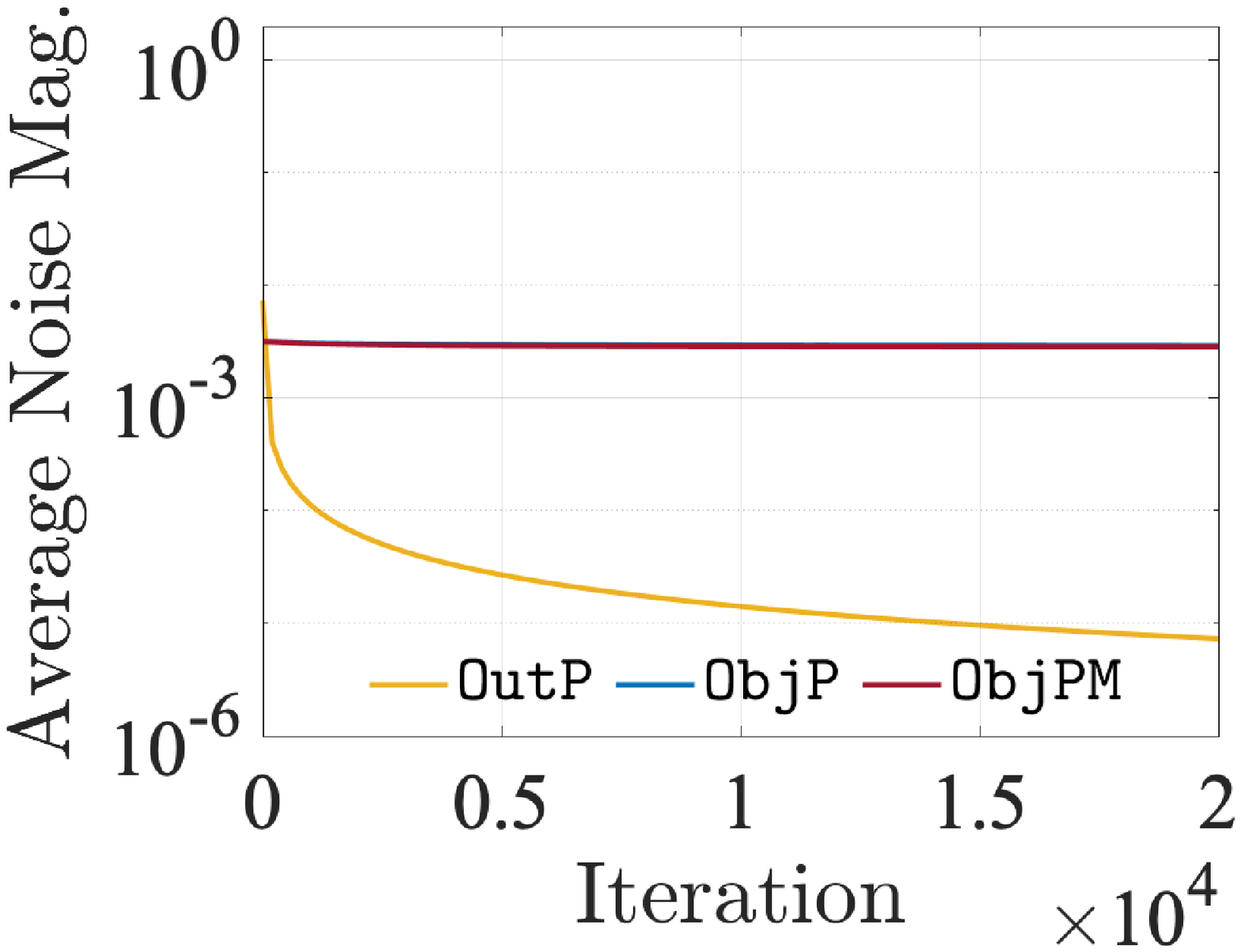}
      \caption{$\bar{\epsilon}=1.0$}
  \end{subfigure}  
  \caption{Average noise magnitudes for every iteration (top: MNIST; bottom: FEMNIST).}
  \label{fig:noise}
\end{figure}
In Figure \ref{fig:noise}, using the same instances as in Section \ref{sec:testing_errors}, we 
show that the average noise magnitudes of all the algorithms increase as $\bar{\epsilon}$ decreases, achieving stronger data privacy.
For fixed $\bar{\epsilon}$, the average noise magnitudes of our algorithms \texttt{ObjPM} and \texttt{ObjP} are greater than those of \texttt{OutP} while the testing errors of our algorithms are less than those of \texttt{OutP}.
These results imply that the performance of our algorithms is less sensitive to random perturbation than that of \texttt{OutP}, even with a larger magnitude of noises for stronger $\bar{\epsilon}$-DP.

\section{Conclusion} \label{sec:conclusion}
We incorporated the objective perturbation and multiple local updates into an IADMM algorithm for training the FL model while ensuring data privacy during the training process.
The proposed DP-IADMM algorithm iteratively solves a sequence of subproblems whose objective functions are randomly perturbed by noises sampled from a calibrated Laplace distribution to ensure $\bar{\epsilon}$-DP.
We showed that the rate of convergence in expectation for the proposed Algorithm \ref{algo:DP-IADMM-Prox} is $\mathcal{O}(1/\sqrt{T})$ for both a smooth and a nonsmooth convex setting and $\mathcal{O}(1/T)$ for a strongly convex setting.
The outperformance of the proposed algorithm was numerically demonstrated with the MNIST and FEMNIST datasets.

We note that the performance of the proposed DP algorithm can be further improved by lowering the magnitude of noises required for ensuring the same level of data privacy (see Figure \ref{fig:noise} showing that our algorithm requires larger noises). By improving the performance further, we expect that the proposed DP algorithm can be utilized for learning from larger decentralized datasets with more features and classes.


\bibliographystyle{IEEEtran}
\bibliography{References.bib}

\begin{thebibliography}{10}
\providecommand{\url}[1]{#1}
\csname url@samestyle\endcsname
\providecommand{\newblock}{\relax}
\providecommand{\bibinfo}[2]{#2}
\providecommand{\BIBentrySTDinterwordspacing}{\spaceskip=0pt\relax}
\providecommand{\BIBentryALTinterwordstretchfactor}{4}
\providecommand{\BIBentryALTinterwordspacing}{\spaceskip=\fontdimen2\font plus
\BIBentryALTinterwordstretchfactor\fontdimen3\font minus
  \fontdimen4\font\relax}
\providecommand{\BIBforeignlanguage}[2]{{%
\expandafter\ifx\csname l@#1\endcsname\relax
\typeout{** WARNING: IEEEtran.bst: No hyphenation pattern has been}%
\typeout{** loaded for the language `#1'. Using the pattern for}%
\typeout{** the default language instead.}%
\else
\language=\csname l@#1\endcsname
\fi
#2}}
\providecommand{\BIBdecl}{\relax}
\BIBdecl

\bibitem{konevcny2015federated}
J.~Kone{\v{c}}n{\`y}, B.~McMahan, and D.~Ramage, ``Federated optimization:
  Distributed optimization beyond the datacenter,'' \emph{arXiv preprint
  arXiv:1511.03575}, 2015.

\bibitem{huang2019dp}
Z.~Huang, R.~Hu, Y.~Guo, E.~Chan-Tin, and Y.~Gong, ``{DP-ADMM: ADMM}-based
  distributed learning with differential privacy,'' \emph{IEEE Transactions on
  Information Forensics and Security}, vol.~15, pp. 1002--1012, 2019.

\bibitem{dwork2014algorithmic}
C.~Dwork, A.~Roth \emph{et~al.}, ``The algorithmic foundations of differential
  privacy.'' \emph{Foundations and Trends in Theoretical Computer Science},
  vol.~9, no. 3-4, pp. 211--407, 2014.

\bibitem{shickel2017deep}
B.~Shickel, P.~J. Tighe, A.~Bihorac, and P.~Rashidi, ``{Deep EHR}: a survey of
  recent advances in deep learning techniques for electronic health record
  {(EHR)} analysis,'' \emph{IEEE Journal of Biomedical and Health Informatics},
  vol.~22, no.~5, pp. 1589--1604, 2017.

\bibitem{mcmahan2017communication}
B.~McMahan, E.~Moore, D.~Ramage, S.~Hampson, and B.~A. y~Arcas,
  ``Communication-efficient learning of deep networks from decentralized
  data,'' in \emph{Artificial Intelligence and Statistics}.\hskip 1em plus
  0.5em minus 0.4em\relax PMLR, 2017, pp. 1273--1282.

\bibitem{zhang2016dynamic}
T.~Zhang and Q.~Zhu, ``Dynamic differential privacy for {ADMM}-based
  distributed classification learning,'' \emph{IEEE Transactions on Information
  Forensics and Security}, vol.~12, no.~1, pp. 172--187, 2016.

\bibitem{wei2020federated}
K.~Wei, J.~Li, M.~Ding, C.~Ma, H.~H. Yang, F.~Farokhi, S.~Jin, T.~Q. Quek, and
  H.~V. Poor, ``Federated learning with differential privacy: Algorithms and
  performance analysis,'' \emph{IEEE Transactions on Information Forensics and
  Security}, vol.~15, pp. 3454--3469, 2020.

\bibitem{naseri2020toward}
M.~Naseri, J.~Hayes, and E.~De~Cristofaro, ``Toward robustness and privacy in
  federated learning: Experimenting with local and central differential
  privacy,'' \emph{arXiv preprint arXiv:2009.03561}, 2020.

\bibitem{madry2017towards}
A.~Madry, A.~Makelov, L.~Schmidt, D.~Tsipras, and A.~Vladu, ``Towards deep
  learning models resistant to adversarial attacks,'' \emph{arXiv preprint
  arXiv:1706.06083}, 2017.

\bibitem{dwork2006calibrating}
C.~Dwork, F.~McSherry, K.~Nissim, and A.~Smith, ``Calibrating noise to
  sensitivity in private data analysis,'' in \emph{Theory of cryptography
  conference}.\hskip 1em plus 0.5em minus 0.4em\relax Springer, 2006, pp.
  265--284.

\bibitem{kaissis2020secure}
G.~A. Kaissis, M.~R. Makowski, D.~R{\"u}ckert, and R.~F. Braren, ``Secure,
  privacy-preserving and federated machine learning in medical imaging,''
  \emph{Nature Machine Intelligence}, vol.~2, no.~6, pp. 305--311, 2020.

\bibitem{shokri2017membership}
R.~Shokri, M.~Stronati, C.~Song, and V.~Shmatikov, ``Membership inference
  attacks against machine learning models,'' in \emph{2017 IEEE Symposium on
  Security and Privacy (SP)}.\hskip 1em plus 0.5em minus 0.4em\relax IEEE,
  2017, pp. 3--18.

\bibitem{fukuchi2017differentially}
K.~Fukuchi, Q.~K. Tran, and J.~Sakuma, ``Differentially private empirical risk
  minimization with input perturbation,'' in \emph{International Conference on
  Discovery Science}.\hskip 1em plus 0.5em minus 0.4em\relax Springer, 2017,
  pp. 82--90.

\bibitem{kang2020input}
Y.~Kang, Y.~Liu, B.~Niu, X.~Tong, L.~Zhang, and W.~Wang, ``Input perturbation:
  A new paradigm between central and local differential privacy,'' \emph{arXiv
  preprint arXiv:2002.08570}, 2020.

\bibitem{chaudhuri2011differentially}
K.~Chaudhuri, C.~Monteleoni, and A.~D. Sarwate, ``Differentially private
  empirical risk minimization.'' \emph{Journal of Machine Learning Research},
  vol.~12, no.~3, 2011.

\bibitem{kifer2012private}
D.~Kifer, A.~Smith, and A.~Thakurta, ``Private convex empirical risk
  minimization and high-dimensional regression,'' in \emph{Conference on
  Learning Theory}.\hskip 1em plus 0.5em minus 0.4em\relax JMLR Workshop and
  Conference Proceedings, 2012, pp. 25--1.

\bibitem{abadi2016deep}
M.~Abadi, A.~Chu, I.~Goodfellow, H.~B. McMahan, I.~Mironov, K.~Talwar, and
  L.~Zhang, ``Deep learning with differential privacy,'' in \emph{Proceedings
  of the 2016 ACM SIGSAC conference on computer and communications security},
  2016, pp. 308--318.

\bibitem{sarwate2013signal}
A.~D. Sarwate and K.~Chaudhuri, ``Signal processing and machine learning with
  differential privacy: Algorithms and challenges for continuous data,''
  \emph{IEEE Signal Processing Magazine}, vol.~30, no.~5, pp. 86--94, 2013.

\bibitem{iyengar2019towards}
R.~Iyengar, J.~P. Near, D.~Song, O.~Thakkar, A.~Thakurta, and L.~Wang,
  ``Towards practical differentially private convex optimization,'' in
  \emph{2019 IEEE Symposium on Security and Privacy (SP)}.\hskip 1em plus 0.5em
  minus 0.4em\relax IEEE, 2019, pp. 299--316.

\bibitem{li2018federated}
T.~Li, A.~K. Sahu, M.~Zaheer, M.~Sanjabi, A.~Talwalkar, and V.~Smith,
  ``Federated optimization in heterogeneous networks,'' \emph{arXiv preprint
  arXiv:1812.06127}, 2018.

\bibitem{zhou2021communication}
S.~Zhou and G.~Y. Li, ``Communication-efficient {ADMM}-based federated
  learning,'' \emph{arXiv preprint arXiv:2110.15318}, 2021.

\bibitem{kairouz2019advances}
P.~Kairouz, H.~B. McMahan, B.~Avent, A.~Bellet, M.~Bennis, A.~N. Bhagoji,
  K.~Bonawitz, Z.~Charles, G.~Cormode, R.~Cummings \emph{et~al.}, ``Advances
  and open problems in federated learning,'' \emph{arXiv preprint
  arXiv:1912.04977}, 2019.

\bibitem{li2019survey}
Q.~Li, Z.~Wen, Z.~Wu, S.~Hu, N.~Wang, Y.~Li, X.~Liu, and B.~He, ``A survey on
  federated learning systems: vision, hype and reality for data privacy and
  protection,'' \emph{arXiv preprint arXiv:1907.09693}, 2019.

\bibitem{li2020federated}
T.~Li, A.~K. Sahu, A.~Talwalkar, and V.~Smith, ``Federated learning:
  Challenges, methods, and future directions,'' \emph{IEEE Signal Processing
  Magazine}, vol.~37, no.~3, pp. 50--60, 2020.

\bibitem{agarwal2018cpsgd}
N.~Agarwal, A.~T. Suresh, F.~Yu, S.~Kumar, and H.~B. Mcmahan, ``{cpSGD:
  Communication-efficient and differentially-private distributed SGD},''
  \emph{arXiv preprint arXiv:1805.10559}, 2018.

\bibitem{teo2010bundle}
C.~H. Teo, S.~Vishwanathan, A.~Smola, and Q.~V. Le, ``Bundle methods for
  regularized risk minimization,'' \emph{Journal of Machine Learning Research},
  vol.~11, no.~1, 2010.

\bibitem{beck2017first}
A.~Beck, \emph{First-order methods in optimization}.\hskip 1em plus 0.5em minus
  0.4em\relax SIAM, 2017.

\bibitem{azadi2014towards}
S.~Azadi and S.~Sra, ``Towards an optimal stochastic alternating direction
  method of multipliers,'' in \emph{International Conference on Machine
  Learning}.\hskip 1em plus 0.5em minus 0.4em\relax PMLR, 2014, pp. 620--628.

\bibitem{lecun1998mnist}
Y.~LeCun, ``The {MNIST} database of handwritten digits,'' \emph{http://yann.
  lecun. com/exdb/mnist/}, 1998.

\bibitem{caldas2018leaf}
S.~Caldas, S.~M.~K. Duddu, P.~Wu, T.~Li, J.~Kone{\v{c}}n{\`y}, H.~B. McMahan,
  V.~Smith, and A.~Talwalkar, ``Leaf: A benchmark for federated settings,''
  \emph{arXiv preprint arXiv:1812.01097}, 2018.

\bibitem{billingsley}
P.~Billingsley, \emph{Probability and measure}.\hskip 1em plus 0.5em minus
  0.4em\relax John Wiley \& Sons, 1995.

\end{thebibliography}

\vspace{0.5in}
\noindent\fbox{\parbox{0.47\textwidth}{
The submitted manuscript has been created by UChicago Argonne, LLC, Operator of Argonne National Laboratory (``Argonne''). Argonne, a U.S. Department of Energy Office of Science laboratory, is operated under Contract No. DE-AC02-06CH11357. The U.S. Government retains for itself, and others acting on its behalf, a paid-up nonexclusive, irrevocable worldwide license in said article to reproduce, prepare derivative works, distribute copies to the public, and perform publicly and display publicly, by or on behalf of the Government. The Department of Energy will provide public access to these results of federally sponsored research in accordance with the DOE Public Access Plan (http://energy.gov/downloads/doe-public-access-plan).}
}

\newpage
\appendices
\onecolumn

\section{Proof of Proposition \ref{prop:pointwise_convergence} } \label{apx-prop:pointwise_convergence}
\noindent
We aim to show that, as $\ell$ increases, $z^{t,e+1}_{p\ell}$ converges to $z^{t,e+1}_p$, where $z^{t,e+1}_p$ (resp., $z^{t,e+1}_{p \ell}$) is the optimal solution of an optimization problem in \eqref{DPADMM-2-Prox} (resp., \eqref{ADMM-2-Prox-log}).

Suppose that $z^{t,e+1}_{p\ell}$ converges to $\widehat{z} \neq z^{t,e+1}_p$ as $\ell$ increases.
Consider $\zeta := \| \widehat{z} - z^{t,e+1}_p \| / 2$.
Since $z^{t,e+1}_{p\ell}$  converges to $\widehat{z}$, there exists $\ell' > 0$ such that $\| \widehat{z} - z^{t,e+1}_{p\ell} \| < \zeta$ for all $\ell \geq \ell'$.
By the triangle inequality, we have
\begin{subequations}
\begin{align}
\| z^{t,e+1}_{p\ell} - z^{t,e+1}_p \| \geq \| \widehat{z} - z^{t,e+1}_p \| - \| \widehat{z} - z^{t,e+1}_{p\ell} \| > 2\zeta - \zeta = \zeta, \ \forall \ell \geq \ell'. \label{triangle}
\end{align}
Since $G^{t,e}_p$ is strongly convex with a constant $\mu > 0$, we have
\begin{align}
G^{t,e}_p(z^{t,e+1}_{p\ell}) - G^{t,e}_p(z^{t,e+1}_p) \geq \textstyle \frac{\mu}{2} \| z^{t,e+1}_{p\ell} - z^{t,e+1}_p \|^2 >  \frac{\mu \zeta^2}{2}, \ \forall \ell \geq \ell', \label{strong_triangle}
\end{align}
where the last inequality holds by \eqref{triangle}.
By adding $g_{\ell} (z^{t,e+1}_{p\ell}) \geq 0$ to the left-hand side of \eqref{strong_triangle}, we derive the following inequality:
\begin{align}
\big\{ G^{t,e}_p(z^{t,e+1}_{p\ell}) + g_{\ell} (z^{t,e+1}_{p\ell})  \big\} - G^{t,e}_p(z^{t,e+1}_p)  > \textstyle \frac{\mu \zeta^2}{2}, \ \forall \ell \geq \ell'. \label{diff_0}
\end{align}
To see the contradiction, consider $\epsilon \in (0, \frac{\mu \zeta^2}{4})$.
The continuity of $G^{t,e}_p: \mathcal{W} \mapsto \mathbb{R}$ at $z^{t,e+1}_p$ implies that for every $\epsilon > 0$, there exists a $\delta > 0$ such that for all $z \in \mathcal{W}$:
\begin{align}
z \in \mathcal{B}_{\delta}(z^{t,e+1}_p) := \{ z \in \mathbb{R}^{J \times K} : \| z - z^{t,e+1}_p \| < \delta \} \ \Rightarrow \  G^{t,e}_p(z) - G^{t,e}_p(z^{t,e+1}_p)  < \epsilon. \label{continuity}
\end{align}
Consider $\tilde{z} \in \mathcal{B}_{\delta}(z^{t,e+1}_p) \cap \textbf{relint}(\mathcal{W})$, where \textbf{relint} indicates the relative interior.
Since $h_m(\tilde{z}) < 0$ for all $m \in [M]$, $g_{\ell}(\tilde{z})$ goes to zero as $\ell$ increases.
Hence, there exists $\ell'' > 0$ such that
\begin{align}
  g_{\ell}(\tilde{z}) = \textstyle \sum_{m=1}^M \ln (1 + e^{\ell h_m(\tilde{z})}) < \epsilon, \ \forall \ell \geq \ell''. \label{ineq_z_tilde}
\end{align}
For all $\ell \geq \ell''$, we derive the following inequalities:
\begin{align}
  G^{t,e}_p(z^{t,e+1}_{p\ell}) + g_{\ell}(z^{t,e+1}_{p\ell}) \leq G^{t,e}_p(\tilde{z}) + g_{\ell}(\tilde{z}) < G^{t,e}_p(\tilde{z}) + \epsilon < G^{t,e}_p(z^{t,e+1}_p) + 2\epsilon, \label{sandwich_ineq}
\end{align}
where
the first inequality holds because $z^{t,e+1}_{p\ell}$ is the optimal solution of \eqref{ADMM-2-Prox-log},
the second inequality holds by \eqref{ineq_z_tilde}, and
the last inequality holds by \eqref{continuity}.
Therefore, we have
\begin{align}
\big\{ G^{t,e}_p(z^{t,e+1}_{p\ell}) + g_{\ell} (z^{t,e+1}_{p\ell})  \big\} - G^{t,e}_p(z^{t,e+1}_p) \underbrace{<}_{\text{from } \eqref{sandwich_ineq}} 2 \epsilon < \frac{\mu \zeta^2}{2}, \ \forall \ell \geq \ell''. \label{diff_1}
\end{align}
Therefore, for all $\ell \geq \max \{ \ell', \ell'' \}$, \eqref{diff_0} and \eqref{diff_1} contradict.
This completes the proof.
\end{subequations}

\section{Proof of Proposition \ref{prop:dpinapproximation}} \label{apx-prop:dpinapproximation}
\noindent
It suffices to show that the following is true:
\begin{subequations}
\begin{align}
  e^{-\bar{\epsilon}} \ \textbf{pdf} \big( z^{t,e+1}_{p\ell} (\mathcal{D}'_p) = \psi \big)
  \leq \textbf{pdf} \big( z^{t,e+1}_{p\ell} (\mathcal{D}_p) = \psi \big)
  \leq e^{\bar{\epsilon}} \ \textbf{pdf} \big( z^{t,e+1}_{p\ell} (\mathcal{D}'_p) = \psi \big), \ \forall \psi \in \mathbb{R}^{J \times K}, \label{DP_PDF}
\end{align}
where $\textbf{pdf}$ represents a probability density function.

Consider $\psi \in \mathbb{R}^{J \times K}$.
If we have $z^{t,e+1}_{p \ell} (\mathcal{D}_p)=\psi$, then $\psi$ is the unique minimizer of \eqref{ADMM-2-Prox-log} because the objective function in \eqref{ADMM-2-Prox-log} is strongly convex.
From the optimality condition of \eqref{ADMM-2-Prox-log}, we derive
\begin{align}
  \tilde{\xi}^{t,e}_p = & - f'_p(z^{t,e}_p;\mathcal{D}_p) + \rho^t(w^{t+1}-\psi) + \lambda^t_p - \nabla g_{\ell}(\psi)  - \textstyle \frac{1}{\eta^{t}} \big( \psi - z_p^{t,e} \big), \label{correspondence}
\end{align}
where $\nabla g_{\ell}(\psi) = \sum_{m=1}^M \frac{\ell e ^{\ell h_m(\psi)}}{1+e^{\ell h_m(\psi)}} \nabla h_m(\psi)$.
Note that the mapping from $\psi$ to $\tilde{\xi}^{t,e}_p$ via \eqref{correspondence} is injective. 
Also the mapping is surjective because for all $\tilde{\xi}^{t,e}_p$, there exists $\psi$ (i.e., the unique minimizer of \eqref{ADMM-2-Prox-log}) such that \eqref{correspondence} holds.
Therefore, the relation between $\psi$ and $\tilde{\xi}^{t,e}_p$ is bijective, which enables us to utilize the inverse function theorem (Theorem 17.2 in \cite{billingsley}), namely,
  \begin{align}
  \textbf{pdf} \big( z^{t,e+1}_{p\ell} (\mathcal{D}_p) = \psi \big) \cdot \big|\textbf{det}[\nabla \tilde{\xi}^{t,e}_p (\psi;\mathcal{D}_p) ] \big| = \text{Lap} \big( \tilde{\xi}^{t,e}_p (\psi;\mathcal{D}_p); 0, \bar{\Delta}_p^{t,e} /\bar{\epsilon} \big), \label{inverse}
  \end{align}
  where \textbf{det} represents a determinant of a matrix, $\text{Lap} $ is from \eqref{Laplace-pdf}, and
  $\nabla \tilde{\xi}^{t,e}_p (\psi;\mathcal{D}_p)$ represents a Jacobian matrix of the mapping from $\psi$ to $\tilde{\xi}^{t,e}_p$ in \eqref{correspondence}, namely,
  \begin{align}
  \nabla \tilde{\xi}^{t,e}_p (\psi;\mathcal{D}_p) =  (-\rho^t - 1/\eta^{t}) \mathbb{I}_{JK} - \nabla \Big( \sum_{m=1}^M \frac{\ell e ^{\ell h_m(\psi)}}{1+e^{\ell h_m(\psi)}} \nabla h_m(\psi) \Big) , \label{Jacobian}
  \end{align}
  where $\mathbb{I}_{JK}$ is an identity matrix of $JK \times JK$ dimensions.
  Since the Jacobian matrix is not affected by the dataset, we have
  \begin{align}
  \nabla \tilde{\xi}^{t,e}_p (\psi;\mathcal{D}_p) = \nabla \tilde{\xi}^{t,e}_p (\psi;\mathcal{D}'_p).  \label{jacobian}
  \end{align}
  Based on \eqref{inverse} and \eqref{jacobian}, we derive the following inequalities:
  \begin{align}
    & \frac{\textbf{pdf} \big( z^{t,e+1}_{p \ell} (\mathcal{D}_p) = \psi \big)}{\textbf{pdf} \big( z^{t,e+1}_{p \ell} (\mathcal{D}'_p ) = \psi \big)}  
    \underbrace{=}_{\text{by \eqref{inverse}}} \frac{\text{Lap} \big( \tilde{\xi}^{t,e}_p (\psi; \mathcal{D}_p); 0, \bar{\Delta}_p^{t,e}/\bar{\epsilon} \big) }{\text{Lap} \big( \tilde{\xi}^{t,e}_p (\psi; \mathcal{D}'_p); 0, \bar{\Delta}_p^{t,e} / \bar{\epsilon} \big) } \cdot \frac{\big|\textbf{det}[\nabla \tilde{\xi}^{t,e}_p (\psi;\mathcal{D}'_p) ] \big|}{\big|\textbf{det}[\nabla \tilde{\xi}^{t,e}_p (\psi;\mathcal{D}_p) ] \big|} 
    \underbrace{=}_{\text{by \eqref{jacobian}}} \frac{\text{Lap} \big( \tilde{\xi}^{t,e}_p (\psi; \mathcal{D}_p); 0, \bar{\Delta}_p^{t,e}/\bar{\epsilon} \big) }{\text{Lap} \big( \tilde{\xi}^{t,e}_p (\psi; \mathcal{D}'_p); 0, \bar{\Delta}_p^{t,e} /\bar{\epsilon} \big) } \nonumber \\
    & = \textbf{exp} \Big( (\bar{\epsilon}/\bar{\Delta}_p^{t,e})(\| \tilde{\xi}^{t,e}_p (\psi;\mathcal{D}'_p) \|_1 - \| \tilde{\xi}^{t,e}_p (\psi;\mathcal{D}_p) \|_1) \Big) 
    \underbrace{\leq}_{\text{triangle inequality}} \textbf{exp} \Big( (\bar{\epsilon}/\bar{\Delta}_p^{t,e})(\| \tilde{\xi}^{t,e}_p (\psi;\mathcal{D}'_p) - \tilde{\xi}^{t,e}_p (\psi;\mathcal{D}_p) \|_1) \Big) \nonumber \\
    & \underbrace{=}_{\text{by \eqref{correspondence}}} \textbf{exp} \Big( (\bar{\epsilon}/\bar{\Delta}_p^{t,e})(\| f'_p(z^{t,e}_p;\mathcal{D}_p) - f'_p(z^{t,e}_p;\mathcal{D}'_p)\|_1) \Big) \underbrace{\leq}_{\text{by \eqref{Delta}}} \textbf{exp} (\bar{\epsilon}), \nonumber
  \end{align}
  \end{subequations}
  where \textbf{exp} represents the exponential function.
  Similarly, one can derive a lower bound in \eqref{DP_PDF}.
  Integrating $\psi$ in \eqref{DP_PDF} over $\mathcal{S}$ yields \eqref{DP_1}.
  This completes the proof.

\section{Existence of $U_1$, $U_2$, and $U_3$ in \eqref{def_upper}} \label{apx:existence_of_UBs}
\noindent
(Existence of $U_2$) 
$U_2$ is well defined because the objective function $\|u-v\|$ is continuous and the feasible region $\mathcal{W}$ is compact. \\
(Existence of $U_1$)
The necessary and sufficient condition of Assumption \ref{assump:convergence} (iii) is that, for all $u \in \mathcal{W}$  and $v \in \partial f_p(u)$, $\| v \|_{\star} \leq H$, where $\|\cdot\|_{\star}$ is the dual norm.
As the dual norm of the Euclidean norm is the Euclidean norm, we have $\| f'_p(u) \| \leq H$.
Since the objective function, which is a maximum of finite continuous functions, is continuous and $\mathcal{W}$ is compact, $U_1$ is well defined. \\
(Existence of $U_3$)
From the norm inequality, we have
\begin{align*}
& \| f'_p(u;\mathcal{D}_p) - f'_p(u;\mathcal{D}'_p)\|_1 \leq \sqrt{JK} \| f'_p(u;\mathcal{D}_p) - f'_p(u;\mathcal{D}'_p)\|_2 \\
& \leq \sqrt{JK} \{ \| f'_p(u;\mathcal{D}_p) \|_2  + \| f'_p(u;\mathcal{D}'_p)\|_2 \} \leq 2 H \sqrt{JK}, \ \forall u \in \mathcal{W},
\end{align*}
where the last inequality holds by Assumption \ref{assump:convergence} (iii). Therefore, $U_3$ is well defined.

\section{Proof of Theorem \ref{thm:smooth} } \label{apx-thm:smooth} 

\subsection{Preliminaries}
\noindent
First, we note that for any symmetric matrix $A$, 
\begin{align}
(a-b)^{\top} A (c-d) = \frac{1}{2} \{ \| a-d \|^2_A  - \| a-c \|^2_A + \| c-b \|^2_A - \| d - b \|^2_A \},  \label{rule}
\end{align}
where $a$, $b$, $c$, and $d$ are vectors of the same size.

Second, we define $\tilde{\lambda}^t_p := \lambda^t_p + \rho^t (w^{t+1} - z^t_p)$ for fixed $t \in [T]$ and $p \in [P]$.
From the optimality condition of \eqref{ADMM-1}, namely, 
$\sum_{p=1}^P \lambda^t_p + \rho^t(w^{t+1}-z^t_p)= \sum_{p=1}^P \tilde{\lambda}^t_p =0$, we have
\begin{align}
& \textstyle \sum_{p=1}^P \langle  \tilde{\lambda}^t_p, w^{t+1}-w \rangle = 0, \ \forall w. \label{Block1:optimality_condition}
\end{align}

\subsection{Inequality derivation for a fixed iteration $t$ and $e$.}
\noindent
First, for a given $p \in [P]$, the optimality condition of \eqref{DPADMM-2-Prox} is given by 
\begin{align*}
\langle \nabla f_p(z^{t,e}_p) - \{ \underbrace{ \lambda^t_p  + \rho^t(w^{t+1}-z^{t,e+1}_p) }_{ \text{``\texttt{A}''} } \} + \tilde{\xi}^{t,e}_p, z^{t,e+1}_p - z_p \rangle \leq \frac{1}{\eta^{t}} \underbrace{\langle z_p - z^{t,e+1}_p, z^{t,e+1}_p - z_p^{t,e} \rangle}_{\text{``\texttt{B}''}}, \ \forall z_p \in \mathcal{W}.
\end{align*}
By defining $\lambda^{t,e+1}_p  := \lambda^t_p + \rho^t (w^{t+1} - z^{t,e+1}_p)$ for the ``\texttt{A}'' term and applying \eqref{rule} on the ``\texttt{B}'' term from the above inequalities, we have
\begin{align}
\langle \nabla f_p(z^{t,e}_p) - \lambda^{t,e+1}_p + \tilde{\xi}^{t,e}_p, z^{t,e+1}_p - z_p \rangle \leq \frac{1}{2\eta^{t}} \Big( \| z_p - z_p^{t,e}\|^2 - \| z_p - z_p^{t,e+1}\|^2 - \| z_p^{t,e} - z_p^{t,e+1} \|^2 \Big). \label{Block2:smooth_inequality_0}
\end{align}
Second, by adding a term $\langle \lambda^{t,e+1}_p, z^{t,e+1}_p-z_p \rangle$ to the subgradient inequality $f_p(z^{t,e}_p) - f_p(z_p)  \leq \langle \nabla f_p(z^{t,e}_p), z^{t,e}_p - z_p \rangle$ for all $z_p$, we derive
\begin{align}
& f_p(z^{t,e}_p) - f_p(z_p) - \underbrace{\langle \lambda^{t,e+1}_p, z^{t,e+1}_p-z_p \rangle}_{\text{``\texttt{C}''}}  \leq \langle \nabla f_p(z^{t,e}_p), z^{t,e}_p-z^{t,e+1}_p \rangle + \langle \nabla f_p(z^{t,e}_p) - \lambda^{t,e+1}_p, z^{t,e+1}_p - z_p \rangle \nonumber \\
& = \langle  \nabla f_p(z^{t,e}_p) + \tilde{\xi}^{t,e}_p, z^{t,e}_p-z^{t,e+1}_p \rangle + \underbrace{\langle \nabla f_p(z^{t,e}_p) - \lambda^{t,e+1}_p + \tilde{\xi}^{t,e}_p, z^{t,e+1}_p-z_p \rangle}_{\text{applying } \eqref{Block2:smooth_inequality_0}} + \langle  \tilde{\xi}^{t,e}_p, z_p - z^{t,e}_p \rangle \nonumber \\
& \leq \langle  \nabla f_p(z^{t,e}_p) + \tilde{\xi}^{t,e}_p, z^{t,e}_p-z^{t,e+1}_p \rangle + \frac{1}{2\eta^{t}} \Big( \| z_p - z_p^{t,e}\|^2 - \| z_p - z_p^{t,e+1}\|^2 - \| z_p^{t,e} - z_p^{t,e+1} \|^2 \Big) + \langle  \tilde{\xi}^{t,e}_p, z_p - z^{t,e}_p \rangle .\nonumber 
\end{align}
Since the ``\texttt{C}'' term from the above inequalities can be written as
\begin{align}
& \langle \lambda^{t,e+1}_p, z^{t,e+1}_p-z_p \rangle = \langle \lambda^{t+1}_p, z^{t,e+1}_p-z_p \rangle + \langle \underbrace{\lambda^{t,e+1}_p - \lambda^{t+1}_p}_{=\rho^t(z^{t+1}_p - z^{t,e+1}_p ) }, z^{t,e+1}_p-z_p \rangle, \nonumber
\end{align}
we obtain
\begin{align}
& f_p(z^{t,e}_p) - f_p(z_p) - \langle \lambda^{t+1}_p, z^{t,e+1}_p-z_p \rangle 
\leq 
\rho^t \langle z^{t+1}_p - z^{t,e+1}_p, z^{t,e+1}_p-z_p \rangle + \underbrace{\langle \nabla f_p(z^{t,e}_p) + \tilde{\xi}^{t,e}_p, z^{t,e}_p-z^{t,e+1}_p \rangle}_{\text{``\texttt{D}''}} + \nonumber \\
& \frac{1}{2\eta^{t}} \Big( \| z_p - z_p^{t,e}\|^2 - \| z_p - z_p^{t,e+1}\|^2 - \| z_p^{t,e} - z_p^{t,e+1} \|^2 \Big) + \langle  \tilde{\xi}^{t,e}_p, z_p - z^{t,e}_p \rangle, \ \forall z_p \in \mathcal{W}.   \label{Block2:smooth_inequality_1}
\end{align}
Third, we derive from the ``\texttt{D}'' term in \eqref{Block2:smooth_inequality_1} that
\begin{align}
& \langle \nabla f_p(z^{t,e}_p) + \tilde{\xi}^{t,e}_p, z^{t,e}_p-z^{t,e+1}_p \rangle = \underbrace{\langle \tilde{\xi}^{t,e}_p, z^{t,e}_p-z^{t,e+1}_p \rangle}_{ \text{applying Young's inequality} }  + \underbrace{\langle \nabla f_p(z^{t,e}_p), z^{t,e}_p-z^{t,e+1}_p \rangle}_{\text{applying the L-smoothness of $f_p$}}  \nonumber \\
\leq & \Big\{
\frac{1}{2(1/\eta^{t} - L)} \| \tilde{\xi}^{t,e}_p \|^2 +  \frac{1/\eta^{t} - L} {2} \| z_p^{t,e+1} - z^{t,e}_p\|^2  \Big\}
+ \Big\{ f_p(z_p^{t,e}) - f_p(z_p^{t,e+1}) + \frac{L}{2}\| z_p^{t,e+1} - z_p^{t,e} \|^2 \Big\}, \nonumber 
\end{align}
where $1/\eta^t - L > 0$ by the construction of $\eta^t$ in \eqref{smooth_proximity}.
Therefore, we derive from \eqref{Block2:smooth_inequality_1} the following inequalities:
\begin{align}
  & f_p(z^{t,e+1}_p) - f_p(z_p) - \langle  \lambda^{t+1}_p, z^{t,e+1}_p-z_p \rangle \leq \rho^t \langle z^{t+1}_p - z^{t,e+1}_p, z^{t,e+1}_p-z_p \rangle +  \nonumber \\
  &  \frac{1}{2(1/\eta^{t} - L)} \| \tilde{\xi}^{t,e}_p \|^2 + \frac{1}{2\eta^{t}} \Big( \|z_p-z^{t,e}_p\|^2- \|z_p-z^{t,e+1}_p\|^2   \Big) +  \langle \tilde{\xi}^{t,e}_p, z_p - z^{t,e}_p \rangle, \ \forall z_p \in \mathcal{W}. \label{Block2:smooth_inequality_2}
\end{align}

\subsection{Inequality derivation for a fixed iteration $t$.}
\noindent
Summing \eqref{Block2:smooth_inequality_2} over all $e \in [E]$ and dividing the resulting inequalities by $E$, we obtain 
\begin{align}
  & \frac{1}{E} \sum_{e=1}^E f_p(z^{t,e+1}_p) - f_p(z_p) -  \langle \lambda^{t+1}_p, \underbrace{\frac{1}{E} \sum_{e=1}^E z^{t,e+1}_p}_{\substack{=z^{t+1}_p \text{ from line 21} \\ \text{in Algorithm \ref{algo:DP-IADMM-Prox}} }}-z_p \rangle \leq \rho^t  \underbrace{ \frac{1}{E} \sum_{e=1}^E \langle z^{t+1}_p - z_p^{t,e+1}, z_p^{t,e+1} - z_p \rangle}_{\text{``\texttt{E}''}} + \nonumber \\
  & \frac{1}{E} \sum_{e=1}^E \Big\{ \frac{\| \tilde{\xi}^{t,e}_p \|^2 }{2(1/\eta^{t} - L)} + \frac{1}{2\eta^{t}} \Big( \|z_p-z^{t,e}_p\|^2- \|z_p-z^{t,e+1}_p\|^2 \Big) + \langle  \tilde{\xi}^{t,e}_p, z_p - z^{t,e}_p \rangle \Big\}, \ \forall z_p \in \mathcal{W}. \label{Block2:smooth_inequality_3}
\end{align}
The ``\texttt{E}'' term from \eqref{Block2:smooth_inequality_3} is non-positive because
\begin{align}
  & \frac{1}{E}\sum_{e=1}^E \langle z_p^{t+1} - z_p^{t,e+1}, z_p^{t,e+1} - z_p \rangle = \frac{1}{E^2}\sum_{e=1}^E \sum_{e'=1}^E \langle z_p^{t,e'+1} - z_p^{t,e+1}, z_p^{t,e+1} - z_p \rangle \nonumber \\
  = & \frac{1}{E^2} \sum_{e=1}^E \sum_{e'=1:e'>e}^E \left( \langle z_p^{t,e'+1} - z_p^{t,e+1}, z_p^{t,e+1} - z_p \rangle + \langle z_p^{t,e+1} - z_p^{t,e'+1}, z_p^{t,e'+1} - z_p \rangle \right) \nonumber \\
  = & \frac{1}{E^2} \sum_{e=1}^E \sum_{e'=1:e'>e}^E \langle z_p^{t,e'+1} - z_p^{t,e+1}, - z_p^{t,e'+1} + z_p^{t,e+1} \rangle  \nonumber \\
  \leq & \frac{1}{E^2} \sum_{e=1}^E \sum_{e'=1:e'>e}^E - \| z_p^{t,e'+1} - z_p^{t,e+1}\|^2 \leq 0. \label{non_positive_trick}
\end{align}
Summing the inequalities resulting from \eqref{Block2:smooth_inequality_3} and \eqref{non_positive_trick} over $p \in [P]$, we have
\begin{align}
  & \sum_{p=1}^P \Big[ \frac{1}{E} \sum_{e=1}^E f_p(z^{t,e+1}_p) - f_p(z_p) -  \langle \lambda^{t+1}_p, z^{t+1}_p - z_p \rangle \Big]  \nonumber \\
  \leq  & \sum_{p=1}^P \Big[ \frac{1}{E} \sum_{e=1}^E \Big\{ \frac{\| \tilde{\xi}^{t,e}_p \|^2 }{2(1/\eta^{t} - L)} + \frac{1}{2\eta^{t}} \Big( \|z_p-z^{t,e}_p\|^2- \|z_p-z^{t,e+1}_p\|^2 \Big) + \langle  \tilde{\xi}^{t,e}_p, z_p - z^{t,e}_p \rangle \Big\} \Big]. \label{Block2:smooth_inequality_4}
\end{align}
For ease of exposition, we introduce the following notation:
\begin{align}
& 
z := [z_1^{\top}, \ldots, z_P^{\top}]^{\top},  
\ \ \lambda := [\lambda_1^{\top}, \ldots, \lambda_P^{\top}]^{\top}, 
\ \ \tilde{\lambda} := [\tilde{\lambda}_1^{\top}, \ldots, \tilde{\lambda}_P^{\top}]^{\top},  
\label{Notation_smooth}   \\  
& x := \begin{bmatrix}
  w \\ z \\ \lambda
  \end{bmatrix}, \  
  \tilde{x}^t := \begin{bmatrix}
  w^{t+1} \\ z^{t+1} \\ \tilde{\lambda}^t
  \end{bmatrix}, \
  x^* := \begin{bmatrix}
    w^* \\ z^* \\ \lambda
  \end{bmatrix}, \
  A := \begin{bmatrix}
    \mathbb{I}_J \\ \vdots \\ \mathbb{I}_J
  \end{bmatrix}_{PJ \times J}, \
  G := \begin{bmatrix}
      0 & 0 & A^{\top}     \\
      0 & 0 & -\mathbb{I}_{PJ}     \\
      -A & \mathbb{I}_{PJ} & 0     \\
  \end{bmatrix}, \nonumber  \\    
& x^{(T)} := \textstyle \frac{1}{T} \sum_{t=1}^T \tilde{x}^{t}, \ \
w^{(T)} := \textstyle \frac{1}{T} \sum_{t=1}^T w^{t+1}, \ \
z^{(T)} := \textstyle \frac{1}{TE} \sum_{t=1}^T \sum_{e=1}^{E} z^{t,e+1}, \ \
\lambda^{(T)} := \textstyle \frac{1}{T} \sum_{t=1}^T \tilde{\lambda}^{t}, \nonumber \\
& \textstyle A^{\top} \tilde{\lambda}^t = \sum_{p=1}^P \tilde{\lambda}^t_p , 
\ \  F(z) := \textstyle \sum_{p=1}^P f_p(z_p),
\ \ \tilde{\xi}^{t,e} := [(\tilde{\xi}^{t,e}_1)^{\top}, \ldots, (\tilde{\xi}^{t,e}_P)^{\top}]^{\top}. \nonumber
\end{align}
Based on the above notation as well as \eqref{Block1:optimality_condition} and \eqref{Block2:smooth_inequality_4}, we derive $\text{LHS}^t(w^*,z^*) \leq \text{RHS}^t(z^*)$ at optimal $w^*$ and $\{z^*_p\}_{p=1}^P \in \mathcal{W}$, where
\begin{subequations}
\begin{align}
  & \text{LHS}^t(w^*,z^*) := \frac{1}{E}\sum_{e=1}^E F(z^{t,e+1}) - F(z^*) - \langle \lambda^{t+1}, z^{t+1}- z^* \rangle + \langle A^{\top} \tilde{\lambda}^t, w^{t+1}-w^* \rangle, \label{Smooth_LHS_t}  \\
  & \text{RHS}^t(z^*) := \frac{1}{E}\sum_{e=1}^E \Big\{ \frac{  \| \tilde{\xi}^{t,e} \|^2}{2(1/\eta^{t}-L)}    + \frac{1}{2\eta^{t}} \big( \|z^* -z^{t,e}\|^2  - \| z^* - z^{t,e+1} \|^2 \big)  + \langle \tilde{\xi}^{t,e}, z^* - z^{t,e} \rangle \Big\}. \label{Smooth_RHS_t}
\end{align}
\end{subequations}

\subsection{Lower bound on $\text{LHS}^t(w^*,z^*)$.} \label{apx-sec:smooth_LB_LHS}
\noindent
Recall that
\begin{align}
  \lambda^{t+1} = \lambda^t + \rho^t (A w^{t+1} - z^{t+1}), \ \  \tilde{\lambda}^{t} = \lambda^t + \rho^t (A w^{t+1} - z^t).  \label{lambdas}
\end{align}
By utilizing \eqref{lambdas}, we rewrite $\text{LHS}^t(w^*,z^*)$ in \eqref{Smooth_LHS_t} as follows:
\begin{subequations} 
\begin{align}
\text{LHS}^t(w^*,z^*) = \frac{1}{E}\sum_{e=1}^E F(z^{t,e+1}) - F(z^*) + \Biggr\langle
\begin{bmatrix}
w^{t+1} - w^* \\
z^{t+1} - z^* \\
\tilde{\lambda}^t - \lambda
\end{bmatrix}
, \
\begin{bmatrix}
A^{\top} \tilde{\lambda}^t             \\
-\tilde{\lambda}^t \\
- Aw^{t+1} + z^{t+1}
\end{bmatrix}
-
\begin{bmatrix}
0 \\
\rho^t (z^t - z^{t+1})        \\
(\lambda^t- \lambda^{t+1} ) / \rho^t
\end{bmatrix}
\Biggr \rangle.  \label{Smooth_LHS_t_LB_1}
\end{align}
The third term in \eqref{Smooth_LHS_t_LB_1} can be written as
  \begin{align}
      &\Biggr \langle
      \begin{bmatrix}
          w^{t+1} - w^* \\
          z^{t+1} - z^*  \\
          \tilde{\lambda}^t - \lambda
      \end{bmatrix}
      ,
      \begin{bmatrix}
      A^{\top} \tilde{\lambda}^t \\
      - \tilde{\lambda}^t \\
      -Aw^{t+1} + z^{t+1}
      \end{bmatrix}
      \Biggr \rangle
      & = \langle \tilde{x}^{t} - x^*, G\tilde{x}^{t} \rangle   
      = \underbrace{\langle \tilde{x}^{t} - x^*, G(\tilde{x}^{t} - x^*) \rangle}_{= 0 \text{ as $G$ is skew-symmetric}} + \langle \tilde{x}^{t} - x^* , G x^* \rangle \nonumber \\
      & = \langle \tilde{x}^{t} - x^*, G x^*  \rangle, \label{Smooth_LHS_t_LB_2}
  \end{align}
Based on \eqref{rule}, the last term in \eqref{Smooth_LHS_t_LB_1} can be written as
    \begin{align}
        & \Biggr\langle
        \begin{bmatrix}
        w^* - w^{t+1} \\
        z^* - z^{t+1}   \\
        \lambda - \tilde{\lambda}^t
        \end{bmatrix}
        , \
        \begin{bmatrix}
        0 \\
        \rho^t (z^t-z^{t+1})  \\
        (\lambda^t- \lambda^{t+1} ) / \rho^t
        \end{bmatrix}
        \Biggr \rangle 
         \nonumber \\
        =  & \textstyle \frac{\rho^t}{2} \big( \| z^* - z^{t+1}\|^2 - \|z^* - z^t\|^2 + \|z^{t+1} - z^t\|^2 \big)+ \frac{1}{2\rho^t} \big( \| \lambda - \lambda^{t+1} \|^2 - \| \lambda - \lambda^t \|^2 + \underbrace{\|\tilde{\lambda}^{t} - \lambda^t \|^2}_{\geq 0}  - \underbrace{\| \tilde{\lambda}^t - \lambda^{t+1} \|^2}_{=\|\rho^t(z^{t+1}-z^t)\|^2} \big) \nonumber \\
        \geq & \textstyle  \frac{\rho^t}{2} \big( \|z^* - z^{t+1}\|^2  - \|z^* - z^t\|^2 \big) +  \frac{1}{2\rho^t}  \big( \| \lambda - \lambda^{t+1} \|^2  - \| \lambda - \lambda^t \|^2 \big). 
    \end{align}  
Therefore, we have
\begin{align}
  & \text{LHS}^t(w^*,z^*) \geq \frac{1}{E}\sum_{e=1}^E F(z^{t,e+1}) - F(z^*) + \langle \tilde{x}^{t} - x^*, G x^*  \rangle  \nonumber \\
  & \hspace{12mm} + \frac{\rho^t}{2} \big( \|z^* - z^{t+1}\|^2  - \|z^*-z^t\|^2 \big) + \frac{1}{2\rho^t} \big( \| \lambda - \lambda^{t+1} \|^2  - \| \lambda - \lambda^t \|^2 \big).  \label{Smooth_LHS_t_LB_4}
\end{align}
\end{subequations}

\subsection{Lower bound on $\text{LHS}(w^*,z^*) :=\frac{1}{T} \sum_{t=1}^T \ \text{LHS}^t(w^*,z^*)$.}
\noindent
Summing \eqref{Smooth_LHS_t_LB_4} over $t \in [T]$ and dividing the resulting inequality by $T$, we have 
\begin{align}
& \text{LHS}(w^*,z^*) \geq \underbrace{\frac{1}{TE} \sum_{t=1}^T \sum_{e=1}^E F(z^{t,e+1})}_{\geq F(z^{(T)}) \text{ as F is convex}} -  F(z^*) + \underbrace{\langle x^{(T)} - x^* , Gx^* \rangle}_{\text{``\texttt{F}''}} \nonumber \\
& \hspace{20mm} + \frac{1}{T} \Big\{ \underbrace{\sum_{t=1}^T  \frac{ \rho^t}{2}  \big( \|z^* - z^{t+1}\|^2  - \|z^*-z^t\|^2 \big)}_{\text{``\texttt{G}''}} +  \underbrace{\sum_{t=1}^T  \frac{1}{2\rho^t} \big( \| \lambda - \lambda^{t+1} \|^2  - \| \lambda - \lambda^t \|^2 \big)}_{\text{``\texttt{H}''}} \Big\}. \label{LB_LHS_1}
\end{align}
The ``\texttt{F}'' term in \eqref{LB_LHS_1} can be written as
\begin{align*}    
\big\langle x^{(T)} - x^*, Gx \big\rangle & = \langle A w^{(T)} - z^{(T)} \underbrace{- Aw^* + z^*}_{=0}, \lambda \rangle - \langle \lambda^{(T)} - \lambda, \underbrace{Aw^*-z^*}_{=0} \rangle = \langle \lambda, Aw^{(T)} - z^{(T)} \rangle.
\end{align*}
The ``\texttt{G}'' term in \eqref{LB_LHS_1} can be written as
\begin{align*}    
& \sum_{t=1}^T \frac{\rho^t}{2}  \big( \|z^* - z^{t+1}\|^2  - \|z^*-z^t\|^2 \big) = - \frac{\rho^1}{2} \|z^*-z^1\|^2 + \sum_{t=2}^T \underbrace{\Big( \frac{\rho^{t-1} - \rho^t}{2} \Big)}_{\leq 0}\|z^* - z^t\|^2 + \underbrace{\frac{\rho^T}{2}\|z^* - z^{T+1}\|^2}_{\geq 0} \nonumber  \\
& \geq  - \frac{\rho^1}{2} U_2^2 + \sum_{t=2}^T \Big( \frac{\rho^{t-1} - \rho^t}{2} \Big) U_2^2 = \frac{ - U_2^2 \rho^{T}}{2} \geq \frac{ - U_2^2 \rho^{\text{max}}}{2}. 
\end{align*}
The ``\texttt{H}'' term in \eqref{LB_LHS_1} can be written as
\begin{align*}              
& \sum_{t=1}^T \frac{1}{2\rho^t} \big( \| \lambda - \lambda^{t+1} \|^2  - \| \lambda - \lambda^t \|^2 \big) = - \frac{1}{2\rho^1} \|\lambda-\lambda^1\|^2 + \underbrace{\sum_{t=2}^T \Big( \frac{1}{2 \rho^{t-1}} - \frac{1}{2 \rho^{t}} \Big)\|\lambda - \lambda^t\|^2 + \frac{1}{2\rho^T}\|\lambda - \lambda^{T+1}\|^2}_{\geq 0} \nonumber  \\
& \geq - \frac{1}{2\rho^1} \|\lambda-\lambda^1\|^2.
\end{align*}
Therefore, we derive
\begin{align}
  & \text{LHS}(w^*,z^*) \geq F(z^{(T)}) - F(z^*) +  \langle \lambda, A w^{(T)} - z^{(T)} \rangle - \frac{1}{T} \Big( \frac{U_2^2  \rho^{\text{max}}}{2} + \frac{1}{2\rho^1} \|\lambda-\lambda^1\|^2 \Big). \label{Smooth_LHS_LB}
\end{align}
Since this inequality holds for any $\lambda$, we select $\lambda$ that maximizes the right-hand side of \eqref{Smooth_LHS_LB} subject to a ball centered at zero with the radius $\gamma$:
\begin{subequations}
\label{max_lambdas}  
\begin{align}
\bullet \ & \max_{\lambda: \| \lambda \| \leq \gamma } \ \langle \lambda, Aw^{(T)}-z^{(T)} \rangle = \gamma \| Aw^{(T)}-z^{(T)} \|,  \\  
\bullet \ & \max_{\lambda: \| \lambda \| \leq \gamma } \ \| \lambda - \lambda^1 \|^2 = \| \lambda^1 \|^2 + \max_{\lambda: \| \lambda \| \leq \gamma } \ \{ \| \lambda\|^2 - 2 \langle \lambda, \lambda^1\rangle \} \leq  (\gamma + \|\lambda^1\|)^2.
\end{align}
\end{subequations}
Based on \eqref{Smooth_LHS_LB} and \eqref{max_lambdas}, we derive
\begin{align}
  \text{LHS}(w^*,z^*) & \geq F(z^{(T)}) - F(z^*) + \gamma \| Aw^{(T)}-z^{(T)} \|  -   \frac{U_2^2  \rho^{\text{max}} +  (\gamma + \|\lambda^1\|)^2 / \rho^1 }{2T}. \label{Smooth_LHS_t_LB_3}
\end{align}

\subsection{Upper bound on $\text{RHS}(z^*) := \frac{1}{T} \sum_{t=1}^T \ \text{RHS}^t(z^*)$.} \label{apx-sec:RHS_UB}
\noindent 
It follows from \eqref{Smooth_RHS_t} that
\begin{align*}
\text{RHS}(z^*) = \frac{1}{TE} \Big[ \sum_{t=1}^T \sum_{e=1}^E  \Big\{  \frac{ \| \tilde{\xi}^{t,e} \|^2}{2(1/\eta^{t}-L)} + \langle \tilde{\xi}^{t,e}, z^* - z^{t,e} \rangle \Big\} + \underbrace{ \sum_{t=1}^T \sum_{e=1}^E \frac{1}{2\eta^{t}} ( \|z^* -z^{t,e} \|^2  - \|z^*-z^{t,e+1}\|^2) }_{\text{``\texttt{I}''}} \Big]. 
\end{align*}
The ``\texttt{I}'' term from the above can be written as
\begin{align}
  & \sum_{t=1}^T \sum_{e=1}^E \frac{1}{2\eta^{t}}\big(\|z^*-z^{t,e}\|^2-\|z^*-z^{t,e+1}\|^2 \big) = \sum_{t=1}^T  \frac{1}{2\eta^{t}}\big(\|z^*-z^{t,1}\|^2-\|z^*- \underbrace{z^{t,E+1}}_{= z^{t+1,1}}\|^2 
  \big) \nonumber \\
  & = \frac{1}{2\eta^1} \| z^* - z^{1,1} \|^2 + \sum_{t=2}^T \underbrace{\Big( \frac{1}{2\eta^{t}} - \frac{1}{2\eta^{t-1}} \Big)}_{\geq 0} \|z^*-z^{t,1}\|^2 \underbrace{ - \frac{1}{2\eta^{T}} \| z^*- z^{T,E+1}  \|^2}_{\leq 0}  \leq \frac{U_2^2}{2\eta^{T}}. \nonumber
\end{align}
Therefore, we have
\begin{align}
\text{RHS}(z^*) \leq \frac{U_2^2}{2 T E \eta^{T}} + \frac{1}{TE} \sum_{t=1}^T \sum_{e=1}^E \Big\{ \frac{ \| \tilde{\xi}^{t,e} \|^2 }{2(1/\eta^{t} - L)} +  \langle \tilde{\xi}^{t,e}, z^* - z^{t,e} \rangle \Big\}. \label{Smooth_RHS_t_LB_1}
\end{align}  

\subsection{Taking expectation.}
\noindent
By taking expectation on the inequality derived from \eqref{Smooth_LHS_t_LB_3} and \eqref{Smooth_RHS_t_LB_1}, we have
\begin{align*}
  & \mathbb{E} \Big[ F(z^{(T)}) - F(z^*) + \gamma \| Aw^{(T)} - z^{(T)} \| \Big]  \leq \frac{U_2^2  \rho^{\text{max}} +  (\gamma + \|\lambda^1\|)^2 / \rho^1 }{2T} \nonumber \\
  & + \frac{U_2^2}{2 T E \eta^{T}} + \frac{1}{TE} \sum_{t=1}^T \sum_{e=1}^E \Big\{ \frac{ 1 }{2(1/\eta^{t} - L)} \sum_{p=1}^P \underbrace{\mathbb{E}[\| \tilde{\xi}^{t,e}_p \|^2]}_{\leq \overline{U}(\bar{\epsilon})} +  \langle \underbrace{\mathbb{E}[\tilde{\xi}^{t,e}]}_{=0}, z^* - z^{t,e} \rangle \Big\},
\end{align*}
where
\begin{align}
  & \overline{U}(\bar{\epsilon}) :=  2 JK U_3^2/\bar{\epsilon}^2  \geq   \sum_{j=1}^J\sum_{k=1}^K 2 (\bar{\Delta}^{t,e}_p)^2/\bar{\epsilon}^2 = \sum_{j=1}^J \sum_{k=1}^K \mathbb{E}[ (\tilde{\xi}^{t,e}_{pjk})^2 ] = \mathbb{E} \big[ \|\tilde{\xi}^{t,e}_p \|^2 \big]. \label{U_eps} 
\end{align}
By noting that
\begin{align*}  
  & \frac{U _2^2}{2 T E \eta^{T}} = \frac{U_2^2 L}{2TE} + \frac{U_2^2 }{2 \bar{\epsilon} E \sqrt{T}},  \\    
  & \sum_{t=1}^T \sum_{e=1}^E \frac{1}{2(1/\eta^{t} - L)} = E \bar{\epsilon} \sum_{t=1}^T \frac{1}{2 \sqrt{t}} \leq E \bar{\epsilon} \sum_{t=1}^T \frac{1}{\sqrt{t} + \sqrt{t-1}} =  E \bar{\epsilon} \sum_{t=1}^T (\sqrt{t} - \sqrt{t-1}) = E \bar{\epsilon} \sqrt{T},  
\end{align*}
we have
\begin{align*}
  & \mathbb{E} \Big[ F(z^{(T)}) - F(z^*) + \gamma \| Aw^{(T)} - z^{(T)} \| \Big]  \leq \frac{U_2^2  (\rho^{\text{max}} +L/E) +  (\gamma + \|\lambda^1\|)^2 / \rho^1 }{2T}   + \frac{ 2PJKU_3^2 + U_2^2/(2E) }{ \bar{\epsilon} \sqrt{T} }.
\end{align*}
This completes the proof.

 
\section{Proof of Theorem \ref{thm:nonsmooth} } \label{apx-thm:nonsmooth}
\noindent
The proof in this section is similar to that in Appendix \ref{apx-thm:smooth} except that
\begin{enumerate}
\item the $L$-smoothness of $f_p$ can no longer be applied to the ``\texttt{D}'' term in \eqref{Block2:smooth_inequality_1} when deriving an upper bound of the term in a nonsmooth setting and
\item the definition of $(\tilde{x}^t, z^{(T)})$ is different from that in \eqref{Notation_smooth} of Appendix \ref{apx-thm:smooth}.
\end{enumerate}

\subsection{Inequality derivation for a fixed iteration $t$ and $e$.}
\noindent  
Applying Young's inequality on the ``\texttt{D}'' term in \eqref{Block2:smooth_inequality_1} yields
\begin{align}
& f_p(z^{t,e}_p) - f_p(z_p) -  \langle \lambda^{t+1}_p, z^{t,e+1}_p-z_p \rangle \leq \rho^t \langle z^{t+1}_p - z_p^{t,e+1}, z_p^{t,e+1} - z_p \rangle  + \nonumber \\
& \frac{\eta^{t}}{2} \| f'_p(z^{t,e}_p) + \tilde{\xi}^{t,e}_p\|^2  + \frac{1}{2\eta^{t}} \Big\{ \|z_p-z^{t,e}_p\|^2- \|z_p-z^{t,e+1}_p\|^2 \Big\}  + \langle  \tilde{\xi}^{t,e}_p, z_p - z^{t,e}_p \rangle, \ \forall z_p \in \mathcal{W}. \label{Block2:nonsmooth_inequality_1}
\end{align}

\subsection{Inequality derivation for a fixed iteration $t$.} \label{apx-sec:nonsmooth_ineq}
\noindent
Summing \eqref{Block2:nonsmooth_inequality_1} over all $e \in [E]$ and dividing the resulting inequalities by $E$, we get 
\begin{align}
  & \frac{1}{E} \sum_{e=1}^E f_p(z^{t,e}_p) - f_p(z_p) -  \langle \lambda^{t+1}_p, \underbrace{\frac{1}{E} \sum_{e=1}^E z^{t,e+1}_p}_{=z^{t+1}_p}-z_p \rangle \leq  \rho^t  \underbrace{ \frac{1}{E} \sum_{e=1}^E \langle z_p^{t+1} - z_p^{t,e+1}, z_p^{t,e+1} - z_p \rangle}_{\leq 0 \text{ from \eqref{non_positive_trick}} } + \nonumber \\
  & \frac{1}{E} \sum_{e=1}^E \Big\{ \frac{\eta^{t}}{2} \| f'_p(z^{t,e}_p) + \tilde{\xi}^{t,e}_p\|^2  + \frac{1}{2\eta^{t}} \Big( \|z_p-z^{t,e}_p\|^2- \|z_p-z^{t,e+1}_p\|^2 \Big) + \langle  \tilde{\xi}^{t,e}_p, z_p - z^{t,e}_p \rangle \Big\}. \label{Block2:nonsmooth_inequality_2}
\end{align}
Summing the inequalities \eqref{Block2:nonsmooth_inequality_2} over $p \in [P]$, we have 
\begin{align}
  &\sum_{p=1}^P \Big[ \frac{1}{E} \sum_{e=1}^E f_p(z^{t,e}_p) - f_p(z_p) -  \langle \lambda^{t+1}_p,  z^{t+1}_p -z_p \rangle \Big] \leq   \nonumber \\
  &\sum_{p=1}^P \Big[ \frac{1}{E} \sum_{e=1}^E \Big\{ \frac{\eta^{t}}{2} \| f'_p(z^{t,e}_p) + \tilde{\xi}^{t,e}_p\|^2  + \frac{1}{2\eta^{t}} \Big( \|z_p-z^{t,e}_p\|^2- \|z_p-z^{t,e+1}_p\|^2 \Big) + \langle  \tilde{\xi}^{t,e}_p, z_p - z^{t,e}_p \rangle \Big\} \Big].  \label{Block2:nonsmooth_inequality_3}
\end{align}

\noindent
For ease of exposition, we introduce $z, \lambda, \tilde{\lambda}, x, x^*, A, G, x^{(T)}, w^{(T)}, \lambda^{(T)}, A^{\top} \tilde{\lambda}^t, F(z), \tilde{\xi}^{t,e}$ defined in \eqref{Notation_smooth} with modifications of the following notation:
\begin{subequations}
\label{nonsmooth_notations}
\begin{align}
& \tilde{x}^t := \begin{bmatrix}
  w^{t+1} \\ z^{t} \\ \tilde{\lambda}^t
  \end{bmatrix}, \ \
z^{(T)} := \textstyle \frac{1}{TE} \sum_{t=1}^T \sum_{e=1}^{E} z^{t,e}.
\end{align}
We also define
\begin{align}
f'(z):=[ f'_1(z_1)^{\top}, \ldots, f'_P(z_P)^{\top} ]^{\top}.
\end{align}
\end{subequations}
Based on this notation as well as \eqref{Block1:optimality_condition} and \eqref{Block2:nonsmooth_inequality_3}, we derive $\text{LHS}^t(w^*,z^*) \leq \text{RHS}^t(z^*)$ at optimal $w^*$ and $\{z^*_p\}_{p=1}^P \in \mathcal{W}$, where
\begin{subequations}
\begin{align}
  & \text{LHS}^t(w^*,z^*) :=  \frac{1}{E}\sum_{e=1}^E F(z^{t,e}) - F(z^*) - \langle \lambda^{t+1}, z^{t+1}- z^* \rangle + \langle A^{\top} \tilde{\lambda}^t, w^{t+1}-w^* \rangle, \label{Nonsmooth_LHS_t}  \\
  & \text{RHS}^t(z^*) :=  \frac{1}{E}\sum_{e=1}^E \Big\{ \frac{\eta^{t}}{2} \| f'(z^{t,e}) + \tilde{\xi}^{t,e} \|^2   + \frac{1}{2\eta^{t}} \big( \|z^* -z^{t,e}\|^2  - \| z^* - z^{t,e+1} \|^2 \big)  + \langle \tilde{\xi}^{t,e}, z^* - z^{t,e} \rangle \Big\}. \label{Nonsmooth_RHS_t}
\end{align}
\end{subequations}

\subsection{Lower bound on $\text{LHS}^t(w^*,z^*)$.}
\label{apx:nonsmooth-B}
\noindent
By following the steps in Appendix \ref{apx-sec:smooth_LB_LHS}, one can derive inequalities similar to \eqref{Smooth_LHS_t_LB_4}, as follows:
\begin{align}
  & \text{LHS}^t(w^*,z^*) \geq \frac{1}{E}\sum_{e=1}^E F(z^{t,e}) - F(z^*) + \langle \tilde{x}^{t} - x^*, G x^*  \rangle \underbrace{- \langle \lambda, z^{t+1} - z^t \rangle}_{\text{``\texttt{J}''}} \nonumber \\
  & \hspace{12mm} + \frac{\rho^t}{2} \big( \|z^* - z^{t+1}\|^2  - \|z^*-z^t\|^2 \big) + \frac{1}{2\rho^t} \big( \| \lambda - \lambda^{t+1} \|^2  - \| \lambda - \lambda^t \|^2 \big).  \label{Nonsmooth_LHS_t_LB}
\end{align}
Note that the ``\texttt{J}'' term in \eqref{Nonsmooth_LHS_t_LB} does not exist in \eqref{Smooth_LHS_t_LB_4} because the definition of $\tilde{x}^t$ in \eqref{nonsmooth_notations} is different from that in \eqref{Notation_smooth}.

\subsection{Lower bound on $\text{LHS}(w^*,z^*) :=\frac{1}{T} \sum_{t=1}^T \ \text{LHS}^t(w^*,z^*)$.}
\noindent 
Summing \eqref{Nonsmooth_LHS_t_LB} over $t \in [T]$ and dividing the resulting inequality by $T$, we have
\begin{align}
& \text{LHS}(w^*,z^*) \geq \underbrace{\frac{1}{TE}  \sum_{t=1}^T \sum_{e=1}^E F(z^{t,e}) }_{\geq F(z^{(T)}) \text{ as F is convex} }-  F(z^*) + \underbrace{\langle x^{(T)} - x^* , Gx^* \rangle}_{= \text{``\texttt{F}'' term in \eqref{LB_LHS_1}} } - \underbrace{\frac{1}{T} \langle \lambda, z^{T+1} - z^1 \rangle}_{\text{``\texttt{K}''}} \nonumber \\
& \hspace{20mm} + \frac{1}{T} \Big\{ \underbrace{\sum_{t=1}^T  \frac{ \rho^t}{2}  \big( \|z^* - z^{t+1}\|^2  - \|z^*-z^t\|^2 \big)}_{= \text{``\texttt{G}'' term in \eqref{LB_LHS_1}}} +  \underbrace{\sum_{t=1}^T  \frac{1}{2\rho^t} \big( \| \lambda - \lambda^{t+1} \|^2  - \| \lambda - \lambda^t \|^2 \big)}_{= \text{``\texttt{H}'' term in \eqref{LB_LHS_1}}} \Big\}. \label{nonsmooth_LB_LHS_1}
\end{align}
The ``\texttt{K}'' term in \eqref{nonsmooth_LB_LHS_1} can be written as
\begin{align*}  
- \frac{1}{T} \langle \lambda, z^{T+1} - z^1 \rangle \geq - \frac{1}{T}  \| \lambda \| \| z^{T+1} - z^1 \| \geq - \| \lambda \| U_2.
\end{align*}
Therefore, we have
\begin{align}
  & \text{LHS}(w^*,z^*) \geq F(z^{(T)}) - F(z^*) +  \langle \lambda, A w^{(T)} - z^{(T)} \rangle - \frac{1}{T} \Big(\| \lambda \| U_2 + \frac{U_2^2  \rho^{\text{max}}}{2} + \frac{1}{2\rho^1} \|\lambda-\lambda^1\|^2 \Big). \nonumber
\end{align}
Based on \eqref{max_lambdas} and $\max_{\lambda: \| \lambda \| \leq \gamma } \ \| \lambda \| U_2 \leq \gamma U_2$, we have
\begin{align}
  \text{LHS}(w^*,z^*) & \geq F(z^{(T)}) - F(z^*) + \gamma \| Aw^{(T)}-z^{(T)} \|  -   \frac{U_2^2 \rho^{\text{max}} + (\gamma + \|\lambda^1\|)^2 / \rho^1 + 2 \gamma U_2 }{2T}. \label{Nonsmooth_LHS_LB_7}
\end{align}

\subsection{Upper bound on $\text{RHS}(z^*) := \frac{1}{T} \sum_{t=1}^T \ \text{RHS}^t(z^*)$.}
\noindent 
By following the steps in Appendix \ref{apx-sec:RHS_UB}, we obtain
\begin{align}
\text{RHS}(z^*) \leq \frac{U_2^2}{2 TE \eta^{T}} + \frac{1}{TE} \sum_{t=1}^T \sum_{e=1}^E \Big\{ \frac{\eta^{t}}{2} \| f'(z^{t,e}) + \tilde{\xi}^{t,e} \|^2 +  \langle \tilde{\xi}^{t,e}, z^* - z^{t,e} \rangle \Big\}. \label{Nonsmooth_RHS_UB_2}
\end{align}  

\subsection{Taking expectation.}
\noindent
By taking expectation on the inequality derived from \eqref{Nonsmooth_LHS_LB_7} and \eqref{Nonsmooth_RHS_UB_2}, we have
\begin{align*}
  & \mathbb{E} \Big[ F(z^{(T)}) - F(z^*) + \gamma \| Aw^{(T)} - z^{(T)} \| \Big]  \leq \frac{U_2^2 \rho^{\text{max}} + (\gamma + \|\lambda^1\|)^2 / \rho^1 + 2\gamma U_2}{2T} \nonumber \\
  & + \frac{U_2^2}{2 TE \eta^{T}} + \frac{1}{TE} \sum_{t=1}^T \sum_{e=1}^E \Big\{ \frac{\eta^{t}}{2} \sum_{p=1}^P \underbrace{\mathbb{E}[ \| f'_p(z_p^{t,e}) + \tilde{\xi}^{t,e}_p \|^2 ]}_{\leq U_1^2+\overline{U}(\bar{\epsilon})} +  \langle \underbrace{\mathbb{E}[\tilde{\xi}^{t,e}]}_{=0}, z^* - z^{t,e} \rangle \Big\},
\end{align*}
where $\overline{U}(\bar{\epsilon})$ is from \eqref{U_eps}.
We note that
\begin{align*}
  & \frac{U_2^2}{2 TE \eta^{T}} = \frac{U_2^2 /(2E)}{ \sqrt{T}},  \\
  & \sum_{t=1}^T \sum_{e=1}^E \frac{\eta^{t}}{2}   = E \sum_{t=1}^T \frac{1}{2 \sqrt{t}} \leq E \sum_{t=1}^T \frac{1}{\sqrt{t} + \sqrt{t-1}} = E \sum_{t=1}^T   (\sqrt{t} - \sqrt{t-1}) =  E  \sqrt{T}.
\end{align*}
Therefore, we have
\begin{align*}
  & \mathbb{E} \Big[ F(z^{(T)}) - F(z^*) + \gamma \| Aw^{(T)} - z^{(T)} \| \Big]  \leq \frac{U_2^2 \rho^{\text{max}} + (\gamma + \|\lambda^1\|)^2 / \rho^1 +  2\gamma U_2}{2T}   + \frac{ 2PJKU_3^2 /\bar{\epsilon} + PU_1^2 + U_2^2 / (2E) }{ \sqrt{T} }.
\end{align*}
This completes the proof. 


\section{Proof of Theorem \ref{thm:strong} } \label{apx-thm:strong}
\noindent
The proof in this section is similar to that in Appendix \ref{apx-thm:nonsmooth} except that
\begin{enumerate}
  \item the $\alpha$-strong convexity of $f_p$ is utilized to tighten the right-hand side of inequality \eqref{Block2:nonsmooth_inequality_1} and
  \item the definition of $x^{(T)}, w^{(T)}, z^{(T)}, \lambda^{(T)}$ is modified to the following:
    \begin{align}  
    & x^{(T)} := \textstyle \frac{2}{T(T+1)} \sum_{t=1}^T t \tilde{x}^{t}, \ \
    w^{(T)} := \textstyle \frac{2}{T(T+1)} \sum_{t=1}^T t w^{t+1},  \label{Notation_strong}  \\
    & z^{(T)} := \textstyle \frac{2}{T(T+1)} \sum_{t=1}^T t  (\frac{1}{E} \sum_{e=1}^E z^{t,e}), \ \
    \lambda^{(T)} := \textstyle \frac{2}{T(T+1)} \sum_{t=1}^T t \tilde{\lambda}^{t}.  \nonumber
    \end{align}
\end{enumerate}

\subsection{Inequality derivation for a fixed iteration $t$ and $e$.}
\noindent
For a given $p \in [P]$, it follows from the $\alpha$-strong convexity of the function $f_p$ that
\begin{align}
f_p(z^{t,e}_p) - f_p(z_p)  \leq \langle f'_p(z^{t,e}_p), z^{t,e}_p - z_p \rangle - \frac{\alpha}{2}\| z_p - z^{t,e}_p\|^2. \label{Block2:strong_inequality_1}
\end{align} 
By utilizing \eqref{Block2:strong_inequality_1} and \eqref{Block2:nonsmooth_inequality_1}, we obtain
\begin{align}
& f_p(z^{t,e}_p) - f_p(z_p) - \langle  \lambda^{t+1}_p, z^{t,e+1}_p-z_p \rangle \leq \rho^t \langle z_p^{t+1} - z_p^{t,e+1}, z_p^{t,e+1} - z_p \rangle  - \frac{\alpha}{2}\| z_p - z^{t,e}_p\|^2 + \nonumber \\
& \frac{\eta^{t}}{2} \| f'_p(z^{t,e}_p) + \tilde{\xi}^{t,e}_p\|^2  + \frac{1}{2\eta^{t}} \Big( \|z_p-z^{t,e}_p\|^2- \|z_p-z^{t,e+1}_p\|^2 \Big ) + \langle  \tilde{\xi}^{t,e}_p, z_p - z^{t,e}_p \rangle, \ \forall z_p \in \mathcal{W}.\label{Block2:strong_inequality_2}
\end{align}
Note that compared with \eqref{Block2:nonsmooth_inequality_1}, the inequalities \eqref{Block2:strong_inequality_2} have an additional term $-(\alpha/2)\|z_p-z_p^{t,e}\|^2$. 

\subsection{Inequality derivation for a fixed iteration $t$.}
\noindent
Following the steps to derive \eqref{Block2:nonsmooth_inequality_3} in Appendix \ref{apx-sec:nonsmooth_ineq}, we derive the following from \eqref{Block2:strong_inequality_2}:
\begin{align}
  & \sum_{p=1}^P \Big[ \frac{1}{E} \sum_{e=1}^E f_p(z^{t,e}_p) - f_p(z_p) -  \langle \lambda^{t+1}_p,  z^{t+1}_p -z_p \rangle \Big] \leq   \nonumber \\
  & \sum_{p=1}^P \Big[ \frac{1}{E} \sum_{e=1}^E \Big\{ \frac{\eta^{t}}{2} \| f'_p(z^{t,e}_p) + \tilde{\xi}^{t,e}_p\|^2  + \big(\frac{1}{2\eta^{t}} - \frac{\alpha}{2} \big) \|z_p-z^{t,e}_p\|^2 - \frac{1}{2\eta^{t}} \|z_p-z^{t,e+1}_p\|^2  + \langle  \tilde{\xi}^{t,e}_p, z_p - z^{t,e}_p \rangle \Big\} \Big]. \label{Block2:strong_inequality_3}
\end{align}
For ease of exposition, we introduce $z, \lambda, \tilde{\lambda}, x, x^*, A, G, A^{\top} \tilde{\lambda}^t, F(z), \tilde{\xi}^{t,e}$ as defined in \eqref{Notation_smooth}, $\tilde{x}^t, f'(z)$ as defined in \eqref{nonsmooth_notations}, and the definition \eqref{Notation_strong}.
Based on this notation as well as \eqref{Block1:optimality_condition} and \eqref{Block2:strong_inequality_3}, we derive $\text{LHS}^t(w^*,z^*) \leq \text{RHS}^t(z^*)$ at optimal $w^*$ and $\{z^*_p\}_{p=1}^P \in \mathcal{W}$, where
\begin{subequations}
  \begin{align}
    & \text{LHS}^t(w^*,z^*) :=  \frac{1}{E} \sum_{e=1}^E F(z^{t,e}) - F(z^*) - \langle \lambda^{t+1}, z^{t+1}- z^* \rangle + \langle A^{\top} \tilde{\lambda}^t, w^{t+1}-w^* \rangle, \label{Strong_LHS_t}  \\
    & \text{RHS}^t(z^*) := \frac{1}{E} \sum_{e=1}^E \Big\{  \frac{\eta^{t}}{2} \| f'(z^{t,e}) + \tilde{\xi}^{t,e} \|^2   + \big(\frac{1}{2\eta^{t}} - \frac{\alpha}{2})  \|z^* -z^{t,e}\|^2  - \frac{1}{2\eta^{t}} \| z^* - z^{t,e+1} \|^2  + \langle \tilde{\xi}^{t,e}, z^* - z^{t,e} \rangle \Big\}. \label{Strong_RHS_t}
  \end{align}
\end{subequations}

\subsection{Lower bound on $\text{LHS}(w^*,z^*) :=\frac{2}{T(T+1)} \sum_{t=1}^T t \ \text{LHS}^t(w^*,z^*)$.}
\noindent
By utilizing \eqref{Nonsmooth_LHS_t_LB}, namely, a lower bound on \eqref{Strong_LHS_t},  we have
\begin{align}
& \text{LHS}(w^*,z^*) \geq \underbrace{\frac{2}{T(T+1)} \sum_{t=1}^T t \big( \frac{1}{E}\sum_{e=1}^E F(z^{t,e}) \big)}_{\geq F(z^{(T)}) \text{ as F is convex}} - F(z^*) +  \underbrace{\langle x^{(T)} - x^*, Gx^* \rangle}_{= \text{``\texttt{F}'' term in \eqref{LB_LHS_1}} } - \frac{2}{T(T+1)} \underbrace{\langle \lambda, \sum_{t=1}^T t (z^{t+1} - z^t ) \rangle}_{\text{``\texttt{L}''}}  \nonumber \\
& \hspace{20mm} + \frac{2}{T(T+1)}  \Big\{  \underbrace{\sum_{t=1}^T \frac{t \rho^t}{2}  \big( \|z^* - z^{t+1}\|^2  - \|z^* - z^t\|^2 \big)}_{\text{``\texttt{M}''}} + \underbrace{\sum_{t=1}^T \frac{t}{2\rho^t} \big( \| \lambda - \lambda^{t+1} \|^2  - \| \lambda - \lambda^t \|^2 \big)}_{\text{``\texttt{N}''}} \Big\}. \label{strong:LHS_LB}
\end{align}
The ``\texttt{L}'' term in \eqref{strong:LHS_LB} can be written as
\begin{align*}
\langle \lambda, \sum_{t=1}^T t (z^{t+1}-z^t) \rangle = \langle \lambda,  \sum_{t=1}^T (z^{T+1}-z^t) \rangle \leq \sum_{t=1}^T \| \lambda \| \| z^{T+1} - z^t \| \leq \sum_{t=1}^T \| \lambda \| U_2 =  T U_2 \| \lambda \|.  
\end{align*}
The ``\texttt{M}'' term in \eqref{strong:LHS_LB} can be written as
\begin{align*}
& \sum_{t=1}^T \frac{t \rho^t}{2} \big( \| z^* - z^{t+1} \|^2 - \|z^* - z^t\|^2 \big)   \\
& = - \frac{\rho^1}{2} \|z^*-z^1\|^2 + \sum_{t=2}^T \Big( \underbrace{\frac{(t-1)\rho^{t-1} - t\rho^t}{2}}_{\leq 0 \text{ as } \rho^t \geq \rho^{t-1}} \Big)\|z^*-z^t \|^2 + \frac{T\rho^T}{2}\|z^*-z^{T+1} \|^2    \\
& \geq - \frac{\rho^1 U_2^2}{2}  + \sum_{t=2}^T \Big( \frac{(t-1)\rho^{t-1} - t\rho^t}{2} \Big) U_2^2 = \frac{ - U_2^2 T \rho^{T}}{2} \geq \frac{ - U_2^2 T \rho^{\text{max}}}{2}.  
\end{align*}
The ``\texttt{N}'' term in \eqref{strong:LHS_LB} can be written as
\begin{align*}      
& \sum_{t=1}^T \frac{t}{2\rho^t} \big( \| \lambda - \lambda^{t+1}\|^2  - \| \lambda - \lambda^t \|^2 \big) \geq - \frac{1}{2\rho^1} \| \lambda - \lambda^1 \|^2 + \sum_{t=2}^T \Big( \frac{t-1}{2 \rho^{t-1}} - \frac{t}{2 \rho^{t}} \Big)\| \lambda - \lambda^t \|^2. 
\end{align*}
Therefore, we have
\begin{align}
  & \text{LHS}(w^*,z^*) \geq F(z^{(T)}) - F(z^*) +  \langle \lambda, A w^{(T)} - z^{(T)} \rangle - \frac{2 U_2 \| \lambda \|}{T+1} - \frac{U_2^2 \rho^{\text{max}}}{T+1} \nonumber  \\
  & \hspace{20mm} + \frac{2}{T(T+1)} \Big( - \frac{1}{2\rho^1} \| \lambda - \lambda^1 \|^2 + \sum_{t=2}^T \underbrace{\Big( \frac{t-1}{2 \rho^{t-1}} - \frac{t}{2 \rho^{t}} \Big)}_{ \leq 0 \text{ by Assumption \ref{assump:convergence_stronglyconvex} (ii)}} \| \lambda - \lambda^t \|^2 \Big). \label{strong:LHS_LB_1}
\end{align}
In addition to \eqref{max_lambdas}, by Assumption \ref{assump:convergence_stronglyconvex} (i), we have
\begin{align*}
\bullet \ & \max_{\lambda: \| \lambda \| \leq \gamma } \ \| \lambda - \lambda^t \|^2 = \| \lambda^t \|^2 + \max_{\lambda: \| \lambda \| \leq \gamma } \ \left\{ \| \lambda\|^2 - 2 \langle \lambda, \lambda^t\rangle \right\} \leq  4 \gamma^2.
\end{align*}
By utilizing this to derive a lower bound of the last term in \eqref{strong:LHS_LB_1}, we have
\begin{align*}
- \frac{1}{2\rho^1} \| \lambda - \lambda^1 \|^2 + \sum_{t=2}^T \Big( \frac{t-1}{2 \rho^{t-1}} - \frac{t}{2 \rho^{t}} \Big) \| \lambda - \lambda^t \|^2 \geq 
- \frac{T}{2\rho^T} 4\gamma^2  \geq - \frac{T}{2\rho^1} 4\gamma^2 .
\end{align*}
Therefore, we have
\begin{align}
  \text{LHS}(w^*,z^*) \geq F(z^{(T)}) - F(z^*) + \gamma \| Aw^{(T)}-z^{(T)} \| - \frac{2 U_2 \gamma + U_2^2 \rho^{\text{max}} + 4 \gamma^2 / \rho^1 }{T+1}. \label{Strong_LHS_t_LB_1}
\end{align}

\subsection{Upper bound on $\text{RHS}(z^*) := \frac{2}{T(T+1)} \sum_{t=1}^T t \  \text{RHS}^t(z^*)$.}
\noindent
It follows from \eqref{Strong_RHS_t} that
\begin{align*}
\text{RHS}(z^*) = & \frac{2}{T(T+1)} \sum_{t=1}^T t \Big[  \frac{1}{E} \sum_{e=1}^E \Big\{  \frac{\eta^{t}}{2} \| f'(z^{t,e}) + \tilde{\xi}^{t,e} \|^2  \\
& + \big(\frac{1}{2\eta^{t}} - \frac{\alpha}{2})  \|z^* -z^{t,e}\|^2  - \frac{1}{2\eta^{t}} \| z^* - z^{t,e+1} \|^2  + \langle \tilde{\xi}^{t,e}, z^* - z^{t,e} \rangle \Big\} \Big]. 
\end{align*}
Note that
\begin{align*}
\bullet \ & \eta^{t} = 2/(\alpha(t+2)), \\
\bullet \ &  \sum_{t=1}^T \sum_{e=1}^E t \frac{\eta^{t}}{2} \| f'(z^{t,e}) + \tilde{\xi}^{t,e} \|^2 = \sum_{t=1}^T \sum_{e=1}^E  \frac{t}{\alpha(t+2)}  \| f'(z^{t,e}) + \tilde{\xi}^{t,e} \|^2, \\
\bullet \ & \sum_{t=1}^T \sum_{e=1}^E t \Big\{ (\frac{1}{2\eta^{t}} - \frac{\alpha}{2} ) \|z^* -z^{t,e} \|^2  - \frac{1}{2\eta^{t}} \|z^*-z^{t,e+1}\|^2 \Big\} = \frac{\alpha}{4} \sum_{t=1}^T \sum_{e=1}^E  \Big\{ t^2 \|z^* -z^{t,e} \|^2  - (t^2+2t) \|z^*-z^{t,e+1}\|^2 \Big\} \nonumber \\
& = \frac{\alpha}{4} \sum_{t=1}^T \Big\{  t^2   \Big( \|z^* -z^{t,1} \|^2  - \|z^*-z^{t,E+1}\|^2 \Big) -2t  \|z^*-z^{t,E+1}\|^2 \underbrace{- 2t\sum_{e=1}^{E-1} \|z^*-z^{t,e+1}\|^2}_{\leq 0}   \Big\} \nonumber \\
& \leq \frac{\alpha}{4} \sum_{t=1}^T \Big\{  t^2  \|z^* -z^{t,1} \|^2  - t(t+2) \|z^*-z^{t,E+1}\|^2 \Big\} \nonumber \\
& = \frac{\alpha}{4} \Big\{  \|z^* -z^{1,1} \|^2  + \sum_{t=2}^T t^2  \|z^* - \underbrace{z^{t,1}}_{ = z^{t-1,E+1}} \|^2  - \underbrace{\sum_{t=1}^{T-1} t(t+2) \|z^*-z^{t,E+1}\|^2}_{=\sum_{t=2}^{T} (t^2-1) \|z^*-z^{t-1,E+1}\|^2} \underbrace{- T(T+2) \|z^*-z^{T,E+1}\|^2}_{\leq 0} \Big\} \nonumber \\
& \leq \frac{\alpha}{4} \Big\{  \|z^* -z^{1,1} \|^2  + \sum_{t=2}^T  \|z^* - z^{t-1,E+1} \|^2  \Big\}   \leq \frac{\alpha}{4} T U_2^2.
\end{align*}  
Therefore, we have
\begin{align}
\text{RHS}(z^*) \leq \frac{\alpha U_2^2/E}{2(T+1)} + \frac{2}{ET(T+1)} \sum_{t=1}^T \sum_{e=1}^E \Big\{ \frac{t}{\alpha(t+2)} \sum_{p=1}^P \| f'_p(z^t_p) + \tilde{\xi}^{t,e}_p \|^2 +   t \langle \tilde{\xi}^{t,e}, z^* - z^{t,e} \rangle \Big\}. \label{Strong_RHS_t_LB_1}
\end{align}  

\subsection{Taking expectation.}
\noindent
By taking expectation on the inequality derived from \eqref{Strong_LHS_t_LB_1} and \eqref{Strong_RHS_t_LB_1}, we have
\begin{align*}
  & \mathbb{E} \Big[ F(z^{(T)}) - F(z^*) + \gamma \| Aw^{(T)} - z^{(T)} \| \Big]  \leq \frac{2 U_2 \gamma + U_2^2 \rho^{\text{max}} + 4 \gamma^2 / \rho^1 + \alpha U_2^2/(2E) }{T+1}
   \nonumber \\
   & + \frac{2}{ET(T+1)} \sum_{t=1}^T \sum_{e=1}^E \Big\{ \frac{t}{\alpha(t+2)} \sum_{p=1}^P \underbrace{\mathbb{E}\big[ \| f'_p(z^{t,e}_p) + \tilde{\xi}^{t,e}_p \|^2}_{\leq U_1^2+\overline{U}(\bar{\epsilon})} \big] + t  \langle \underbrace{\mathbb{E}[ \tilde{\xi}^{t,e}]}_{=0}, z^* - z^t \rangle \big] \\
\leq & \frac{2 U_2 \gamma + U_2^2 \rho^{\text{max}} + 4 \gamma^2 / \rho^1 + \alpha U_2^2/(2E) }{T+1}
+ \frac{2}{ET(T+1)} \frac{EP(U_1^2+\overline{U}(\bar{\epsilon}))}{\alpha} \underbrace{\sum_{t=1}^T \frac{t}{t+2}}_{ =  T - \sum_{t=3}^{T+2} \frac{2}{t} \leq T} \nonumber \\
\leq & \frac{2 U_2 \gamma + U_2^2 \rho^{\text{max}} + 4 \gamma^2 / \rho^1 + \alpha U_2^2/(2E) + 2P(U_1^2+\overline{U}(\bar{\epsilon}) )/\alpha }{T+1},
\end{align*}
where $\overline{U}(\bar{\epsilon})$ is from \eqref{U_eps}.
This completes the proof.

\section{Multiclass Logistic Regression Model} \label{apx:model}
The multiclass logistic regression model considered in this paper is \eqref{ERM_0} with
\begin{align}
  & \ell(w; x_{pi},y_{pi}) := - \textstyle\sum_{k=1}^K  y_{pik} \ln \big(h_k(w;x_{pi})\big), \ \forall p \in [P], \forall i \in [I_p],  \nonumber \\
  & h_k(w;x_{pi}) :=  \frac{\exp(\textstyle\sum_{j=1}^J  x_{pij} w_{jk}) }{\sum_{k'=1}^K \exp(\textstyle\sum_{j=1}^J  x_{pij} w_{jk'}) }, \ \forall p \in [P], \forall i \in [I_p], \forall k \in [K], \nonumber  \\
  & r(w) := \textstyle\sum_{j=1}^J \sum_{k=1}^K w_{jk}^2, \nonumber  \nonumber  \\
  & f_p(w) = \textstyle - \frac{1}{I} \sum_{i=1}^{I_p} \sum_{k=1}^K \big\{ y_{pik} \ln (h_k(w;x_{pi})) \big\} + \frac{\beta}{P} \sum_{j=1}^J \sum_{k=1}^K  w_{jk}^2, \ \forall p \in [P], \nonumber  \\
  & \nabla_{w_{jk}} f_p(w) = \textstyle\frac{1}{I}  \sum_{i=1}^{I_p} x_{pij} (h_k(w;x_{pi})-y_{pik}) + \frac{2\beta}{P} w_{jk} , \ \forall p \in [P], \forall j \in [J], \forall k \in [K]. \label{gradient}
\end{align}

\section{Choice of the Penalty Parameter $\rho^t$} \label{apx-hyperparameter-rho}
We test various $\rho^t$ for our algorithms and set it as $\hat{\rho}^t$ in \eqref{dynamic_rho} with (i) $c_1=2$, $c_2=5$, and $T_c=10000$ for MNIST and (ii) $c_1=0.005$, $c_2=0.05$, and $T_c=2000$ for FEMNIST.

Since these parameter settings may not lead \texttt{OutP} to its best performance, we test various $\rho^t$ for \texttt{OutP} using a set of static parameters, $\rho^t \in \{0.1, 1, 10\}$ for all $t \in [T]$, where $\rho^t=0.1$ is chosen in \cite{huang2019dp}, and dynamic parameters $\rho^t \in \{\hat{\rho}^t, \hat{\rho}^t/100 \}$, where $\hat{\rho}^t$ is from \eqref{dynamic_rho}.
In Figure \ref{fig:hyper_rho} we report the testing errors of \texttt{OutP} using MNIST and FEMNIST under various $\rho^t$ and $\bar{\epsilon}$.
The results imply that the performance of \texttt{OutP} is not greatly affected by the choice of $\rho^t$, but $\bar{\epsilon}$.
Hence, for all algorithms, we use $\hat{\rho}^t$ in \eqref{dynamic_rho}.

\begin{figure}[!htt]
  \centering
  \begin{subfigure}[b]{0.24\textwidth}
      \centering
      \includegraphics[width=\textwidth]{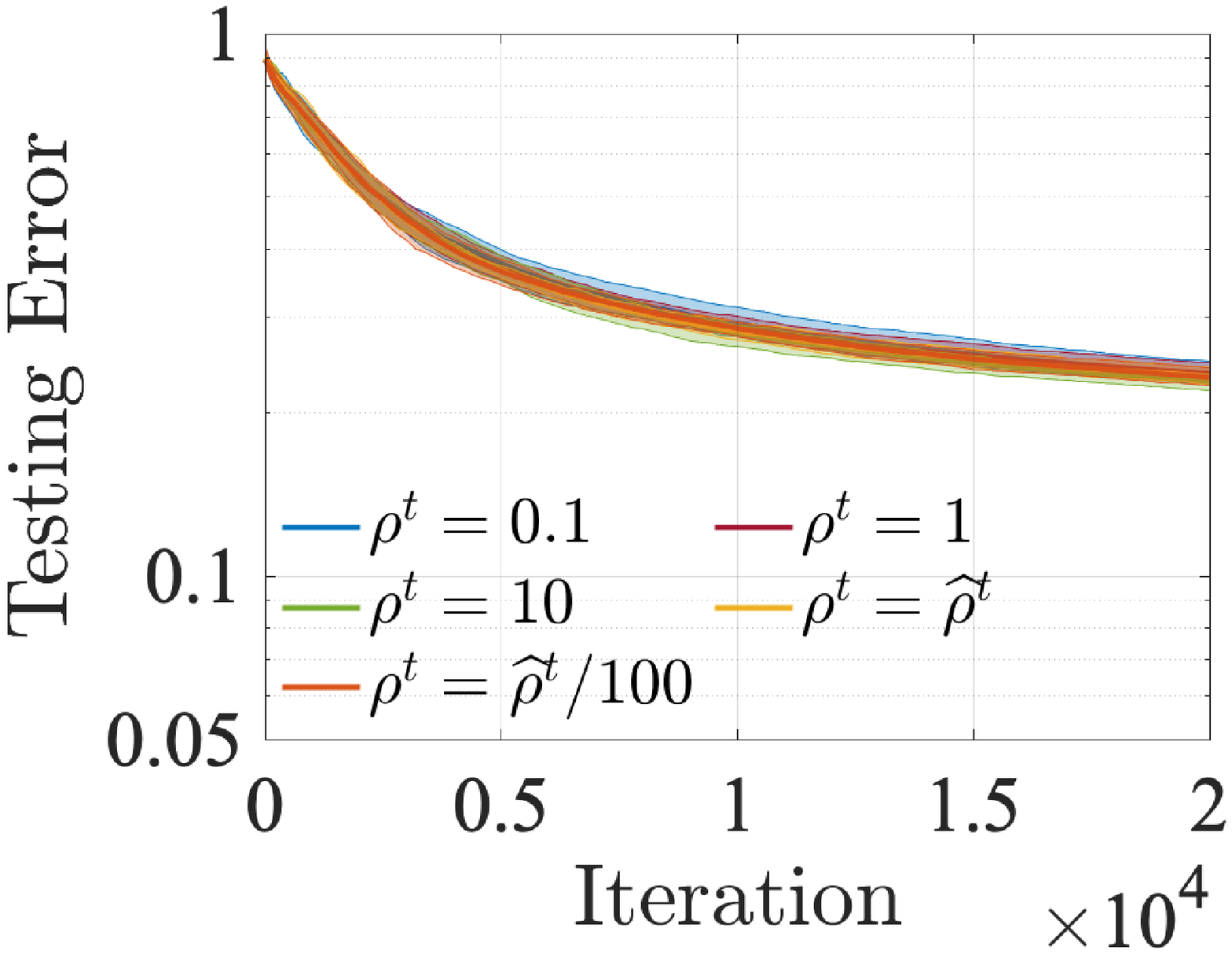}
      \includegraphics[width=\textwidth]{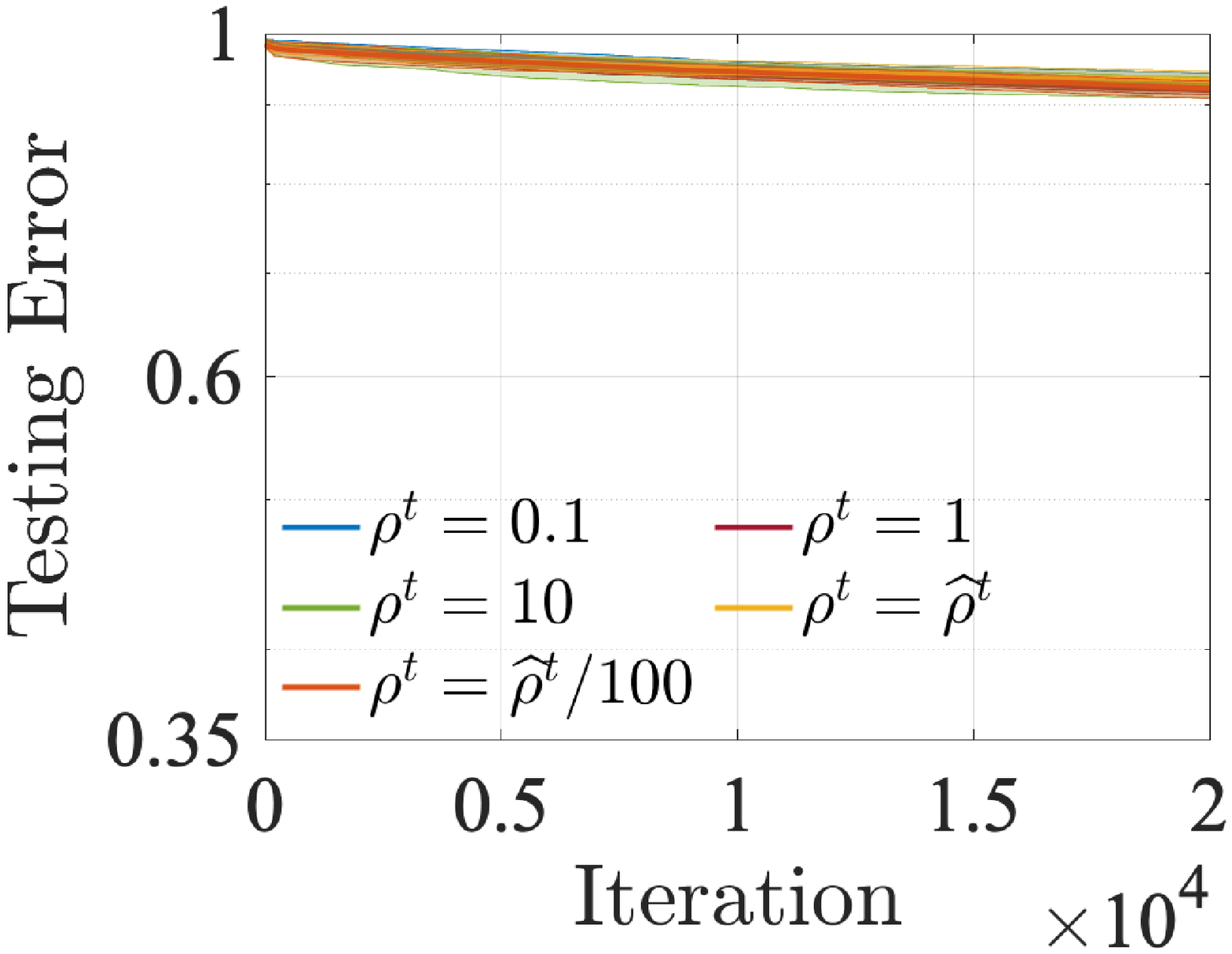}
      \caption{$\bar{\epsilon}=0.05$}
  \end{subfigure}
  \begin{subfigure}[b]{0.24\textwidth}
    \centering
    \includegraphics[width=\textwidth]{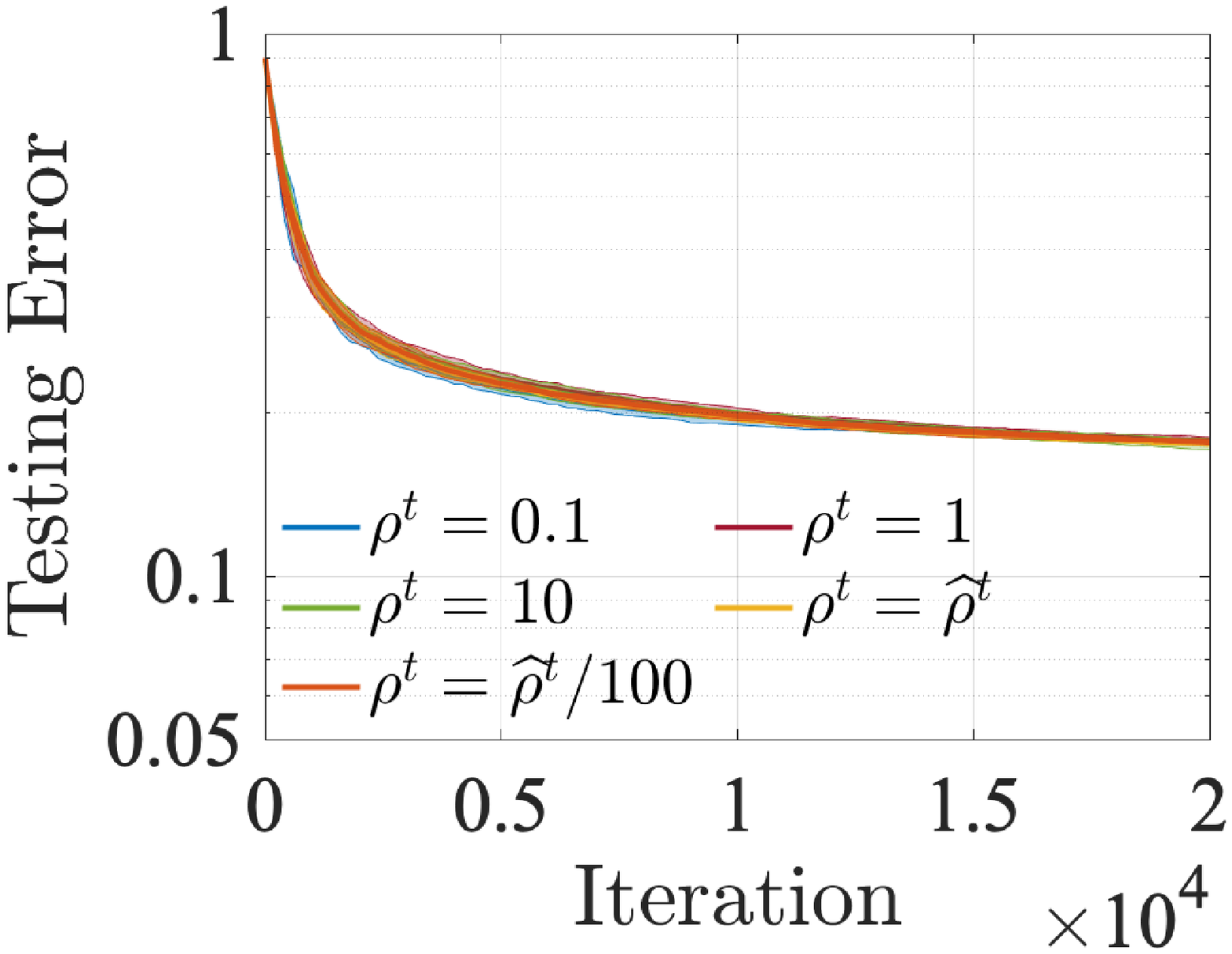}
    \includegraphics[width=\textwidth]{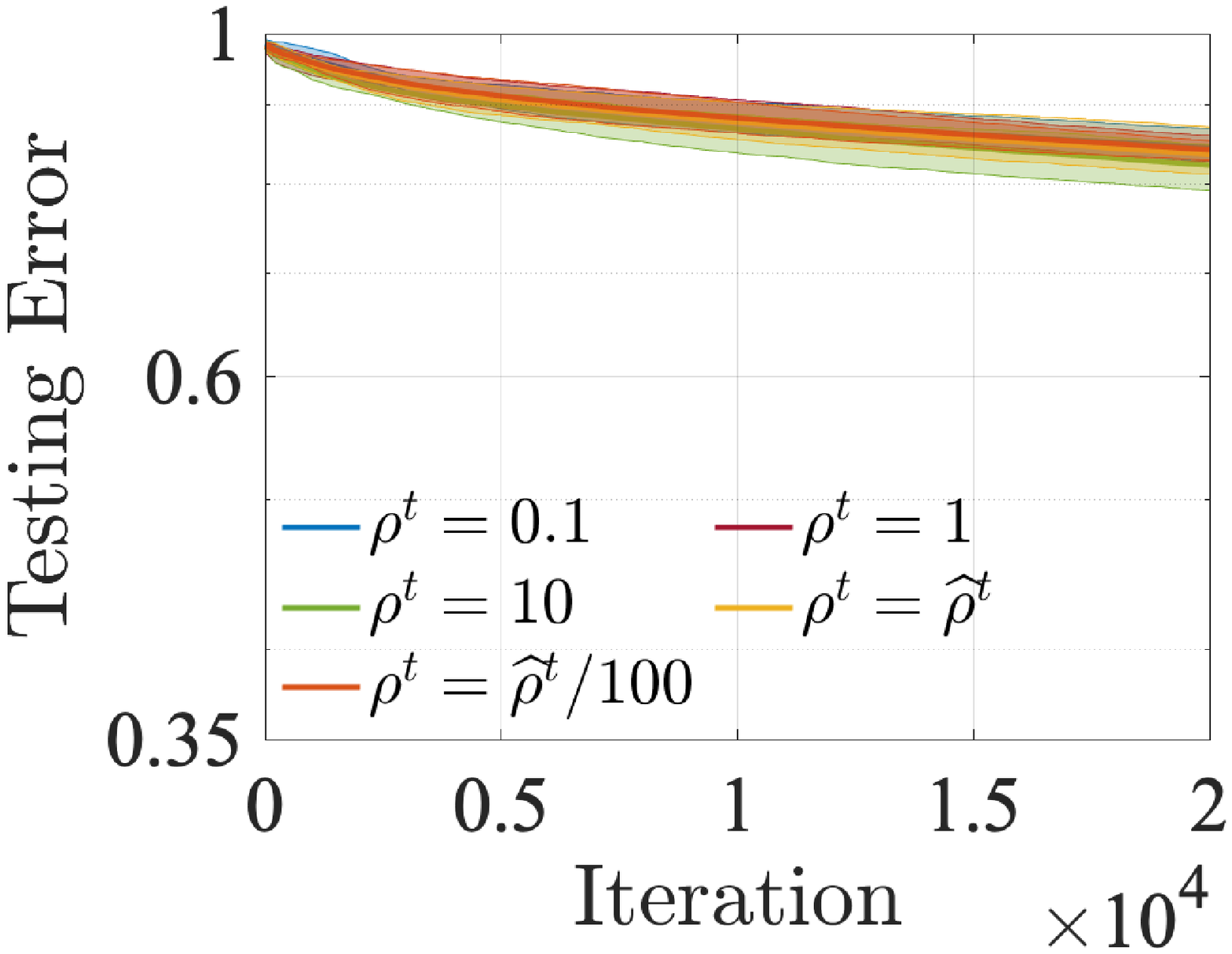}
    \caption{$\bar{\epsilon}=0.1$}
\end{subfigure}
  \begin{subfigure}[b]{0.24\textwidth}
      \centering
      \includegraphics[width=\textwidth]{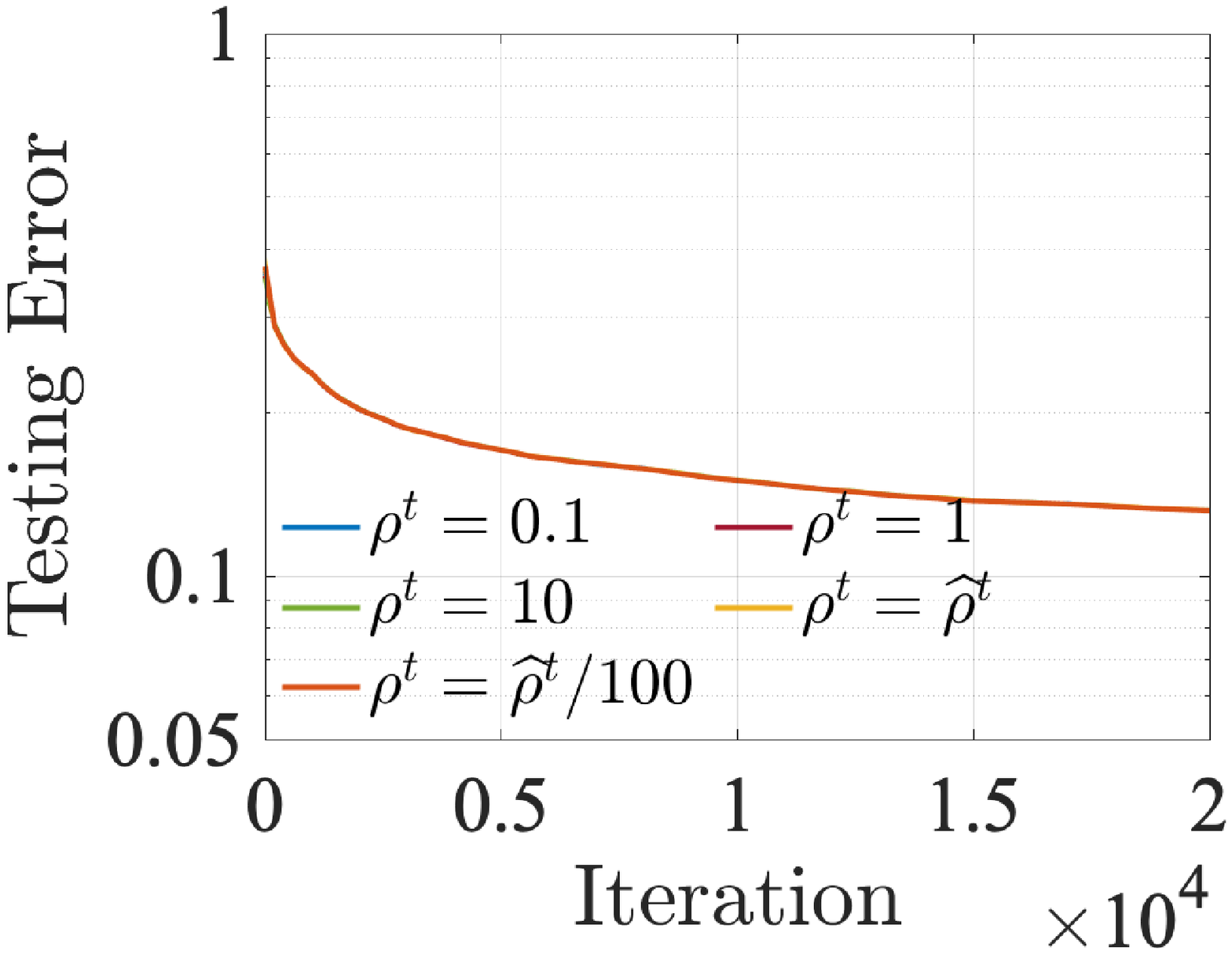}
      \includegraphics[width=\textwidth]{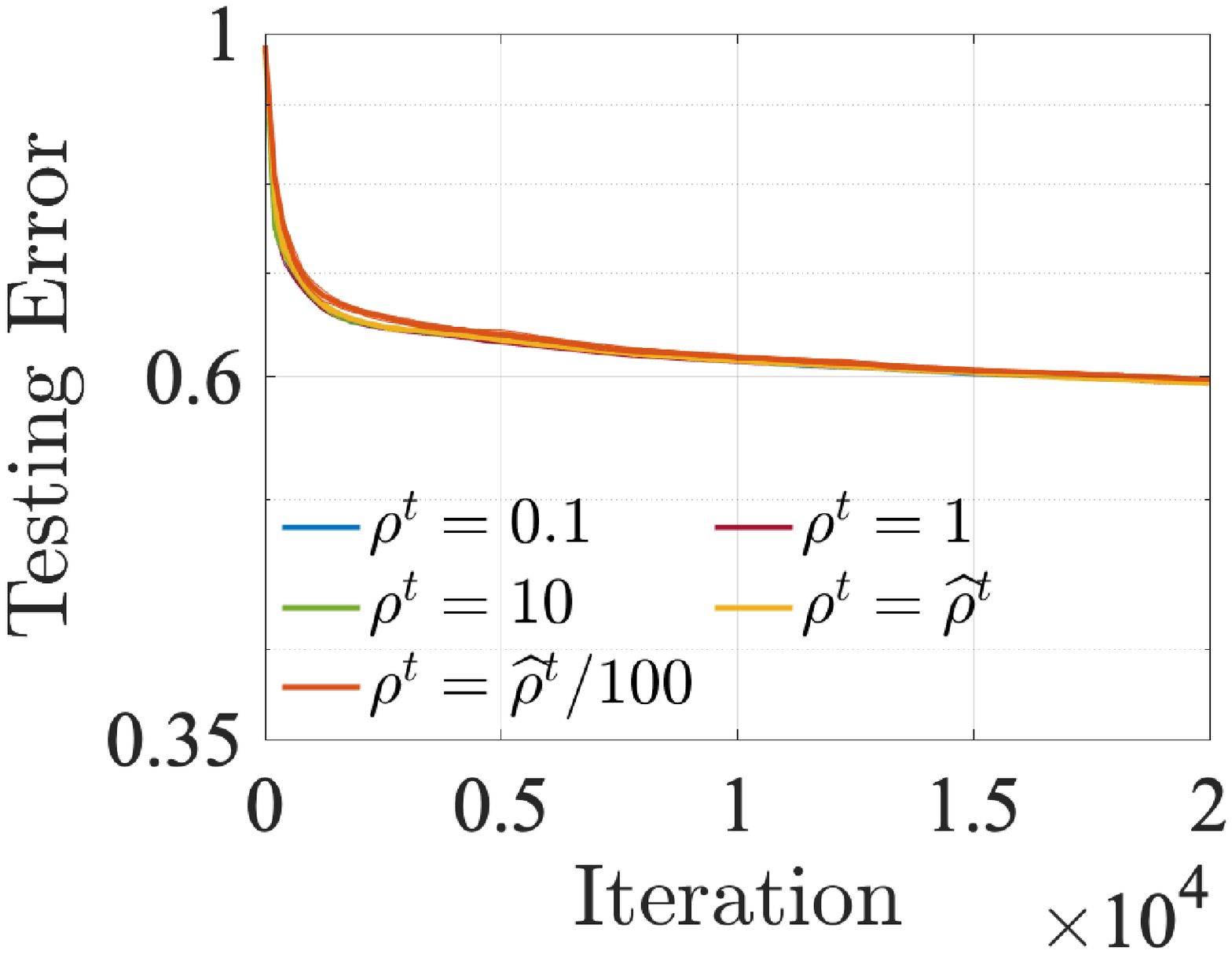}
      \caption{$\bar{\epsilon}=1$}
  \end{subfigure}
  \begin{subfigure}[b]{0.24\textwidth}
    \centering
    \includegraphics[width=\textwidth]{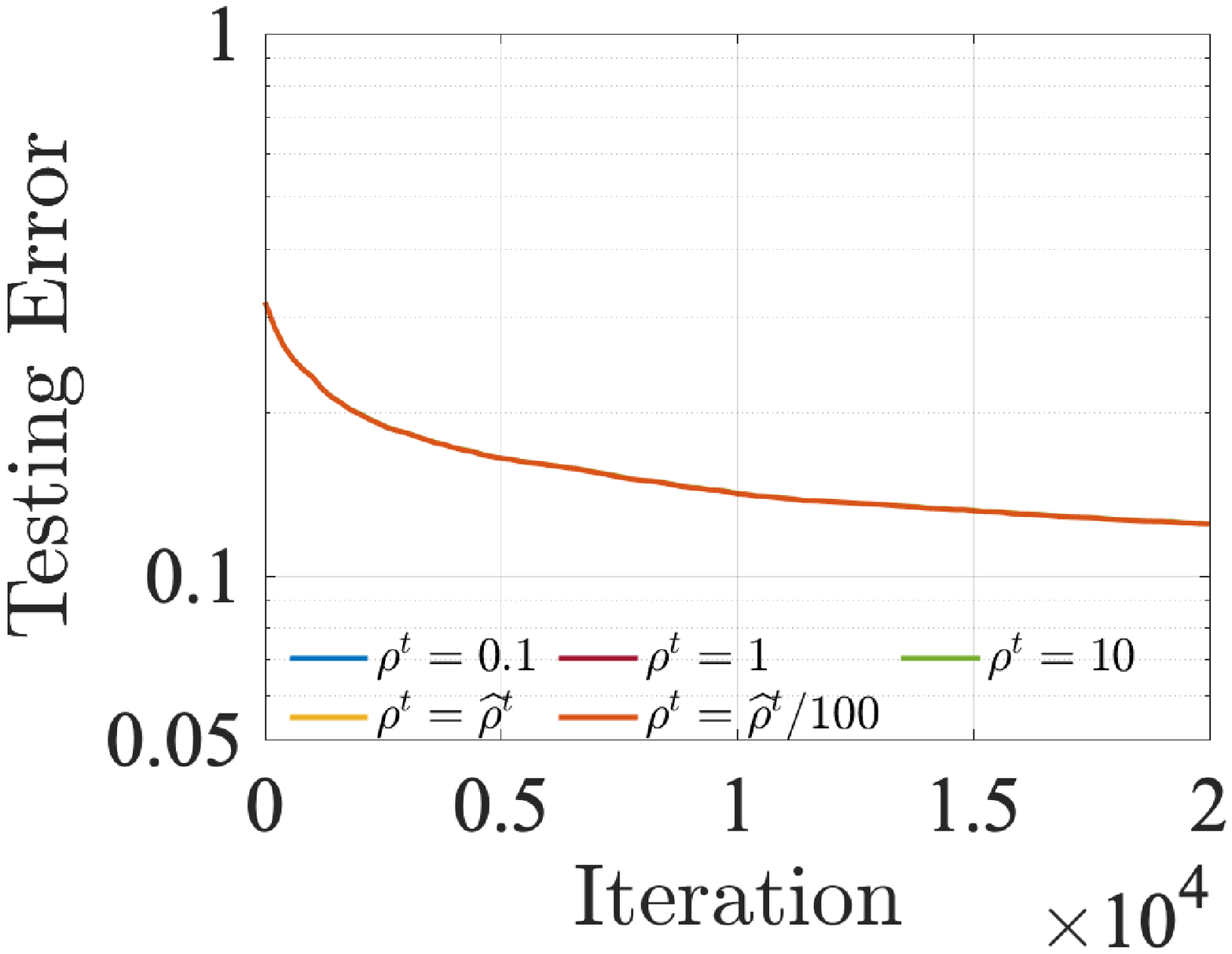}
    \includegraphics[width=\textwidth]{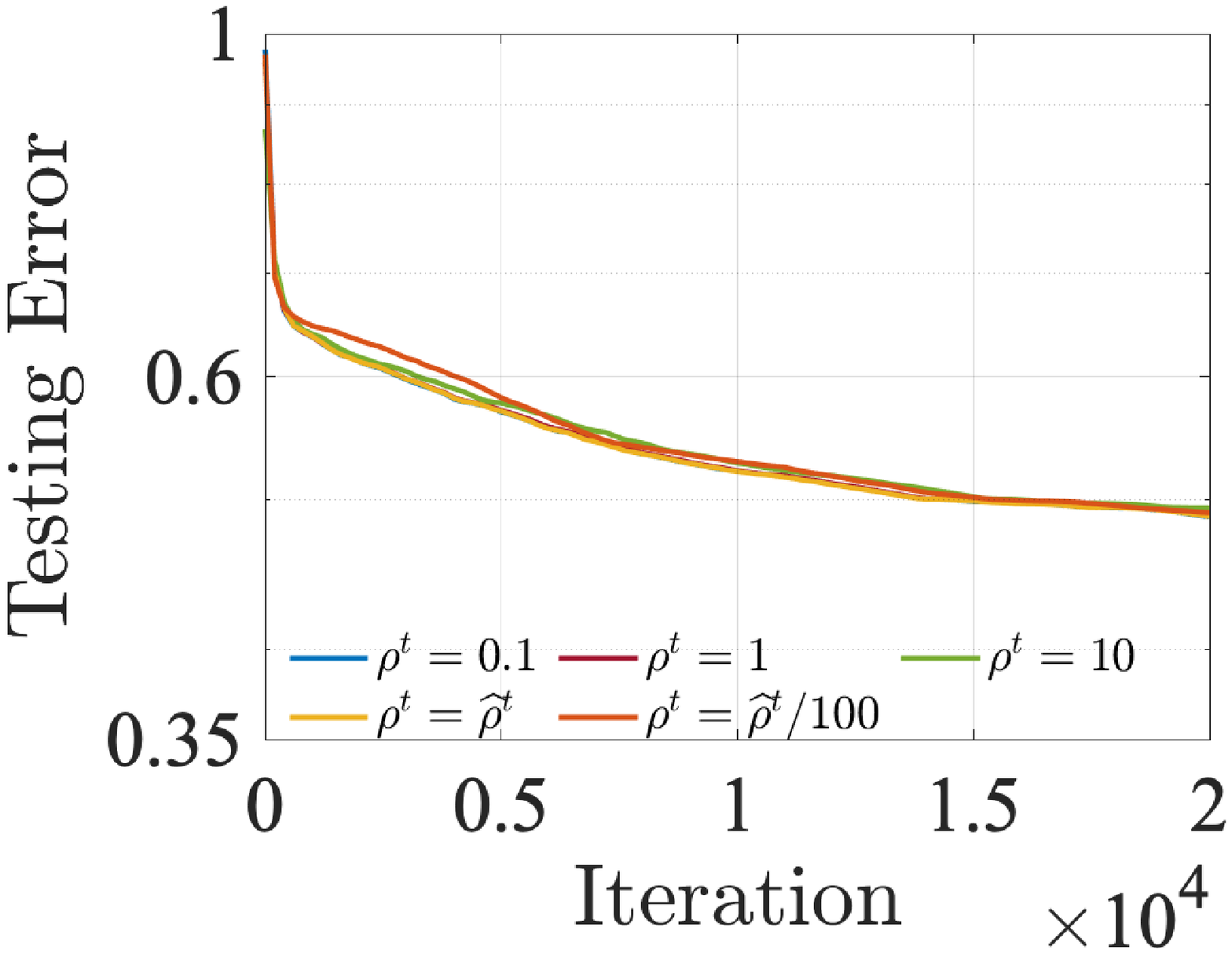}
    \caption{$\bar{\epsilon}=5$}
  \end{subfigure}
     \caption{Testing errors of \texttt{OutP} using MNIST (top) and FEMNIST (bottom).}
     \label{fig:hyper_rho}
\end{figure}

\end{document}